%% file: ms.tex
\documentclass[twoside,11pt]{article}

\usepackage{jmlr2e}
\usepackage{hyperref}       
\hypersetup{hidelinks}
\usepackage{url}            
\usepackage{booktabs}       
\usepackage{amsfonts}       
\usepackage{nicefrac}       
\usepackage{microtype}      
\usepackage{color}          

\usepackage{graphicx} 
\usepackage{amsmath}
\usepackage{mathtools}
\usepackage{bbm}

\usepackage{stmaryrd}
\usepackage{enumitem}
\usepackage{multirow}
\usepackage{algorithm}
\usepackage{algorithmic}
\usepackage{tikz}
\usetikzlibrary{shapes,backgrounds}
\usepackage{url}
\usepackage{xspace}
\usepackage{thmtools, thm-restate}
\usepackage{diagbox}
\usepackage{xspace}
\usepackage{graphicx} 
\usepackage{caption}
\usepackage{subcaption}

\allowdisplaybreaks

\newcommand{\tosee}[1]{#1}

\newcommand{\prodspace}{\mathcal{X}\times A \times \mathcal{Y}}
\newcommand{\poisoned}{\mathfrak{M}}
\newcommand{\ObjVec}{\mathfrak{V}}
\newcommand{\Bin}{\operatorname{Bin}}

\DeclareMathOperator*{\argmin}{argmin}

\newcommand{\myparagraph}[1]{\paragraph{#1}\ }

\makeatletter
\DeclareRobustCommand\onedot{\futurelet\@let@token\@onedot}
\def\@onedot{\ifx\@let@token.\else.\null\fi\xspace}
\def\iid{{i.i.d}\onedot}
 
\def\aka{{a.k.a}\onedot}

\def\etc{{etc}\onedot}

\makeatother

\graphicspath{ {images/} }

\usepackage{lastpage}
\jmlrheading{23}{2022}{1-\pageref{LastPage}}{10/21; Revised
3/22}{5/22}{21-1189}{Nikola Konstantinov and Christoph H. Lampert}
\ShortHeadings{Fairness-Aware PAC Learning from Corrupted Data}{Konstantinov and Lampert}
\firstpageno{1}

\begin{document}

\title{Fairness-Aware PAC Learning from Corrupted Data}

\author{\name Nikola Konstantinov$^1$ \email nikolahristov.konstantinov@inf.ethz.ch \\
       \addr Post-Doctoral Fellow, ETH AI Center\\
       Universitätstrasse 6\\
       8092 Zürich, Switzerland
       \AND
       \name Christoph H. Lampert \email chl@ist.ac.at \\
       \addr Institute of Science and Technology Austria (ISTA)\\
       Am Campus 1\\
       3400 Klosterneuburg, Austria}

\editor{Pradeep Ravikumar}
\maketitle

\setcounter{footnote}{1}

\begin{abstract}
Addressing fairness concerns about machine learning models is a crucial step towards their long-term adoption in real-world automated systems. While many approaches have been developed for training fair models from data, little is known about the robustness of these methods to data corruption. In this work we consider fairness-aware learning under worst-case data manipulations. We show that an adversary can in some situations force any learner to return an overly biased classifier, regardless of the sample size and with or without degrading accuracy, and that the strength of the excess bias increases for learning problems with underrepresented protected groups in the data. We also prove that our hardness results are tight up to constant factors. To this end, we study two natural learning algorithms that optimize for both accuracy and fairness and show that these algorithms enjoy guarantees that are order-optimal in terms of the corruption ratio and the protected groups frequencies in the large data limit.\footnotetext{Nikola Konstantinov conducted the work on this paper while being at the Institute of Science and Technology Austria (ISTA). A preliminary version of the work contained in this article appeared in the AFCR@NeurIPS workshop \citep{konstantinov2021impossibility}.}
\end{abstract}


\section{Introduction}
\label{sec:introduction}

\input{introduction}

\section{Related work}
\label{sec:related-work}

\input{related_work}

\section{Preliminaries}
\label{sec:setup}

\input{preliminaries}

\section{Lower bounds}
\label{sec:lower-bounds}

\input{lower_bounds}

\section{Upper bounds}
\label{sec:upper-bounds}

\input{upper_bounds}

\section{Discussion}
\label{sec:conclusion}

\input{conclusion}

\section*{Acknowledgments}
The authors thank Eugenia Iofinova and Bernd Prach for providing feedback on early versions of this paper. This publication was made possible by an ETH AI Center postdoctoral fellowship to Nikola Konstantinov.

\bibliography{ms}

\onecolumn
\appendix

\clearpage

\centerline{\Huge Supplementary Material}
\bigskip

The supplementary material is structured as follows.
\begin{itemize}
\item \textbf{Appendix \ref{sec:appendix-lower-bounds-proofs}} contains the proofs of all lower bounds results. Section \ref{sec:pareto-lower-bounds-proofs} focuses on the Pareto lower bounds. Section \ref{sec:good-accuracy-lower-bounds-proofs} contains the proofs for the lower bounds on fairness, given that accuracy is kept optimal. 
\item \textbf{Appendix \ref{sec:appendix-upper-bounds-proofs}} contains the complete proofs of our upper bound results. In particular, Section \ref{sec:appendix-notation-and-tools} explains the notation and introduces the classic concentration tools that we will use. In Section \ref{sec:concentration-lemmas-proofs} a number of concentration results under corrupted data for the demographic parity and equal opportunity fairness notions are shown. Finally, Section \ref{sec:upper-bounds-themselves-proofs} gives the formal proofs of all upper bound results, building on the concentration inequalities from the previous section.
\end{itemize}

\section{Lower bounds proofs}
\label{sec:appendix-lower-bounds-proofs}

\input{appendix_lower_bound_proofs}

\section{Upper bounds proofs}
\label{sec:appendix-upper-bounds-proofs}

\input{appendix_upper_bounds_proofs}

\end{document}

%% file: introduction.tex
Recent years have seen machine learning models greatly advancing the state-of-art performance of automated systems on many real-world tasks. As learned models become increasingly adopted in high-stake decision making, various fairness concerns arise. Indeed, it is now widely recognized that without addressing fairness issues during training, machine learning models can exhibit discriminatory behavior at prediction time \citep{barocas-hardt-narayanan}. Designing principled methods for certifying the fairness of a model is therefore key for increasing the trust in these methods among the general public. 

To this end many ways of measuring and optimizing the fairness of learned models have been developed. The problem is perhaps best studied in the context of group fairness in classification, where the decisions of a binary classifier have to be nondiscriminatory with respect to a certain protected attribute, such as gender or race \citep{barocas-hardt-narayanan}. This is typically done by formulating a desirable fairness property for the task at hand and then optimizing for this property, alongside with accuracy, be it via a data preprocessing step, a modification of the training procedure, or by post-processing of a learned classifier on held-out data \citep{mehrabi2019survey}. The underlying assumption is that by ensuring that the fairness property holds exactly or approximately based on the available data, one obtains a classifier whose decisions will also be fair at prediction time.

A major drawback of this framework is that for many real-world applications the training and validation data available are often times unreliable and biased \citep{biggio2018wild, mehrabi2019survey}. For example, demographic data collected via surveys or online polls is often difficult and expensive to verify. More generally any human-generated data is likely to contain various historical biases. Datasets collected via crowdsourcing or web crawing are also prone to both unwittingly created errors and conscious or even adversarially created biases.

These issues naturally raise concerns about the current practice of training and certifying fair models on such datasets. In fact, recent work has demonstrated empirically that strong poisoning attacks can negatively impact the fairness of \textit{specific learners} based on loss minimization \tosee{\citep{solans2020poisoning, chang2020adversarial, mehrabi2020exacerbating}}. At the same time, little is known about the fundamental limits of fairness-aware learning from corrupted data. Previous work has only partially addressed the problem by studying weak data corruption models, for example by making specific label/attribute noise assumptions. However, these assumptions do not cover all possible (often unknown) problems that real-world data can possess. More generally, in order to avoid a cat-and-mouse game of designing defenses and attacks for fair machine learning models, one would need to be able to \textit{certify fairness} as a property that holds when training under arbitrary, even adversarial, manipulations of the training data \citep{kearns1993learning}.

\myparagraph{Contributions} In our work, we address the aforementioned issues by studying the effect of arbitrary data corruptions on fair learning algorithms. Specifically, we explore the fundamental limits of fairness-aware PAC learning within the classic \textit{malicious adversary model} of \cite{valiant1985learning}, where the adversary can replace a fraction of the data points with arbitrary data, with full knowledge of the learning algorithm, the data distribution and the remaining samples. We focus on binary classification with two popular group fairness constraints - demographic parity \citep{calders2009building} and equal opportunity \citep{hardt2016equality}.

First we show that learning under this adversarial model is provably impossible in a PAC sense - there is \textit{no learning algorithm that can ensure convergence with high probability to a point on the accuracy-fairness Pareto front} on the set of all finite hypothesis spaces, even in the limit of infinite training data. Furthermore, the irreducible excess gap in the fairness measures we study is inversely proportional to the frequency of the rarer of the two protected attributes groups. This makes the robust learning problem especially hard when one of the protected subgroups in the data is underrepresented. These hardness results hold for \textit{any learning algorithm} based on a corrupted dataset, including pre-, in- and post-processing methods in particular.

Perhaps an even more concerning result from a practical perspective is that the adversary can also ensure that any learning algorithm will output a classifier that is \textit{optimal in terms of accuracy, but exhibits a large amount of unfairness}. The bias of such a classifier might go unnoticed for a long time in production systems, especially in applications where sensitive attributes are not revealed to the system at prediction time for privacy reasons.

We also show that our hardness results are tight up to constant factors, in terms of the corruption ratio and the protected group frequencies, by proving matching upper bounds. To this end we study the performance of two natural types of learning algorithms under the malicious adversary model. We show that both algorithms achieve order-optimal performance in the infinite data regime, \textit{thereby providing tight upper and lower bounds on the irreducible error of fairness-aware statistical learning under adversarial data corruption.} 

We conclude with a discussion on the implications of our hardness results, emphasizing the need for developing and studying further data corruption models for fairness-aware learning, as well as on the importance of strict data collection practices in the context of fair machine learning.

%% file: related_work.tex
To the best of our knowledge, we are the first to investigate the information-theoretic limits of fairness-aware learning against a malicious adversary. There is, however, related previous work on PAC learning analysis of fair algorithms, robust fair learning, and learning with poisoned training data, that we discuss in this section.

\myparagraph{Fairness in classification} Fairness-aware learning has been widely studied in the context of classification. We refer to \cite{mehrabi2019survey} for an exhaustive introduction to the field. In this paper we focus on two popular notions of group fairness - demographic parity \citep{calders2009building} and equal opportunity \citep{hardt2016equality}. On the methodological side, our upper bounds analysis employs a technique for proving concentration of estimates of conditional probabilities that has previously been used in the context of group fairness by \cite{woodworth2017learning} and \cite{agarwal2018reductions}. A number of hardness results for fair learning are also known. In particular, \cite{kleinberg2016inherent} prove the incompatability of three fairness notions for a broad class of learning problems and \cite{pmlr-v81-menon18a} quantify fundamental trade-offs between fairness and accuracy. Both of these works, however, focus on learning with \iid clean data.

\myparagraph{Fairness and data corruption}
Most relevant for our setup are a number of recent works that empirically study attacks and defenses on fair learners under adversarial data poisoning. In particular, \cite{solans2020poisoning}, \cite{chang2020adversarial} and \cite{mehrabi2020exacerbating} consider practical, gradient-based poisoning attacks against machine learning algorithms. All of these works demonstrate empirically that poisoned data can severely damage the performance of fair learners that are based on empirical loss minimization. In our work we go beyond this by proving a set of hardness results that hold for \textit{arbitrary learning algorithms}. On the defense side, \cite{roh2020fr} design and empirically study an adversarial training approach for dealing with data corruption when training fair models. Their defense is shown to be effective against specific poisoning attacks that aim to reduce the model accuracy. In contrast, for our upper bounds we are interested in learners that provably work against any poisoning attack, including those that can target the fairness properties of the model as well.

Among works focusing on weaker adversarial models, a particularly popular topic is the one of fair learning with noisy or adversarially perturbed protected attributes \citep{lamy2019noise, awasthi2020equalized, wang2020robust, celis2021fair, mehrotra2020mitigating, celis2021fair2}. Under the explicit assumption that the corruption does not affect the inputs and the labels, these works propose algorithms that can recover a fair model despite the data corruption. A related, but conceptually different topic is the one of fair learning without demographic information \citep{hashimoto2018fairness, kallus2019assessing, mozannar2020fair, lahoti2020fairness}. Another commonly assumed type of corruption is label noise, which is shown to be overcomable under various assumptions by \cite{de2018learning}, \cite{jiang2020identifying}, \cite{wang2020fair} and \cite{fogliato2020fairness}. The concurrent work of \cite{jo2022breaking} studies the hardness of fairness-aware learning with adversarial corruptions of both the labels and the protected attributes (but not the input variables), also allowing for the adversary to choose the points it can manipulate. However, they focus on studying adversarial strategies for enforcing a fixed target model, while we focus on understanding the statistical limits on the performance of the learner in terms of both fairness and accuracy.

A distributionally robust approach for certifying fairness is taken by \cite{taskesen2020distributionally}, under the assumption that the real data distribution falls within a Wasserstein ball centered at the empirical data distribution. In \cite{ignatiev2020towards} a formal methods framework for certifying fairness through unawareness, even in the presence of a specific type of data bias that targets their desired fairness measure, is provided. The vulnerability of fair learning algorithms to specific types of data corruption has also been demonstrated on real-world data by \cite{calders2013unbiased} and \cite{kallus2018residual}.

An orthogonal line of work shows that imposing fairness constraints can neutralize the effects of corrupted data, under specific assumptions on the type of bias present \citep{blum2020recovering}. Also related are the works of \cite{tae2019data} and \cite{li2020tilted} who propose procedures for data cleaning/outlier detection, without a specific adversarial model, that in particular improve fairness performance.

\myparagraph{Learning against an adversary} Learning from corrupted training data is a field with long history, where both the theoretical and the practical aspects of attacking and defending  ML models have been widely studied \citep{angluin1988learning, kearns1993learning, cesa1999sample, bshouty2002pac, biggio2012poisoning, charikar2017learning, steinhardt2017certified, chen2017targeted, diakonikolas2019sever}. In this work we study fair learning within the so-called malicious adversary model, introduced by \cite{valiant1985learning}. The fundamental limits of classic PAC learning in this context have been extensively explored by \cite{kearns1993learning} and \cite{cesa1999sample}. Our paper adds an additional dimension to this line of work, where fairness is considered alongside with accuracy as an objective for the learner.

%% file: preliminaries.tex
In this section we formalize the problem of fairness-aware learning against a malicious adversary, by giving precise definitions of the learning objectives and the studied data corruption model.

\subsection{Fairness-aware learning} 
\label{sec:fairness-aware-learning}
Throughout the paper we adopt the following standard group fairness classification framework. We consider a product space $\prodspace$, where $\mathcal{X}$ is an input space, $\mathcal{Y} = \{0,1\}$ is a binary label space and $A = \{0,1\}$ is a set corresponding to a binary protected attribute (for example, being part of a majority/minority group). We assume that there is an unknown true data distribution $\mathbb{P}\in\mathcal{P}(\prodspace)$ from which the clean data is sampled. Denote by $\mathcal{H} \subseteq \{h: \mathcal{X} \to \mathcal{Y}\}$ the hypothesis space of all classifiers to be considered.

\myparagraph{PAC learning} Adopting a statistical PAC learning setup, we are interested in designing learning procedures that find a classifier based on training examples. Formally, a (statistical) fairness-aware learner $\mathcal{L}:\cup_{n\in\mathbb{N}}(\prodspace)^n \to \mathcal{H}$ is a function that takes a labeled dataset of an arbitrary size and outputs a hypothesis. Note that we consider learning in the purely statistical sense here, focusing on \textit{any} procedure that outputs a hypothesis, regardless of its computational complexity, and seeking learners that are sample-efficient instead.

In a clean data setup, the learner is trained on a dataset $S^c = \{(x^c_i, a^c_i, y^c_i)\}_{i=1}^n$ sampled \iid from $\mathbb{P}$ and outputs a hypothesis $h \coloneqq \mathcal{L}(S^c)$. The performance of a learner can be measured via the expected $0/1$ loss (\aka the risk) with respect to the distribution $\mathbb{P}$
\begin{equation}
\label{eqn:zero-one-loss}
\mathcal{R}(h, \mathbb{P}) = \mathbb{P}(h(X) \neq Y).
\end{equation}

\myparagraph{Group fairness in classification} In (group) fairness-aware learning, an additional desirable property of the classifier $h = \mathcal{L}(S^c)$ is that its decisions are fair in the sense that it does not exhibit discrimination with respect to one of the protected subgroups in the population. Many different formal notions of group fairness have previously been proposed in the literature. The problem of selecting the ``right'' fairness measure is in general application-dependent and beyond the scope of this work. 

Here we focus on the two arguably most widely adopted measures. The first one, \textit{demographic parity} \citep{calders2009building}, requires that the decisions of the classifier are independent of the protected attribute, that is
\begin{equation}
\label{eqn:demog_par}
\mathbb{P}(h(X) = 1| A = 0) = \mathbb{P}(h(X) = 1| A = 1).
\end{equation}
The second one, \textit{equal opportunity} \citep{hardt2016equality}, states that the true positive rates of the classifier should be equal across the protected groups, that is
\begin{equation}
\label{eqn:equal_opp}
\mathbb{P}(h(X) = 1| A = 0, Y = 1) = \mathbb{P}(h(X) = 1| A = 1, Y = 1).
\end{equation}
In this definition, an implicit assumption is that $Y = 1$ corresponds to a beneficial outcome (for example, an applicant receiving a job), so that this fairness notion only considers instances where the correct outcome should be advantageous.

In practice, it is rarely the case that a classifier achieves perfect fairness. Therefore, we will instead be interested in controlling the \textit{amount of unfairness} that $h$ possesses, measured via corresponding fairness deviation measures $\mathcal{D}(h)$ \citep{woodworth2017learning, menon2018cost, williamson2019fairness}. Here we adopt the \textit{mean difference score} measure of \cite{calders2010three} and \cite{menon2018cost} for demographic parity
\begin{align}
\label{eqn:demog-parity-deviation}
\mathcal{D}^{par}(h, \mathbb{P}) = \Big|\mathbb{P}(h(X) = 1| A = 0) - \mathbb{P}(h(X) = 1| A = 1)\Big| 
\end{align}
and its analog for equal opportunity
\begin{align}
\label{eqn:equal-opp-deviation}
\mathcal{D}^{opp}(h, \mathbb{P}) = \Big|\mathbb{P}(h(X) = 1| A = 0, Y = 1) - \mathbb{P}(h(X) = 1| A = 1, Y = 1)\Big|.
\end{align}
To avoid degenerate cases for these measures, we assume throughout the paper that $P_ a = \mathbb{P}(A = a) > 0$ and $P_{1a} = \mathbb{P}(Y = 1, A = a) > 0$ for both $a\in\{0,1\}$. \tosee{For the rest of the paper, whenever we are interested in demographic parity fairness, we assume without loss of generality that $A = 0$ is the minority class, so that $P_0 \leq \frac{1}{2} \leq P_1$. Similarly, whenever the fairness notion is equal opportunity, we assume without loss of generality that $P_{10} \leq P_{11}$ (note that we do not make any assumption about $P_0$ and $P_1$ in this case).}

Whenever the underlying distribution is clear from the context, we will drop the dependence of $\mathcal{R}(h, \mathbb{P})$ and $\mathcal{D}(h, \mathbb{P})$ on $\mathbb{P}$ and simply write $\mathcal{R}(h)$ and $\mathcal{D}(h)$.

\subsection{Learning against an adversary}
\label{sec:learning-against-an-adversary}
As argued in the introduction, machine learning models are often trained on unreliable datasets, where some of the points might be corrupted by noise, human biases and/or malicious agents. To model arbitrary manipulations of the data, we assume the presence of an adversary that can modify a certain fraction of the dataset and study fair learning in this context. In addition to not being partial to a specific type of data corruption, this worst-case approach has the advantage of providing a \textit{certificate for fairness}: if a system can work against a strong adversarial model, it will be effective under \textit{any circumstances that are covered by the model}.  

Formally, a \textit{fairness-aware adversary} is any procedure for manipulating a dataset, that is a \textit{possibly randomized function} $\mathcal{A}: \cup_{n\in \mathbb{N}} (\prodspace)^n \to \cup_{n\in \mathbb{N}}(\prodspace)^n$ that takes in a clean dataset $S^c = \{(x^c_i, a^c_i, y^c_i)\}_{i=1}^n$ sampled \iid from $\mathbb{P}$ and outputs a new, corrupted, dataset $S^p = \{(x^p_i, a^p_i, y^p_i)\}_{i=1}^n$ \textit{of the same size}. Depending on the type of restrictions that are imposed on the adversary, various adversarial models can be obtained.

In this work we adopt the powerful \textit{malicious adversary model}, first introduced by \cite{valiant1985learning} and extensively studied by \cite{kearns1993learning} and \cite{cesa1999sample} \footnote{\tosee{Strictly speaking, the nasty noise model of \cite{bshouty2002pac, diakonikolas2019robust}, in which the adversary can even choose the marked points, is even stronger than the adversary model we consider here. However, we opted for studying the malicious noise model, because this weaker adversary is already sufficient for showing strong impossibility results on learning. We refer to Section \ref{sec:conclusion} for further discussion on the choice of adversarial model.}}. The formal data generating procedure is as follows:
\begin{itemize}
\item An \iid \textit{clean dataset} $S^c = \{(x^c_i, a^c_i, y^c_i)\}_{i=1}^n$ is sampled from $\mathbb{P}$.
\item Each index/point $i\in\{1, 2, \ldots, n\}$ is \textit{marked} independently with probability $\alpha$, for a fixed constant $\alpha\in [0,0.5)$. \tosee{Denote all marked indexes by $\poisoned \subseteq [n]$.}
\item The \textit{malicious adversary} computes, in a possibly randomized manner, a corrupted dataset $S^p = \{(x^p_i, a^p_i, y^p_i)\}_{i=1}^n \in (\prodspace)^n$, with the only restriction that $(x^p_i, a^p_i, y^p_i) = (x^c_i, a^c_i, y^c_i)$ for all $i\not\in\poisoned$. That is, the adversary can replace all marked data points in an arbitrary manner, with \textit{no assumptions whatsoever} about the points $(x^p_i, a^p_i, y^p_i)$ for $i\in\poisoned$.
\item The corrupted dataset $S^p$ is then passed on to the learner, who computes $\mathcal{L}(S^p)$.
\end{itemize}
For a fixed $\alpha\in [0, 0.5)$, we say that $\mathcal{A}$ is a malicious adversary of power $\alpha$. Note that the number of marked points is $|\poisoned| \sim Bin(n, \alpha)$. 

Since no assumptions are made on the corrupted data points, they can, in particular, depend on the learner $\mathcal{L}$, the data distribution $\mathbb{P}$, the clean data $S^c$ and all other parameters of the learning problem. That is, the adversary acts with full knowledge of the learning setup and without any computational constraints, which is in lines with our worst-case approach. Note that this is in contrast to the learner $\mathcal{L}$ that can only access the data points in $S^p$. We refer to Section \ref{sec:limits_of_learning} for a more formal treatment. 

\subsection{Multi-objective learning}
\label{sec:multi-objective-adv-learning}
Our goal is to study the performance of the classifier $\mathcal{L}(S^p)$ \textit{learned on the corrupted data}, both in terms of its expected loss $\mathcal{R}(\mathcal{L}(S^p), \mathbb{P})$ and its fairness deviation $\mathcal{D}(\mathcal{L}(S^p), \mathbb{P})$ \textit{on the clean (test) distribution $\mathbb{P}$}. We will be interested in the probabilities of these quantities being large or small, under the randomness of the sampling of $S^p$ - that is the randomness of the clean data, the marked points and the adversary.

Note that it is not a priori clear how to trade-off the two metrics and that this is likely to be application-dependent. Therefore it is also unclear how to evaluate the quality of a hypothesis. In our work we study two possible ways to do so.

\myparagraph{Weighted objective} One approach is to assume that a (application dependent) trade-off parameter $\lambda \geq 0$ is predetermined, so that the learner has to approximately minimize
\begin{align}
L_{\lambda}(h) = \mathcal{R}(h) + \lambda \mathcal{D}(h).
\end{align}
\tosee{The value of $\lambda$ will likely be application-dependent and to be determined by the entity issuing the learner. There are various legal and ethical considerations that may apply when determining a desired accuracy-fairness trade-off \citep{barocas-hardt-narayanan}. Therefore, we leave $\lambda$ as an arbitrary, but fixed parameter, similarly to \cite{menon2018cost}, and we assume that $\lambda$ is known by both the learner and the adversary.}

For a given value of $\lambda$, the quality of the hypothesis $\mathcal{L}(h^S)$ can be directly measured via $L_{\lambda}(\mathcal{L}(h^S)) - \min_{h\in\mathcal{H}}L_{\lambda}(h)$. We will use $L_{\lambda}^{par}$ and $L_{\lambda}^{opp}$ to denote the weighted objectives with $\mathcal{D}^{par}$ and $\mathcal{D}^{opp}$ respectively.

\myparagraph{Element-wise comparisons} Alternatively, one may want to consider the two objectives independently. Given a classifier $h\in\mathcal{H}$, denote by $\ObjVec(h) = \left(\mathcal{R}(h), \mathcal{D}(h)\right)$ the vector consisting of the values of the two objectives. Note that $\ObjVec$ does not, in general, induce a total order on $\mathcal{H}$. Instead we can only compare two classifiers $h_1, h_2\in\mathcal{H}$ if, say, $h_1$ dominates $h_2$ in the sense that both $\mathcal{R}(h_1) \leq \mathcal{R}(h_2)$ and $\mathcal{D}(h_1) \leq \mathcal{D}(h_2)$. We denote this relation by $\ObjVec(h_1) \preceq \ObjVec(h_2)$. As we will see, these component-wise comparisons are still useful for understanding the limits of learning against an adversary.

Since $\mathcal{R}(h)$ and $\mathcal{D}(h)$ are two independent objectives, there is no clear notion of an ``optimal'' classifier under the $\preceq$ relation in general. Therefore, when studying the pairwise objective $ \ObjVec(h)$, we will assume there exists a classifier that is optimal both in terms of fairness and accuracy. Then this classifier is optimal also under the $\preceq$ relation and hence the quality of any other hypotheses can be measured against it. 

Specifically, \tosee{for our analysis of the $\ObjVec$ objective in Section \ref{sec:comp-wise-upper-bounds}}, we assume that there exists a $h^* \in \mathcal{H}$, such that $h^* \in \argmin_{h\in\mathcal{H}}\mathcal{R}(h)$ and $h^* \in \argmin_{h\in\mathcal{H}}\mathcal{D}(h)$, so that $\ObjVec(h^*) \preceq \ObjVec(h)$ for all $h\in\mathcal{H}$. Then the quality of $\mathcal{L}(S^p)$ can be measured as the $\mathbb{R}^2$ vector
\begin{align}
\mathbf{L}(\mathcal{L}(S^p)) = \ObjVec(\mathcal{L}(S^p)) - \ObjVec(h^*).
\end{align}
As with the weighted objective, we use $\mathbf{L}^{par}(\mathcal{L}(S^p))$ and $\mathbf{L}^{opp}(\mathcal{L}(S^p))$ to denote the loss vector when demographic parity and equal opportunity are used respectively.

One particular situation that we will study in which a component-wise optimal classifier $h^*$ exists, is within the realizable PAC learning model with equal opportunity fairness. Indeed, whenever a classifier $h^* \in \mathcal{H}$ satisfies $\mathbb{P}(h^*(X) = Y) = 1$, we have that both $\mathcal{R}(h^*) = 0$ and $\mathcal{D}^{opp}(h^*) = 0$ and so $\mathbf{L}^{opp}(\mathcal{L}(S^p)) = \ObjVec^{opp}(\mathcal{L}(S^p))$. \tosee{More generally, the existence of $h^*$ is plausible whenever the equal opportunity fairness notion is considered, since it is known that this fairness notion generally aligns well with accuracy \citep{hardt2016equality}. We expect that our analysis of the $\ObjVec$ objective can be extended to situations where $h^*$ is only $\epsilon$-approximately optimal in both objectives, which would cover more real-world situations, and we deem that an interesting direction for future work.}

\subsection{The limits of fairness-aware learning against an adversary}
\label{sec:limits_of_learning}

\myparagraph{Lower and upper bounds analysis} Over the next sections we will be showing lower and upper bounds on $L_{\lambda}(\mathcal{L}(S^p))$ and $\mathbf{L}(\mathcal{L}(S^p))$, that is, the risk and the fairness deviation measure achieved by the learner when trained on the corrupted data. \textit{Our lower bounds} can be thought of as hardness results that describe a limit on how well the learner can perform against the adversary. These are based on explicit constructions of hard learning problems and adversaries that demonstrate these limitations. \textit{Our upper bounds} complement the hardness results by constructing learners that recover a classifier with guarantees on fairness and accuracy that match the lower bounds, for a wide range of learning problems and adversaries.

Crucial in these results is the ordering of the quantifiers. These matter not only for the comparison between the upper and the lower bounds, but also for the sake of formalizing the powers of the adversary and the learner. Recall that the learner only operates with knowledge of the corrupted dataset. At the same time, the adversary is assumed to know not only the clean data, but also the target distribution and the learner. Therefore, our lower bounds are structured as follows:
\begin{center}
\textit{For any learner $\mathcal{L}$ there exists a distribution $\mathbb{P}$ and an adversary $\mathcal{A}$,\\ such that with constant probability \ldots} 
\end{center} 

Note in particular that the adversary can be chosen after the learner is constructed and together with the distribution and it can therefore be tailored to their choice. At the same time, our upper bounds read as:
\begin{center}
\textit{There exists a learner $\mathcal{L}$, such that for any distribution $\mathbb{P}$, any adversary $\mathcal{A}$ and any $\delta\in (0,1)$, with probability at least $1-\delta$ \ldots} 
\end{center} 

Since the learner is fixed before the distribution and the adversary are, it has to work for any such pair. 

We note that all probability statements refer to the randomness in the full generation process of the dataset $S^p$, that is the randomness of the clean data, the marked points and the adversary. For a fixed clean data distribution $\mathbb{P}$ and a fixed adversary $\mathcal{A}$, we denote the distribution of $S^p$ as $\mathbb{P}^{\mathcal{A}}$. 

\myparagraph{Role of the hypothesis space} Learnability in our setup can be studied either as a property of any fixed hypothesis space, or as a property of a class of hypothesis spaces, for example the hypothesis spaces of finite size or finite VC dimension. However, one can easily see that for certain hypothesis spaces fairness can be satisfied trivially. For example, whenever $\mathcal{H}$ contains a classifier that is constant on the whole input space (that is, always predicts $1$ or always predicts $0$), a learner that returns this constant classifier, regardless of the observed data, will always be perfectly fair with respect to both fairness notions, under any distribution and against any adversary. We therefore opt to study the learnability of \textit{classes of hypothesis spaces.} 

In particular, our hardness results demonstrate the \textit{existence of a finite hypothesis space}, such that a certain amount of excess inaccuracy and/or unfairness is unavoidable. Therefore, no learner can achieve better guarantees on the class of all finite hypothesis spaces, even in the infinite training data limit. This is contrast to, for example, classic PAC learning with clean data, where the ERM algorithm is a PAC learner for all finite hypothesis spaces and more generally all spaces of finite VC dimension \citep{shalev2014understanding}. 

On the other hand, the learners we construct for the upper bounds are shown to work for \textit{any hypothesis space} that is finite or of finite VC dimension, in all cases matching the lower bounds.

\myparagraph{Parameters of the learning problem} Our bounds will depend explicitly on the corruption ratio $\alpha$ and on the smaller of the protected class frequencies $P_0 = \mathbb{P}(A = 0)$ (for demographic parity) or on $P_{10} = \mathbb{P}(Y = 1, A = 0) \leq \mathbb{P}(Y = 1, A = 1)$ (for equal opportunity). To understand the limits of fairness-aware learning against a malicious adversary, we will analyze our bounds for small values of $\alpha$ and $P_0$ or $P_{10}$. Intuitively, the smaller the corruption rate $\alpha$ is, the easier it is for the learner to recover an accurate and fair hypothesis. On the other hand, a small value for $P_0$ or $P_{10}$ implies that one of the subgroups is underrepresented in the population, and so intuitively the adversary can hide a lot of information about this group and thus prevent the learner from finding a fair hypothesis. 

As we will see, this intuition is reflected in our bounds, which give a tool for understanding the effect of these quantities on the hardness of the learning problem. Comparing the lower bounds, which hold regardless of the sample size $n$, to the upper bounds in the limit of $n\to \infty$ allows us to reason about the absolute limits of fairness-aware learning against a malicious adversary. Indeed, in this large data limit, we find that our upper and lower bounds match in terms of their dependence on $\alpha$ and $P_0$ or $P_{10}$ up to constant factors. We note that designing algorithms that achieve \textit{sample-optimal} guarantees in our context is beyond the scope of this work. However, we will also be interested in the \textit{statistical rates of convergence} of the studied learners to the irreducible gap certified by the lower bounds. We refer to Section \ref{sec:comp-wise-upper-bounds} for a formal treatment.

%% file: lower_bounds.tex
We now present a series of hardness results that demonstrate that fair learning in the presence of a malicious adversary is provably impossible in a PAC learning sense. \textbf{Complete proofs of all results in this section can be found in Appendix \ref{sec:appendix-lower-bounds-proofs}.}

\subsection{Pareto lower bounds}
\label{sec:pareto-lower-bounds}
We begin by presenting two hardness results that intuitively show that for some hypothesis spaces $\mathcal{H}$ the adversary can prevent any learner from reaching the Pareto front of the accuracy-fairness optimization problem. We first demonstrate this for demographic parity:

\begin{theorem}
\label{thm:lower-bound-demog-par-with-p}
Let $0 \leq \alpha < 0.5, 0 < P_0 \leq 0.5$. For any input set $\mathcal{X}$ with at least four distinct points, there exists a finite hypothesis space $\mathcal{H}$, such that for any learning algorithm $\mathcal{L}:\cup_{n\in\mathbb{N}}(\prodspace)^n \to \mathcal{H}$, there exists a distribution $\mathbb{P}$ \textit{for which $\mathbb{P}(A = 0) = P_0$}, a malicious adversary $\mathcal{A}$ of power $\alpha$ and a hypothesis $h^* \in \mathcal{H}$, such that with probability at least $0.5$
\begin{equation*}
\mathcal{R}(\mathcal{L}(S^p), \mathbb{P}) - \mathcal{R}(h^*, \mathbb{P}) \geq \min\left\{\frac{\alpha}{1-\alpha}, 2P_0P_1\right\}
\end{equation*}
 and 
\begin{align*}
\mathcal{D}^{par}(\mathcal{L}(S^p), \mathbb{P}) - \mathcal{D}^{par}(h^*, \mathbb{P}) \geq \min\left\{\frac{\alpha}{2P_0P_1(1-\alpha)}, 1\right\}.
\end{align*}
\end{theorem}
The proof of this theorem (as well as of the other hardness results presented in this section) is based on the so-called \textit{method of induced distributions}, pioneered by \citep{kearns1993learning}. The idea is to construct two distributions that are sufficiently different, so that different classifiers perform well on each, yet can be made indistinguishable after the modifications of the adversary. Then no fixed learner with access only to the corrupted data can be ``correct'' with high probability on both distributions and so any learner will incur an excessively high loss and exhibit excessively high unfairness on at least one of them, regardless of the amount of available data.

Here we provide a sketch proof of Theorem \ref{thm:lower-bound-demog-par-with-p}, to illustrate the type of construction used. A complete proof can be found in Appendix \ref{sec:appendix-lower-bounds-proofs}.

\begin{proof}
\textbf{(Sketch)} Let $\eta = \frac{\alpha}{1-\alpha}$, so that $\alpha = \frac{\eta}{1+\eta}$. We assume here that $\eta = \frac{\alpha}{1-\alpha} \leq 2P_0(1-P_0)$, with the other case following from a similar construction, but with an adversary that uses a smaller value of $\alpha$ (so that it leaves some of the data points at its disposal untouched).

Take four distinct points $\{x_1, x_2, x_3, x_4\}\in\mathcal{X}$. We consider two distributions $\mathbb{P}_0$ and $\mathbb{P}_1$, where each $\mathbb{P}_i$ is defined as
\begin{equation*}
  \mathbb{P}_{i}(x,a,y) =
    \begin{cases}
      1 - P_0 - \eta/2 & \text{if $x = x_1, a = 1, y = 1$}\\
      P_0 - \eta/2 & \text{if $x = x_2, a = 0, y = 0$}\\
      \eta/2  & \text{if $x = x_3, a = i, y = \lnot i$}\\
      \eta/2  & \text{if $x = x_4, a = \lnot i, y = i$}\\
      0 & \text{otherwise}
    \end{cases}       
\end{equation*}
Note that these are valid distributions, since $\eta \leq 2P_0(1-P_0) \leq 2P_0 \leq 2(1-P_0)$ by assumption and also that $P_0 = \mathbb{P}_i(A = 0)$ for both $i\in\{0,1\}$.  Consider the hypothesis space $\mathcal{H} = \{h_0, h_1\}$, with $$h_0(x_1) = 1 \quad h_0(x_2) = 0 \quad h_0(x_3) = 1 \quad h_0(x_4) = 0$$ and
$$h_1(x_1) = 1 \quad h_1(x_2) = 0 \quad h_1(x_3) = 0 \quad h_1(x_4) = 1.$$
The point of this construction is as follows: there are only two points, $x_3$ and $x_4$, where the two distributions differ. This is also where the classifiers differ and, in fact, each classifier $h_i$ is better performing on the distribution $\mathbb{P}_i$, in both accuracy and fairness, than the other classifier.

Indeed, it is easy to verify that 
\begin{equation}
\label{eqn:lower_bound_proof_loss}
L(h_{\lnot i}, \mathbb{P}_i) - L(h_i, \mathbb{P}_i) = \eta, \quad \text{for both } i = 0,1.
\end{equation}
Moreover, 
\begin{equation}
\label{eqn:lower_bound_proof_fairness}
\mathcal{D}^{par}(h_{\lnot i}, \mathbb{P}_i) - \mathcal{D}^{par}(h_i, \mathbb{P}_i) = \frac{\eta}{2P_0 (1-P_0)}, \quad \text{for both } i = 0,1.
\end{equation}
Now what the adversary does is to use all of the marked data to insert points with inputs $x_3$ and $x_4$, but with flipped labels and protected attributes. Then, since the points with inputs $x_3$ and $x_4$ in the original data are sufficiently rare, the adversary manages to hide which one of the two distributions was the original one. 

Specifically, consider a (randomized) malicious adversary $\mathcal{A}_{i}$ of power $\alpha$, that given a clean distribution $\mathbb{P}_{i}$, changes every marked point to $(x_3,\lnot i, i)$ with probability $0.5$ and to $(x_4,i, \lnot i)$ otherwise. Under a distribution $\mathbb{P}_{i}$ and an adversary $\mathcal{A}_{i}$, the probability of seeing a point $(x_3, i, \lnot i)$ is $\frac{\eta}{2} (1-\alpha) = \frac{\eta}{2} \frac{1}{1 + \eta} = \alpha/2$, which is equal to the probability of seeing a point $(x_3, \lnot i, i)$. Therefore, denoting the probability distribution of the corrupted dataset, under a clean distribution $\mathbb{P}_{i}$ and an adversary $\mathcal{A}_{i}$, by $\mathbb{P}'_{i}$, one can verify that $\mathbb{P}'_{0} = \mathbb{P}'_{1}$, so the two initial distributions $\mathbb{P}_0$ and $\mathbb{P}_1$ become indistinguishable under the adversarial manipulation.

The proof concludes by formalizing the observation that any fixed learner $\mathcal{L}: \cup_{n\in \mathbb{N}} (\prodspace)^n \to \{h_0, h_1\}$ will perform poorly on at least one of the distribution-adversary pairs $(\mathbb{P}_i, \mathcal{A}_i)$, since the resulting corrupted data distributions are the same, but the optimal classifiers differ.
\end{proof}

\myparagraph{Discussion} Our hardness result implies that no learner can guarantee reaching a point on the Pareto front in a PAC learning sense, even for a simple family of hypothesis spaces, namely the finite ones.
This is because the adversary can force the learner to return a hypothesis that is \textit{a constant away from optimality} is both objectives, with a \textit{non-vanishing probability}\footnote{\tosee{The constant $0.5$ for the probability of the adversary succeeding is perfectly sufficient for proving the impossibility of PAC learnability. A more refined analysis may yield an even larger constant, although we have not explored this further.}}.
To prove the theorem we explicitly construct a hypothesis space that is not learnable against the malicious adversary. As discussed in Section \ref{sec:limits_of_learning}, a constructive proof is necessary here, because fairness can be trivially satisfied on some hypothesis spaces, for example those that contain a constant classifier, which is fair under any distribution and against any adversary.

We now analyze the bounds and their behavior for small values of $\alpha$ and $P_0$. First assume that $\frac{\alpha}{1 - \alpha} < 2P_{0}P_1$, which in particular is the case whenever  $2\alpha < P_0$. Then under the conditions of the theorem, with probability at least $0.5$\footnote{We use the $\Omega$-notation for lower bounds on the growth rates of functions.}
\begin{align}
\label{eqn:lower-bound-demog-parity-loss-rates}
\mathcal{R}(\mathcal{L}(S^p)) - \mathcal{R}(h^*) \geq \Omega\left(\alpha\right)
\end{align}
and
\begin{align}
\label{eqn:lower-bound-demog-parity-rates}
\mathcal{D}^{par}(\mathcal{L}(S^p)) - \mathcal{D}^{par}(h^*) \geq \Omega\left(\frac{\alpha}{P_0}\right).
\end{align}
The lower bound on the excess loss (\ref{eqn:lower-bound-demog-parity-loss-rates}) is known to hold for any hypothesis space as shown by \cite{kearns1993learning}. What Theorem \ref{thm:lower-bound-demog-par-with-p} adds to this classic result is that for certain hypothesis spaces: 1) the learner can at the same time be forced to produce an excessively unfair classifier; 2) the fairness deviation measure $\mathcal{D}^{par}$ can be increased by $\Omega(\alpha/P_0)$. Note that \textit{these results hold regardless of the sample size $n$}.

Equations (\ref{eqn:lower-bound-demog-parity-loss-rates}) and (\ref{eqn:lower-bound-demog-parity-rates}) immediately imply the following lower bounds on $L_{\lambda}$ and $\mathbf{L}^{par}$:
\begin{equation}
\label{eqn:lower-bounds-demog-par-lambda-weighted}
L_{\lambda}^{par}(L(S^p)) - \min_{h\in\mathcal{H}}\mathcal{L}_{\lambda}^{par}(h) \geq \Omega\left(\alpha + \lambda \frac{\alpha}{P_0}\right).
\end{equation}
\begin{equation}
\label{eqn:lower-bounds-demog-par-component-wise}
\mathbf{L}^{par}(\mathcal{L}(S^p)) \succeq \left(\Omega\left(\alpha\right),\Omega\left(\frac{\alpha}{P_{0}}\right)\right)
\end{equation}
In the second case, when $\frac{\alpha}{1-\alpha} \geq 2P_0P_1$, the adversary can force a constant increase in the loss and make the classifier completely unfair, so that $\mathcal{D}^{par}(\mathcal{L}(S^p)) = 1$.  These observations, combined with the rates from the first case, indicate that unless $\alpha = o(P_0)$, the adversary can ensure that the resulting model's demographic parity deviation measure is constant. In particular, \textit{if one of the protected groups is rare, even very small levels of data corruption can lead to a biased model}.

Next we show a similar result for equal opportunity.
\begin{theorem}
\label{thm:lower-bound-equal-opp-with-p}
Let $0 \leq \alpha < 0.5$ and $P_{10} \leq P_{11} < 1$ be such that $P_{10} +  P_{11} < 1$. For any input set $\mathcal{X}$ with at least five distinct points, there exists a finite hypothesis space $\mathcal{H}$, such that for any learning algorithm $\mathcal{L}:\cup_{n\in\mathbb{N}}(\prodspace)^n \to \mathcal{H}$, there exists a distribution $\mathbb{P}$ \textit{for which $\mathbb{P}(A = a, Y = 1) = P_{1a}$ for $a\in\{0,1\}$}, a malicious adversary $\mathcal{A}$ of power $\alpha$ and a hypothesis $h^* \in \mathcal{H}$, such that with probability at least $0.5$ $$\mathcal{R}(\mathcal{L}(S^p), \mathbb{P}) - \mathcal{R}(h^*, \mathbb{P}) \geq \min\left\{\frac{\alpha}{1-\alpha}, 2P_{10}, 2(1-P_{10} - P_{11})\right\}$$ and $$\mathcal{D}^{opp}(\mathcal{L}(S^p), \mathbb{P}) -\mathcal{D}^{opp}(h^*, \mathbb{P}) \geq \min\left\{\frac{\alpha}{2(1-\alpha)P_{10}}, 1, \frac{1 - P_{10} - P_{11}}{P_{10}}\right\}.$$
\end{theorem}
\myparagraph{Discussion} A similar analysis to the one after Theorem \ref{thm:lower-bound-demog-par-with-p} applies here as well. In particular, whenever $\frac{\alpha}{1-\alpha} \leq 2\min\left\{P_{10}, 1 - P_{10} - P_{11}\right\}$, we obtain
\begin{equation}
\label{eqn:lower-bounds-equal-opp-lambda-weighted}
L_{\lambda}^{opp}(\mathcal{L}(S^p)) - \min_{h\in\mathcal{H}} L_{\lambda}^{opp}(h) \geq \Omega\left(\alpha + \lambda \frac{\alpha}{P_{10}}\right)
\end{equation}
\begin{equation}
\label{eqn:lower-bounds-equal-opp-component-wise}
\mathbf{L}^{opp}(\mathcal{L}(S^p)) \succeq \left(\Omega\left(\alpha\right),\Omega\left(\frac{\alpha}{P_{10}}\right)\right)
\end{equation}
The case when $\frac{\alpha}{1-\alpha} > 2\min\left\{P_{10}, 1 - P_{10} - P_{11}\right\}$ leads to a constant equal opportunity deviation measure. If in addition we have that $1 - P_{10} - P_{11} \geq P_{10}$, a completely unfair classifier will be returned. Consequently, if positive examples associated with one of the protected groups are rare (that is, if $P_{10} = \mathbb{P}(Y=1, A =0)$ is small), then even very small corruption ratios can lead to a biased model. 

\subsection{Hurting fairness without affecting accuracy}
While the results above shed light on the fundamental limits of robust fairness-aware learning against an adversary, models that are inaccurate are often easy to detect in practice. On the other hand, a model that has good accuracy, but exhibits a bias with respect to the protected attribute, can be much more problematic. This is especially true in applications where demographic data is not collected at prediction time for privacy reasons. In this case the model's bias might go unnoticed for a long time, thus adversely affecting one of the population subgroups and potentially extrapolating existing biases from the training data to future decisions.

We now show that such an unfortunate situation is indeed also possible under the malicious adversary model. The following results show that any learner will, in some situations, be forced by the adversary to return a model that is optimal in terms of accuracy, but exhibits unnecessarily high unfairness in terms of demographic parity. 
\begin{theorem}
\label{thm:lower-bound-demog-par-with-p-good-loss}
Let $0 \leq \alpha < 0.5, 0 < P_0 \leq 0.5$. For any input set $\mathcal{X}$ with at least four distinct points, there exists a finite hypothesis space $\mathcal{H}$, such that for any learning algorithm $\mathcal{L}:\cup_{n\in\mathbb{N}}(\prodspace)^n \to \mathcal{H}$, there exists a distribution $\mathbb{P}$ \textit{for which $\mathbb{P}(A = 0) = P_0$}, a malicious adversary $\mathcal{A}$ of power $\alpha$ and a hypothesis $h^* \in \mathcal{H}$, such that with probability at least $0.5$
\begin{equation*}
\mathcal{R}(\mathcal{L}(S^p), \mathbb{P}) = \mathcal{R}(h^*, \mathbb{P}) = \min_{h\in\mathcal{H}}\mathcal{R}(h, \mathbb{P})
\end{equation*}
 and 
\begin{equation*}
\mathcal{D}^{par}(\mathcal{L}(S^p), \mathbb{P}) - \mathcal{D}^{par}(h^*, \mathbb{P}) \geq \min\left\{\frac{\alpha}{2P_0}, 1\right\}.
\end{equation*}
\end{theorem}
We also present a corresponding result for equal opportunity.
\begin{theorem}
\label{thm:lower-bound-equal-opp-with-p-good-loss}
Let $0 \leq \alpha < 0.5, P_{10} \leq P_{11} < 1$ be such that $P_{10} + P_{11} < 1$. For any input set $\mathcal{X}$ with at least five distinct points, there exists a finite hypothesis space $\mathcal{H}$, such that for any learning algorithm $\mathcal{L}:\cup_{n\in\mathbb{N}}(\prodspace)^n \to \mathcal{H}$, there exists a distribution $\mathbb{P}$ \textit{for which $\mathbb{P}(A = a, Y = 1) = P_{1a}$ for $a\in\{0,1\}$}, a malicious adversary $\mathcal{A}$ of power $\alpha$ and a hypothesis $h^* \in \mathcal{H}$, such that with probability at least $0.5$
\begin{align*}
\mathcal{R}(\mathcal{L}(S^p), \mathbb{P}) = \mathcal{R}(h^*, \mathbb{P}) = \min_{h\in\mathcal{H}} \mathcal{R}(h, \mathbb{P})
\end{align*}
and 
\begin{align*}
\mathcal{D}^{opp} (\mathcal{L}(S^p), \mathbb{P}) - \mathcal{D}^{opp}(h^*, \mathbb{P}) \geq \min\left\{\frac{\alpha}{2(1-\alpha)P_{10}}\left(1-\frac{P_{10}}{P_{11}}\right), 1-\frac{P_{10}}{P_{11}}\right\}.
\end{align*}
\end{theorem}
Once again the error terms on the fairness notions are inversely proportional to $P_0$ and $P_{10}$ respectively, indicating that datasets in which one of the subgroups is underrepresented are particularly vulnerable to data manipulations. \tosee{In Theorem \ref{thm:lower-bound-equal-opp-with-p-good-loss} an additional multiplicative factor of $1 - \frac{P_{10}}{P_{11}}$ appears - while we believe this to be an artifact of the proof technique and not inherent, we do not currently have a lower bound construction that circumvents this term. However, considering the asymptotic behavior where $\alpha \to 0, P_{10} \to 0$, but $P_{11} = \Theta(1)$, this additional term is negligible.}

%% file: upper_bounds.tex
We now prove that the (sample-size-independent) lower bounds from the previous section are tight up to constant factors, by providing matching upper bounds for the same problem. We do so by studying the performance of two natural types of fairness-aware learning algorithms under the malicious adversary model. We find that these algorithms achieve order-optimal performance in the large data regime.

\textbf{Complete proofs of all results in this section can be found in Appendix \ref{sec:appendix-upper-bounds-proofs}.} A sketch of the proofs is also presented in Section \ref{sec:upper_bounds_proof_sketch}.

\subsection{Upper bounds on the $\lambda$-weighted objectives}
\label{sec:upper-bounds-lambda-objectives}

The first type of algorithms we study simply minimize an empirical estimate of the $\lambda$-weighted objective $L_{\lambda}$. We show that with high probability such learners achieve an order-optimal deviation from $\min_{h\in\mathcal{H}}L_{\lambda}(h)$ in the large data regime, as long as $\mathcal{H}$ has a finite VC dimension. 

\tosee{Throughout this section we assume that $\lambda > 0$ is an arbitrary, but fixed parameter, chosen depending on domain-specific considerations (see also Section \ref{sec:multi-objective-adv-learning}).}

\myparagraph{Bound for demographic parity} Let $h\in\mathcal{H}$ be a fixed hypothesis. We consider the following natural estimate of $\mathcal{D}^{par}(h)$, as given in equation (\ref{eqn:demog-parity-deviation}), based on the corrupted dataset $S^p = \{(x^p_i, a^p_i, y^p_i)\}_{i=1}^n$:
\begin{align}
\label{eqn:demog-parity-estimate}
\widehat{\mathcal{D}}^{par}(h) = \left|\frac{\sum_{i=1}^n \mathbbm{1}\{h(x^p_i) = 1, a^p_i = 0\}}{\sum_{i=1}^n \mathbbm{1}\{a^p_i = 0\}} - \frac{\sum_{i=1}^n \mathbbm{1}\{h(x^p_i) = 1, a^p_i = 1\}}{\sum_{i=1}^n \mathbbm{1}\{a^p_i = 1\}}\right|,
\end{align}
with the convention that $\frac{0}{0} = 0$ for the purposes of this definition. We also denote the empirical risk of $h$ on $S^p$ by $\widehat{R}^p(h) = \frac{1}{n}\sum_{i=1}^n \mathbbm{1}\{h(x^p_i) \neq y^p_i\}$.

Suppose that the learner $\mathcal{L}^{par}_{\lambda}:\cup_{n = 1}^{\infty} (\prodspace)^n \to \mathcal{H}$ is such that 
\begin{align*}
\label{eqn:learner-agnostic}
\mathcal{L}^{par}_{\lambda}(S^p) \in \argmin_{h\in\mathcal{H}}(\widehat{R}^p(h) + \lambda \widehat{\mathcal{D}}^{par}(h)), \quad \text{for all } S^p.
\end{align*}
That is, $\mathcal{L}^{par}_{\lambda}$ always returns a minimizer of the $\lambda$-weighted empirical objective. Then the following result holds.

\begin{theorem}
\label{thm:agnostic-upper-bound-demog-par}
Let $\mathcal{H}$ be any hypothesis space with $d = VC(\mathcal{H}) < \infty$. Let $\mathbb{P}\in \mathcal{P}(\prodspace)$ be a fixed distribution and let $\mathcal{A}$ be any malicious adversary of power $\alpha < 0.5$. Denote by $\mathbb{P}^{\mathcal{A}}$ the probability distribution of the corrupted data $S^p$, under the random sampling of the clean data, the marked points and the randomness of the adversary. Then for any $\delta\in(0,1)$ and $n \geq \max \left\{\frac{8\log(16/\delta)}{(1-\alpha)P_0}, \frac{12\log(12/\delta)}{\alpha}, \frac{d}{2}\right\}$, we have:
\begin{align*}
\mathbb{P}^{\mathcal{A}}\left(L_{\lambda}^{par}\big(\widehat{h}\big)  \leq \min_{h\in\mathcal{H}} L_{\lambda}^{par}(h) + \Delta^{par}_{\lambda}\right) > 1 - \delta,
\end{align*}
where \tosee{$\widehat{h} \coloneqq \mathcal{L}^{par}_{\lambda}(S^p)$ is the hypothesis returned by the learner, $L_{\lambda}^{par}(h) = \mathcal{R}(h) + \lambda \mathcal{D}^{par}(h)$ is the $\lambda$-weighted objective and} \footnote{The $\widetilde{\mathcal{O}}$-notation hides constant and logarithmic factors.}
\begin{align*}
\Delta^{par}_{\lambda} = 3\alpha + \lambda (2\Delta^{par}) + \widetilde{\mathcal{O}}\left(\sqrt{\frac{d}{n}} + \lambda \sqrt{\frac{d}{P_0n}}\right)
\end{align*}
and 
\begin{align*}
\Delta^{par} = \frac{2\alpha}{\frac{P_0}{3} + \alpha} = \mathcal{O}\left(\frac{\alpha}{P_0}\right).
\end{align*}
\end{theorem}
This result shows that for any $\mathcal{H}$ of finite VC dimension, any distribution $\mathbb{P}$ and against any malicious adversary $\mathcal{A}$ of power $\alpha$, the learner $\mathcal{L}^{par}_{\lambda}$ is able, for sufficiently large values of the sample size $n \geq \Omega((P_0/\alpha)^2)$, to return with high probability a hypothesis $\widehat{h}$ such that
\begin{equation}
\label{eqn:rates-lambda-weighted-upper-bound-par}
L_{\lambda}^{par}\big(\widehat{h}\big) - \min_{h\in\mathcal{H}} L_{\lambda}^{par}(h) \leq \mathcal{O}\left(\alpha + \lambda \frac{\alpha}{P_0}\right).
\end{equation}
Note that these rates on the irreducible error term match our lower bound from Theorem \ref{thm:lower-bound-demog-par-with-p} and Inequality (\ref{eqn:lower-bounds-demog-par-lambda-weighted}). Indeed, the hardness result shows that no algorithm can guarantee better error rates than those in (\ref{eqn:rates-lambda-weighted-upper-bound-par}) on the family of finite hypothesis sets and hence also on the hypothesis sets with finite VC dimension.

\myparagraph{Bound for equal opportunity} Similarly, we consider the following estimate for the equal opportunity deviation measure:
\begin{align}
\label{eqn:equal-opp-estimate}
\widehat{\mathcal{D}}^{opp}(h) = \left|\frac{\sum_{i=1}^n \mathbbm{1}\{h(x^p_i) = 1, a^p_i = 0, y^p_i = 1\}}{\sum_{i=1}^n \mathbbm{1}\{a^p_i = 0, y^p_i = 1\}} - \frac{\sum_{i=1}^n \mathbbm{1}\{h(x^p_i) = 1, a^p_i = 1, y^p_i = 1\}}{\sum_{i=1}^n \mathbbm{1}\{a^p_i = 1, y^p_i = 1\}}\right|,
\end{align}
with the convention that $\frac{0}{0} = 0$ for the purposes of this definition. Suppose that a learner $\mathcal{L}^{opp}_{\lambda}:\cup_{n = 1}^{\infty} (\prodspace)^n \to \mathcal{H}$ is such that 
\begin{equation*}
\mathcal{L}^{opp}_{\lambda}(S^p) \in \argmin_{h\in\mathcal{H}}(\widehat{R}^p(h) + \lambda \widehat{\mathcal{D}}^{opp}(h)), \quad \text{for all } S^p,
\end{equation*}
that is, always returns a minimizer of the $\lambda$-weighted empirical objective. Then:
\begin{theorem}
\label{thm:agnostic-upper-bound-equal-opp}
Let $\mathcal{H}$ be any hypothesis space with $d = VC(\mathcal{H}) < \infty$. Let $\mathbb{P}\in \mathcal{P}(\prodspace)$ be a fixed distribution and let $\mathcal{A}$ be any malicious adversary of power $\alpha < 0.5$. Then for any $\delta\in(0,1)$ and $n \geq \max \left\{\frac{8\log(16/\delta)}{(1-\alpha) P_{10}}, \frac{12\log(12/\delta)}{\alpha}, \frac{d}{2}\right\}$
\begin{align*}
\label{eqn:agnostic-upper-bound-equal-opp}
\mathbb{P}^{\mathcal{A}}\left(L_{\lambda}^{opp}\big(\widehat{h}\big)  \leq \min_{h\in\mathcal{H}} L_{\lambda}^{opp}(h) + \Delta^{opp}_{\lambda}\right) > 1 - \delta,
\end{align*}
where $\widehat{h}\coloneqq \mathcal{L}^{opp}_{\lambda}(S^p)$ is the hypothesis returned by the learner, $L_{\lambda}^{opp}(h) = \mathcal{R}(h) + \lambda \mathcal{D}^{opp}(h)$ and
\begin{align*}
\Delta^{opp}_{\lambda} = 3\alpha + \lambda (2\Delta^{opp}) + \widetilde{\mathcal{O}}\left(\sqrt{\frac{d}{n}} + \lambda \sqrt{\frac{d}{P_{10}n}}\right)
\end{align*}
and 
\begin{align*}
\Delta^{opp} = \frac{2\alpha}{\frac{P_{10}}{3} + \alpha} = \mathcal{O}\left(\frac{\alpha}{P_{10}}\right).
\end{align*}
\end{theorem}
Again, for a sufficiently large sample size, this result implies an upper bound on the excess loss of the hypothesis $\widehat{h} \coloneqq \mathcal{L}^{opp}_{\lambda}(S^p)$ returned by the learner in terms of the weighted objective
\begin{equation}
\label{eqn:rates-lambda-weighted-upper-bound-opp}
L_{\lambda}^{opp}\big(\widehat{h}\big) - \min_{h\in\mathcal{H}} L_{\lambda}^{opp}(h) \leq \mathcal{O}\left(\alpha + \lambda \frac{\alpha}{P_{10}}\right),
\end{equation}
which is again order optimal, according to Theorem \ref{thm:lower-bound-equal-opp-with-p} and Inequality (\ref{eqn:lower-bounds-equal-opp-lambda-weighted}).

\subsection{Component-wise upper bounds}
\label{sec:comp-wise-upper-bounds}
We now introduce a second type of algorithms, which return a hypothesis that achieves both a small loss and a small fairness deviation measure on the training data, or, if no such hypothesis exists, a random hypothesis. We show that, in the case when there exists a classifier that is optimal in both accuracy and fairness, with high probability such learners return a hypothesis $h\in\mathcal{H}$ that is order-optimal in both elements of the objective vector $\mathbf{L}(h)$, as long as $\mathcal{H}$ is of finite VC dimension and $n$ is sufficiently large. Finally, in the case of realizable PAC learning with equal opportunity fairness, we are able to provide an algorithm that achieves such order-optimal guarantees with \textit{fast statistical rates}, for any finite hypothesis space.

\tosee{Throughout the section only, we assume that there exists a classifier $h^* \in \mathcal{H}$, such that $\ObjVec(h^*) \preceq \ObjVec(h)$ for all $h\in \mathcal{H}$.} That is, $\mathcal{R}(h^*) \leq \mathcal{R}(h)$ and $\mathcal{D}(h^*) \leq \mathcal{D}(h)$ for all $h\in\mathcal{H}$. We also assume that $d = VC(\mathcal{H}) < \infty$.

We note that the algorithms studied in this section require the knowledge of $\alpha$ and of $P_{0}$ and $P_{10}$ for demographic parity and equal opportunity respectively, since they explicitly use these quantities when selecting a hypothesis. Even if these quantities are unknown in advance, estimates can often be obtained in practice, for example by having the quality of a small random subset of the data $S^p$ verified by a trusted authority, or via conducting an additional survey/crowdsourcing experiment. 

\myparagraph{Bound for demographic parity}
Given a corrupted dataset $S^p = \{(x^p_i, a^p_i, y^p_i)\}$, let $\widehat{h}^{r} \in \argmin_{h\in\mathcal{H}} \widehat{\mathcal{R}}^p(h)$ and $\widehat{h}^{par} \in \argmin_{h\in\mathcal{H}} \widehat{\mathcal{D}}^{par}(h)$. Further, we define the sets
\begin{align*}
\mathcal{H}_1 & = \left\{h\in\mathcal{H}: \widehat{\mathcal{R}}^p(h) - \widehat{\mathcal{R}}^p(\widehat{h}^{r}) \leq 3\alpha + 4\sqrt{\frac{8d\log(\frac{en}{d}) + 2\log(16/\delta)}{n}} \right\} \\
\mathcal{H}_2 & = \left\{h\in\mathcal{H}: \widehat{\mathcal{D}}^{par}(h) - \widehat{\mathcal{D}}^{par}(\widehat{h}^{par}) \leq 2\Delta^{par} + 32 \sqrt{\frac{2d\log(\frac{2en}{d}) + 2\log(96/\delta)}{(1-\alpha)P_0n}} \right\}.
\end{align*}
That is, $\mathcal{H}_1$ and $\mathcal{H}_2$ are the sets of classifiers that are not far from optimal on the train data, in terms of their risk and their fairness respectively. The upper bound terms are selected according to the concentration properties of the two measures and describe the amount of expected variability of those, due to the randomness of the training data. Now define the \textit{component-wise learner}:
\begin{align*}
\mathcal{L}^{par}_{cw}(S^p) = 
\begin{cases}
    \text{any }h \in \mathcal{H}_1 \cap \mathcal{H}_2, & \text{if } \mathcal{H}_1 \cap \mathcal{H}_2 \neq \emptyset\\
    \text{any }h \in \mathcal{H}, & \text{otherwise,}
\end{cases}
\end{align*}
that returns a classifier that is good in both metrics, if such exists, or an arbitrary classifier otherwise.

\tosee{Intuitively, whenever a classifier is an element of $\mathcal{H}_1 \cap \mathcal{H}_2$, it performs relatively well on the data in terms of both accuracy and fairness, thereby being a good candidate for learning. On the other hand, situations where no such classifier exists are expected to be rare. This is because $h^*$ is optimal in both metrics on the true data distribution and so it is also likely to perform close to optimal on the corrupted data, allowing for variations due to finite sample effects and data corruption. Therefore, with high probability $h^* \in \mathcal{H}_1 \cap \mathcal{H}_2$, so that the intersection is non-empty.}

Formally, the following result holds.

\begin{theorem}
\label{thm:comp-wise-upper-bound-demog-par}
Let $\mathcal{H}$ be any hypothesis space with $d = VC(\mathcal{H}) < \infty$. Let $\mathbb{P}\in \mathcal{P}(\prodspace)$ be a fixed distribution and let $\mathcal{A}$ be any malicious adversary of power $\alpha < 0.5$. Suppose that there exists a hypothesis $h^*\in\mathcal{H}$, such that $\ObjVec(h^*) \preceq \ObjVec(h)$ for all $h\in\mathcal{H}$. Then for any $\delta\in(0,1)$ and $n \geq \max \left\{\frac{8\log(16/\delta)}{(1-\alpha)P_0}, \frac{12\log(12/\delta)}{\alpha}, \frac{d}{2}\right\}$, with probability at least $1-\delta$:
\begin{align*}
\mathbf{L}^{par}\big(\widehat{h}\big)  \preceq \left(6\alpha + \widetilde{\mathcal{O}}\left(\sqrt{\frac{d}{n}}\right), 4\Delta^{par} +  \widetilde{\mathcal{O}}\left(\sqrt{\frac{d}{P_0 n}}\right) \right),
\end{align*}
where \tosee{$\widehat{h} \coloneqq \mathcal{L}^{par}_{cw}(S^p)$ is the hypothesis returned by the learner} and $$\mathbf{L}^{par}(\widehat{h}) = \left(\mathcal{R}(\widehat{h}) - \mathcal{R}(h^*), \mathcal{D}^{par}(\widehat{h}) - \mathcal{D}^{par}(h^*)\right).$$
\end{theorem}
Since $\Delta^{par} = \mathcal{O}\left(\frac{\alpha}{P_0}\right)$, in the large data limit we obtain that
\begin{align}
\label{eqn:upper-bound-demog-par-component-wise}
\mathbf{L}^{par}(\widehat{h}) \preceq \left(\mathcal{O}(\alpha), \mathcal{O}\left(\frac{\alpha}{P_{0}}\right)\right).
\end{align}
Note that this bound is order-optimal for the class of finite hypothesis spaces, and hence also for the class of hypothesis spaces with finite VC dimension, according to Theorem \ref{thm:lower-bound-demog-par-with-p} and Inequality (\ref{eqn:lower-bounds-demog-par-component-wise}).

\myparagraph{Bound for equal opportunity}
Similarly, let $\widehat{h}^{opp} \in \argmin_{h\in\mathcal{H}} \widehat{\mathcal{D}}^{opp}(h)$. Further, we define the set
\begin{align*}
\mathcal{H}_3 & = \left\{h\in\mathcal{H}: \widehat{\mathcal{D}}^{opp}(h) - \widehat{\mathcal{D}}^{opp}(\widehat{h}^{opp}) \leq 2\Delta^{opp} + 32 \sqrt{\frac{2d\log(\frac{2en}{d}) + 2\log(96/\delta)}{(1-\alpha)P_{10}n}} \right\}.
\end{align*}
That is, $\mathcal{H}_3$ is the set of classifiers that are not far from optimal on the train data, in terms of equal opportunity fairness. Now define the \textit{component-wise learner} for equal opportunity:
\begin{align*}
\mathcal{L}^{opp}_{cw}(S^p) = 
\begin{cases}
    \text{any }h \in \mathcal{H}_1 \cap \mathcal{H}_3, & \text{if } \mathcal{H}_1 \cap \mathcal{H}_3 \neq \emptyset\\
    \text{any }h \in \mathcal{H}, & \text{otherwise,}
\end{cases}
\end{align*}
that returns a classifier that is good in both metrics, if such exists, or an arbitrary classifier otherwise. Then the following result holds.

\begin{theorem}
\label{thm:comp-wise-upper-bound-equal-opp}
Let $\mathcal{H}$ be any hypothesis space with $d = VC(\mathcal{H}) < \infty$. Let $\mathbb{P}\in \mathcal{P}(\prodspace)$ be a fixed distribution and let $\mathcal{A}$ be any malicious adversary of power $\alpha < 0.5$. Suppose that there exists a hypothesis $h^*\in\mathcal{H}$, such that $\ObjVec(h^*) \preceq \ObjVec(h)$ for all $h\in\mathcal{H}$. Then for any $\delta\in(0,1)$ and $n \geq \max \left\{\frac{8\log(16/\delta)}{(1-\alpha) P_{10}}, \frac{12\log(12/\delta)}{\alpha}, \frac{d}{2}\right\}$, with probability at least $1-\delta$
\begin{align*}
\mathbf{L}^{opp}(\widehat{h})  \preceq \left(6\alpha + \widetilde{\mathcal{O}}\left(\sqrt{\frac{d}{n}}\right), 4\Delta^{opp} +  \widetilde{\mathcal{O}}\left(\sqrt{\frac{d}{P_{10} n}}\right) \right).
\end{align*}
where $\widehat{h} \coloneqq \mathcal{L}^{opp}_{cw}(S^p)$ is the hypothesis returned by the learner and $$\mathbf{L}^{opp}(\widehat{h}) = \left(\mathcal{R}(\widehat{h}) - \mathcal{R}(h^*), \mathcal{D}^{opp}(\widehat{h}) - \mathcal{D}^{opp}(h^*)\right).$$
\end{theorem}
Since $\Delta^{opp} = \mathcal{O}\left(\frac{\alpha}{P_{10}}\right)$,  in the large data limit we obtain that
\begin{align}
\label{eqn:upper-bound-demog-par-component-wise}
\mathbf{L}^{opp}(\widehat{h}) \preceq \left(\mathcal{O}(\alpha), \mathcal{O}\left(\frac{\alpha}{P_{10}}\right)\right).
\end{align}
Note that this bound is order-optimal for the class of finite hypothesis spaces, and hence also for the class of hypothesis spaces with finite VC dimension, according to Theorem \ref{thm:lower-bound-equal-opp-with-p} and Inequality (\ref{eqn:lower-bounds-equal-opp-component-wise}).

\myparagraph{Upper bound with fast rates} Finally, we study learning with the equal opportunity fairness notion, in the realizable PAC learning framework, where a perfectly accurate classifier exists. Given this additional assumption, we are able to certify \textit{convergence to an order-optimal error in both fairness and accuracy at fast statistical rates}. For simplicity we assume that $\mathcal{H}$ is finite here. 

Specifically, note that while the results presented already achieve order-optimal guarantees in the limit as $n\to \infty$, for a finite amount of samples they incur an additional loss of $\widetilde{\mathcal{O}}\left(\frac{1}{\sqrt{n}}\right)$. Regarding $P_0$ (for demographic parity) or $P_{10}$ (for equal opportunity) as fixed, all previous algorithms need $\widetilde{\mathcal{O}}\left(\frac{1}{\alpha^2}\right)$ samples to achieve an excess risk and fairness deviation measure of $\widetilde{\mathcal{O}}\left(\alpha\right)$. In contrast, the algorithm we present now only requires $\mathcal{O}\left(\frac{1}{\alpha}\right)$ samples.
 
Formally, assume that the underlying clean distribution $\mathbb{P}$ is such that there exists a $h^*\in\mathcal{H}$, for which $\mathbb{P}(h^*(X) = Y) = 1$. This implies that $L(h^*) = 0$ and $\mathcal{D}^{opp}(h^*) = 0$.

Key to the design of an algorithm that achieves fast statistical rates for the objective $\mathbf{L}$ are the following empirical estimates:
\begin{align}
\label{eqn:estimates-zero-conditional-probs}
\bar{\gamma}^p_{1a}(h) = \frac{\sum_{i=1}^n \mathbbm{1}\{h(x^p_i) = 0, a^p_i = a, y^p_1 = 1\}}{\sum_{i=1}^n \mathbbm{1}\{a^p_i = a, y^p_i = 1\}}
\end{align}
of $\bar{\gamma}_{1a}(h) \coloneqq \mathbb{P}(h(X) = 0| A = a, Y = 1) = 0$ for $a\in\{0,1\}$. \tosee{The fact that $\bar{\gamma}_{1a}(h) = 0$ (as opposed to $\gamma_{1a}(h) = 1$)
is crucial for obtaining the fast rates, since it allows for a concentration analysis based on the multiplicative Chernoff bounds only, rather than the additive ones and/or Hoeffding's inequality \citep{boucheron2013concentration}, which would lead to rates of $\mathcal{O}(\frac{1}{\alpha^2})$ again. }

Given a (corrupted) training set $S^p$, denote by 
\begin{equation}
\label{eqn:defn-of-h-star-set}
\mathcal{H}^*(S^p)  \vcentcolon = \left\{h\in\mathcal{H}\middle| \max_a \bar{\gamma}^p_{1a}(h) \leq \Delta^{opp} \land \widehat{\mathcal{R}}^p(h) \leq \frac{3\alpha}{2}\right\}
\end{equation}
the set of all classifiers that have a small loss and small values of $\bar{\gamma}^p_{1a}$ for both $a\in\{0,1\}$ on $S^p$. Consider the learner $\mathcal{L}^{fast}$ defined by
\begin{align}
\label{eqn:fancy-learner}
\mathcal{L}^{fast}(S^p) = 
\begin{cases}
    \text{any }h \in \mathcal{H}^*, & \text{if } \mathcal{H}^*\neq \emptyset\\
    \text{any }h \in \mathcal{H}, & \text{otherwise.}
\end{cases}
\end{align}
The intuition behind the construction is similar to before: hypotheses in $\mathcal{H}^*$ perform well on the training data and hence are good candidates. At the same time, we expect that $h^* \in \mathcal{H}^*$, so that $\mathcal{H}^*$ is non-empty.
 
Then the following result holds.
\begin{theorem}
\label{thm:upper-bound-realizable}
Let $\mathcal{H}$ be finite and $\mathbb{P}\in \mathcal{P}(\prodspace)$ be such that for some $h^*\in\mathcal{H}$, $\mathbb{P}(h^*(X) = Y) = 1$. Let $\mathcal{A}$ be any malicious adversary of power $\alpha < 0.5$. Then for any $\delta, \eta\in(0,1)$ and any 

\begin{align*}
n & \geq \max \left\{\frac{8\log(16|\mathcal{H}|/\delta)}{(1-\alpha) P_{10}},  \frac{12\log(12/\delta)}{\alpha}, \frac{2\log(8|\mathcal{H}|/\delta)}{3\eta^2\alpha}, \frac{2\log(\frac{16|\mathcal{H}|}{\delta})}{3\eta^2 (1-\alpha) P_{10}\alpha}\right\} \\ & = \Omega\left(\frac{\log(|\mathcal{H}|/\delta)}{\eta^2 P_{10}\alpha}\right)
\end{align*}
with probability at least $1-\delta$
\begin{align*}
\mathbf{L}^{opp}(\widehat{h}) \preceq \left(\frac{3\alpha}{1-\eta}, \frac{2\Delta^{opp}}{1-\eta}\right),
\end{align*}
where $\widehat{h} \coloneqq \mathcal{L}^{fast}(S^p)$ is the hypothesis returned by the learner and $$\mathbf{L}^{opp}(\widehat{h}) = \left(\mathcal{R}(\widehat{h}) - \mathcal{R}(h^*), \mathcal{D}^{opp}(\widehat{h}) - \mathcal{D}^{opp}(h^*)\right).$$
\end{theorem}
As an immediate consequence of Theorem \ref{thm:upper-bound-realizable}, setting $\eta = \frac{1}{2}$, say, yields that for large $n$, with high probability
\begin{align}
\label{eqn:upper-bound-equal-opp-component-wise}
\mathbf{L}^{opp}(\widehat{h}) \preceq \left(\mathcal{O}(\alpha), \mathcal{O}\left(\frac{\alpha}{P_{10}}\right)\right).
\end{align}
Again, this bound is order-optimal for finite hypothesis sets, according to Theorem \ref{thm:lower-bound-equal-opp-with-p} and Inequality (\ref{eqn:lower-bounds-equal-opp-component-wise}). In addition, regarding $P_{10}$ as a constant, the number of samples needed for achieving this order-optimal element-wise error is indeed $\mathcal{O}(\frac{1}{\alpha})$, according to Theorem \ref{thm:upper-bound-realizable}, which is faster than the $\widetilde{\mathcal{O}}(\frac{1}{\alpha^2})$ we obtained with the previous results.

\subsection{Sketch of the upper bounds proofs}
\label{sec:upper_bounds_proof_sketch}

\addtocounter{theorem}{-9}

Here we present a sketch of the proofs of the upper bounds. The complete proofs can be found in Appendix \ref{sec:appendix-upper-bounds-proofs}.

The proofs of Theorems \ref{thm:agnostic-upper-bound-demog-par}, \ref{thm:agnostic-upper-bound-equal-opp}, \ref{thm:comp-wise-upper-bound-demog-par}, \ref{thm:comp-wise-upper-bound-equal-opp} rely on a series of results that describe the deviations of the corrupted fairness estimates $\widehat{\mathcal{D}}(h)$ from the true underlying population values $\mathcal{D}(h)$, uniformly over the hypothesis space $\mathcal{H}$. Key to this is bounding the effect of the data corruption, as expressed by the maximum achievable gap between the corrupted fairness estimates and the corresponding estimates based on the clean (but unknown) subset of the data. Then the large deviation properties of these clean data estimates are studied instead.

Here we make this specific for the case of demographic parity, with the analysis for equal opportunity being similar. We denote
\begin{align*}
\gamma^p_a(h) = \frac{\sum_{i=1}^n \mathbbm{1}\{h(x^p_i) = 1, a^p_i = a\}}{\sum_{i=1}^n \mathbbm{1}\{a^p_i = a\}}
\end{align*}
and 
\begin{align*}
\gamma_a(h) = \mathbb{P}(h(X) = 1| A = a),
\end{align*}
so that $\widehat{\mathcal{D}}^{par}(h) = |\gamma^p_0(h) - \gamma^p_1(h)|$ and $\mathcal{D}^{par}(h) = |\gamma_0(h) - \gamma_1(h)|$.  Note that $\gamma^p_a(h)$ is an estimate of a conditional probability \textit{based on the corrupted data}. We now introduce the corresponding estimate that only uses the \textit{unknown clean subset} of the training set $S^p$
\begin{align*}
\gamma^c_a(h) = \frac{\sum_{i=1}^n \mathbbm{1}\{h(x^p_i) = 1, a^p_i = a, i\not\in\poisoned\}}{\sum_{i=1}^n \mathbbm{1}\{a^p_i = a, i\not\in\poisoned\}}.
\end{align*}

\myparagraph{Bounding the effect of the adversary} First, we bound how far the corrupted estimates $\gamma^p_a(h)$ of $\gamma_a(h)$ are from the clean estimates $\gamma^c_a(h)$, uniformly over the hypothesis space $\mathcal{H}$:
\begin{lemma}
\label{lemma:upper-bound-delta-par}
If $n \geq \max \left\{\frac{8\log(4/\delta)}{(1-\alpha)P_0}, \frac{12\log(3/\delta)}{\alpha}\right\}$, we have
\begin{align*}
\label{eqn:bound_on_delta_appendix}
\mathbb{P}^{\mathcal{A}}\left(\sup_{h\in\mathcal{H}}\left(\left|\gamma^p_0(h) - \gamma^c_0(h)\right| + \left|\gamma^p_1(h) - \gamma^c_1(h)\right|\right) \geq \frac{2\alpha}{\frac{P_0}{3} + \alpha}\right) < \delta.
\end{align*}
\end{lemma}
Informally, this lemma allows us to connect the corrupted estimate $\widehat{\mathcal{D}}^{par}(h)$ with the corresponding ideal clean estimate $\widehat{\mathcal{D}}^c(h) = |\gamma^c_0(h) - \gamma^c_1(h)|$.

\myparagraph{Bounding the deviation of the clean data estimate} Secondly, a technique used by \cite{woodworth2017learning} and \cite{agarwal2018reductions} for proving concentration of fairness measures is used to derive a concentration result for the clean estimates $\gamma^c_a(h)$, around the true population values $\gamma_a(h)$. This, together with Lemma \ref{lemma:upper-bound-delta-par}, allows us to bound the gap between the corrupted estimate $\widehat{\mathcal{D}}^{par}(h)$ and the true population value $\mathcal{D}^{par}(h)$, for a single hypothesis.

\myparagraph{Making the bound uniform over $\mathcal{H}$} Finally, the bound obtained is made uniform over $\mathcal{H}$. For this, we use the classic symmetrization technique \citep{vapnik2013nature} for proving bounds uniformly over hypothesis spaces of finite VC dimension. However, since the objective is different from the 0-1 loss, care is needed to ensure that the argument goes through, so the proof is given in full detail in the supplementary material.

Once a uniform bound on the deviations of the corrupted fairness estimates from the true underlying population values is obtained, the results of Theorems \ref{thm:agnostic-upper-bound-demog-par}, \ref{thm:agnostic-upper-bound-equal-opp}, \ref{thm:comp-wise-upper-bound-demog-par}, \ref{thm:comp-wise-upper-bound-equal-opp} follow similarly to most classic ERM results.

\myparagraph{Proof of Theorem \ref{thm:upper-bound-realizable}}
Similarly to the other results, the proof of Theorem \ref{thm:upper-bound-realizable} first links the corrupted estimates $\bar{\gamma}^p_{1a}$ to their clean counterparts and then uses the clean data concentration to study the behavior of the corrupted estimates. However, an important tool that allows us to obtain the fast statistical rates, is a set of \textit{multiplicative concentration bounds} on the $\bar{\gamma}^p_{1a}$ estimates. \tosee{It is for this reason that $\mathcal{L}^{fast}$ learner uses the $\bar{\gamma}^p_{1a}$ estimates, instead of the $\gamma^p_{1a}$, see also the discussion after equation (\ref{eqn:estimates-zero-conditional-probs})}. Full details and a complete proof can be found in the supplementary material.

%% file: conclusion.tex
In this work we explored the statistical limits of fairness-aware learning algorithms on corrupted data, under the malicious adversary model. Our results show that data manipulations can have an inevitable negative effect on model fairness and that this effect is even more expressed for problems where a subgroup in the population is underrepresented. We also provided upper bounds that match our hardness results up to constant factors, in the large data regime.

Below we outline several implications of our work and discuss some specific extensions that constitute interesting directions for future research.

\paragraph{Implications of our results} While the strong adversarial model and the statistical PAC learning analysis we have considered are mostly of theoretical interest, we believe that the hardness results have several important implications. Indeed, crucial to increasing the trust in learned decision making systems is the ability to guarantee that they exhibit a high amount of fairness, regardless of any known or unforeseen biases in the training data. In contrast, we have shown that this is provably impossible under a strong adversarial model for the data corruption.

We believe that these results stress on the importance of developing and \textit{studying further data corruption models} in the context of fairness-aware learning. As discussed in the related work section, previous research has shown that it can be possible to recover a fair model under corruptions of the labels or the protected attributes only. While real-world data is likely to contain more subtle manipulations, one may hope that for certain applications there will be models of data corruption that are, on the one hand, sufficiently broad to cover the data issues and, on the other hand, specific enough so that fair learning becomes possible.

Our results can also be seen as an indication that strict data collection practices may in fact be necessary for designing provably fair machine learning models. Indeed, our bounds hold under the assumption that the learner can only access one dataset of unknown quality. In contrast, it has been shown that the use of even a small trusted dataset (that is, a certified clean subset of the data) can greatly improve the performance of machine learning models under corruption, both in the context of classic PAC learning \citep{hendrycks2018using, konstantinov2019robust} and in the context of fairness-aware learning \citep{roh2020fr}. Such data can also be helpful for the sake of validating the fairness of a model as a precautionary step before its real-world adoption. 

In summary, understanding and accounting for the types of biases present in machine learning datasets is crucial for addressing the issues brought up in this work and for the development of certifiably fair learning models. 

\paragraph{Extensions to other fairness notions} \tosee{We expect that our analysis can be extended to other group fairness measures. In particular, the work of \cite{agarwal2018reductions} has shown that a broad range of fairness notions based on conditional independence constraints are amendable to concentration of measure analysis, via an application of the proof technique of \cite{woodworth2017learning}. Since this technique is also at the core of the concentration arguments used in the proofs of our upper bounds, we expect that a similar analysis can be conducted for the broader class of fairness measures considered by \cite{agarwal2018reductions}.}

\tosee{The lower bounds, however, require explicit constructions of hard learning problems to be designed and these constructions are necessarily tailored to the specific fairness notions being considered. Therefore, while the proof technique, namely the method of induced distributions \citep{kearns1993learning}, may be a useful tool for showing hardness results about other fairness measures, the key challenge of designing a hard learning problem instance for each measure remains open. }

\paragraph{Extensions to other adversarial models} \tosee{In this work we have studied learning and fairness under the malicious adversary model \citep{kearns1993learning}. It will be interesting to analyze the limits of fairness-aware learning for other adversarial models as well. As mentioned above, studying weaker, application-specific adversaries may allow for PAC learnability.}

\tosee{On the other hand, fairness can be studied under the even stronger nasty noise model of \cite{bshouty2002pac}, which has also recently been analyzed in the context of robust mean estimation \citep{diakonikolas2019robust}. In this model the adversary does not get to manipulate a random subset of the data, but can instead choose the points that it alters. To our awareness, the only work that considers this adversarial model in the context of fairness is that of \cite{celis2021fair2}, who, however, only study manipulations of the protected attribute and not of the labels and features.}

\tosee{Since the nasty noise adversary is strictly stronger than the malicious one, our lower bounds hold for the nasty noise model as well. In particular, achieving optimal fairness remains impossible in this setup. Whether our lower bounds are order-optimal within the nasty adversary model as well, or stronger hardness results can be shown, is an interesting direction for future work.}

\newpage

%% file: appendix_lower_bound_proofs.tex
\addtocounter{theorem}{-1}

In the proofs of our hardness results we use a technique from \citep{kearns1993learning} called the \textit{method of induced distributions}. The idea is to construct two distributions that are sufficiently different, so that different classifiers perform well on each, yet can be made indistinguishable after the modifications of the adversary. Then no fixed learner with access only to the corrupted data can be ``correct'' with high probability on both distributions and so any learner will incur excessively high loss and/or exhibit excessively high unfairness on at least one of the two distributions, regardless of the amount of available data.

The proofs of the four results are structured in a similar way, but use different constructions of the underlying learning problem, tailored to the fairness measure and the type of bound we want to show. 

\subsection{Pareto lower bounds proofs}
\label{sec:pareto-lower-bounds-proofs}

\begin{theorem}
\label{thm:lower-bound-demog-par-with-p-app}
Let $0 \leq \alpha < 0.5, 0 < P_0 \leq 0.5$. For any input set $\mathcal{X}$ with at least four distinct points, there exists a finite hypothesis space $\mathcal{H}$, such that for any learning algorithm $\mathcal{L}:\cup_{n\in\mathbb{N}}(\prodspace)^n \to \mathcal{H}$, there exists a distribution $\mathbb{P}$ \textit{for which $\mathbb{P}(A = 0) = P_0$}, a malicious adversary $\mathcal{A}$ of power $\alpha$ and a hypothesis $h^* \in \mathcal{H}$, such that with probability at least $0.5$ 
\begin{equation*}
L(\mathcal{L}(S^p), \mathbb{P}) - L(h^*, \mathbb{P}) \geq \min\left\{\frac{\alpha}{1-\alpha}, 2P_0(1-P_0)\right\}
\end{equation*}
 and 
\begin{equation*}
\mathcal{D}^{par}(\mathcal{L}(S^p), \mathbb{P}) - \mathcal{D}^{par}(h^*, \mathbb{P}) \geq \min\left\{\frac{\alpha}{2P_0(1-P_0)(1-\alpha)}, 1\right\} \geq \min\left\{\frac{\alpha}{2P_0}, 1\right\}.
\end{equation*}
\end{theorem}
\begin{proof}
Let $\eta = \frac{\alpha}{1-\alpha}$, so that $\alpha = \frac{\eta}{1+\eta}$.

\paragraph{Case 1} Assume that $\eta = \frac{\alpha}{1-\alpha} \leq 2P_0(1-P_0)$.
Take four distinct points $\{x_1, x_2, x_3, x_4\}\in\mathcal{X}$. We consider two distributions $\mathbb{P}_0$ and $\mathbb{P}_1$, where each $\mathbb{P}_i$ is defined as
\begin{equation*}
  \mathbb{P}_{i}(x,a,y) =
    \begin{cases}
      1 - P_0 - \eta/2 & \text{if $x = x_1, a = 1, y = 1$}\\
      P_0 - \eta/2 & \text{if $x = x_2, a = 0, y = 0$}\\
      \eta/2  & \text{if $x = x_3, a = i, y = \lnot i$}\\
      \eta/2  & \text{if $x = x_4, a = \lnot i, y = i$}\\
      0 & \text{otherwise}
    \end{cases}       
\end{equation*}
Note that these are valid distributions, since $\eta \leq 2P_0(1-P_0) \leq 2P_0 \leq 2(1-P_0)$ by assumption and also that $P_0 = \mathbb{P}_i(A = 0)$ for both $i\in\{0,1\}$.  Consider the hypothesis space $\mathcal{H} = \{h_0, h_1\}$, with $$h_0(x_1) = 1 \quad h_0(x_2) = 0 \quad h_0(x_3) = 1 \quad h_0(x_4) = 0$$ and
$$h_1(x_1) = 1 \quad h_1(x_2) = 0 \quad h_1(x_3) = 0 \quad h_1(x_4) = 1.$$
\noindent Note that $L(h_i, \mathbb{P}_i) = 0$ for both $i = 0,1$. Moreover, 
\begin{align*}
\mathcal{D}^{par}(h_0, \mathbb{P}_0) & = \left|\mathbb{P}_{(X,A,Y)\sim \mathbb{P}_0}(h_0(X) = 1| A = 0) - \mathbb{P}_{(X,A,Y)\sim \mathbb{P}_0}(h_0(X) = 1| A = 1)\right| \\
& = \left|\frac{\eta/2}{P_0 -\eta/2 + \eta/2} - \frac{1 - P_0 - \eta/2}{1 - P_0 - \eta/2 + \eta/2}\right| \\
& = \left|\frac{\eta}{2P_0} - \frac{2 - 2P_0 - \eta}{2(1-P_0)}\right| \\
& = \left|\frac{\eta}{2P_0(1-P_0)} - 1\right|\\
& = 1 - \frac{\eta}{2P_0(1-P_0)},
\end{align*}
since $\eta \leq 2P_0(1-P_0)$ by assumption. Furthermore, 
\begin{align*}
\mathcal{D}^{par}(h_1, \mathbb{P}_0) & = \left|\mathbb{P}_{(X,A,Y)\sim \mathbb{P}_0}(h_1(X) = 1| A = 0) - \mathbb{P}_{(X,A,Y)\sim \mathbb{P}_0}(h_1(X) = 1| A = 1)\right| \\
& = \left|0 - 1\right| \\
& = 1
\end{align*}
Therefore, $\mathcal{D}^{par}(h_1, \mathbb{P}_0) - \mathcal{D}^{par}(h_0, \mathbb{P}_0) =  \frac{\eta}{2P_0(1-P_0)}$. Similarly,
\begin{align*}
\mathcal{D}^{par}(h_1, \mathbb{P}_1) & = \left|\mathbb{P}_{(X,A,Y)\sim \mathbb{P}_1}(h_1(X) = 1| A = 0) - \mathbb{P}_{(X,A,Y)\sim \mathbb{P}_1}(h_1(X) = 1| A = 1)\right| \\
& = \left|\frac{\eta/2}{P_0 -\eta/2 + \eta/2} - \frac{1 - P_0 - \eta/2}{1 - P_0 - \eta/2 + \eta/2}\right| \\
& = 1 - \frac{\eta}{2P_0(1-P_0)}
\end{align*}
and 
\begin{align*}
\mathcal{D}^{par}(h_0, \mathbb{P}_1) & = \left|\mathbb{P}_{(X,A,Y)\sim \mathbb{P}_1}(h_0(X) = 1| A = 0) - \mathbb{P}_{(X,A,Y)\sim \mathbb{P}_1}(h_0(X) = 1| A = 1)\right| \\
& = \left|0 - 1\right| \\
& = 1,
\end{align*}
so that $\mathcal{D}^{par}(h_0, \mathbb{P}_1) - \mathcal{D}^{par}(h_1, \mathbb{P}_1) = \frac{\eta}{2P_0(1-P_0)}$.

Consider a (randomized) malicious adversary $\mathcal{A}_{i}$ of power $\alpha$, that given a clean distribution $\mathbb{P}_{i}$, changes every marked point to $(x_3,\lnot i, i)$ with probability $0.5$ and to $(x_4,i, \lnot i)$ otherwise. Under a distribution $\mathbb{P}_{i}$ and an adversary $\mathcal{A}_{i}$, the probability of seeing a point $(x_3, i, \lnot i)$ is $\frac{\eta}{2} (1-\alpha) = \frac{\eta}{2} \frac{1}{1 + \eta} = \alpha/2$, which is equal to the probability of seeing a point $(x_3, \lnot i, i)$. Therefore, denoting the probability distribution of the corrupted dataset, under a clean distribution $\mathbb{P}_{i}$ and an adversary $\mathcal{A}_{i}$, by $\mathbb{P}'_{i}$ (as a shorthand for $\mathbb{P}_i^{\mathcal{A}_i}$), we have
\begin{equation*}
  \mathbb{P}'_{i}(x,a,y) =
    \begin{cases}
      (1-\alpha)(1 - P_0 - \eta/2) & \text{if $x = x_1, a = 1, y = 1$}\\
      (1-\alpha)(P_0 - \eta/2) & \text{if $x = x_2, a = 0, y = 0$}\\
      \alpha/2  & \text{if $x = x_3, a = i, y = \lnot i$}\\
      \alpha/2  & \text{if $x = x_3, a = \lnot i, y = i$}\\
      \alpha/2  & \text{if $x = x_4, a = \lnot i, y = i$}\\
      \alpha/2  & \text{if $x = x_4, a = i, y = \lnot i$}\\
      0 & \text{otherwise}
    \end{cases}      
\end{equation*}
In particular, $\mathbb{P}'_{0} = \mathbb{P}'_{1}$, so the two initial distributions $\mathbb{P}_0$ and $\mathbb{P}_1$ become indistinguishable under the adversarial manipulation.

Fix an arbitrary learner $\mathcal{L}: \cup_{n\in \mathbb{N}} (\prodspace)^n \to \{h_0, h_1\}$. Note that, if the clean distribution is $\mathbb{P}_0$, the events (in the probability space defined by the sampling of the poisoned train data)
\begin{align*}
\{L(\mathcal{L}(S^p), \mathbb{P}_0) - L(h_0, \mathbb{P}_0) \geq \eta \} & = \{\mathcal{L}(S^p) = h_1\} \\ & = \left\{\mathcal{D}^{par}(\mathcal{L}(S^p), \mathbb{P}_0) - \mathcal{D}^{par}(h_0, \mathbb{P}_0) \geq \frac{\eta}{2P_0(1-P_0)}\right\}
\end{align*}
are all the same. Similarly, if the clean distribution is $\mathbb{P}_1$
\begin{align*}
\{L(\mathcal{L}(S^p), \mathbb{P}_1) - L(h_1, \mathbb{P}_1) \geq \eta\} & = \{\mathcal{L}(S^p) = h_0\} \\ & = \left\{\mathcal{D}^{par}(\mathcal{L}(S^p), \mathbb{P}_1) - \mathcal{D}^{par}(h_1, \mathbb{P}_1) \geq \frac{\eta}{2P_0(1-P_0)}\right\} .
\end{align*}
Therefore, depending on whether we choose $\mathbb{P}_0$ or $\mathbb{P}_1$ as a clean distribution, we have
\begin{align*}
\mathbb{P}_{S^p \sim \mathbb{P}'_0} & \left( \left(L(\mathcal{L}(S^p), \mathbb{P}_0) - L(h_0, \mathbb{P}_0) \geq \eta \right) \land \left(\mathcal{D}^{par}(\mathcal{L}(S^p), \mathbb{P}_0) - \mathcal{D}^{par}(h_0, \mathbb{P}_0) \geq \frac{\eta}{2P_0(1-P_0)} \right) \right) \\ & = \mathbb{P}_{S^p \sim \mathbb{P}'_0}\left(\mathcal{L}(S^p) = h_1\right)
\end{align*}
and
\begin{align*}
\mathbb{P}_{S^p \sim \mathbb{P}'_1} & \left( \left(L(\mathcal{L}(S^p), \mathbb{P}_1) - L(h_1, \mathbb{P}_1) \geq \eta \right) \land \left( \mathcal{D}^{par}(\mathcal{L}(S^p), \mathbb{P}_1) - \mathcal{D}^{par}(h_1, \mathbb{P}_1) \geq \frac{\eta}{2P_0(1-P_0)} \right) \right) \\ & = \mathbb{P}_{S^p \sim \mathbb{P}'_1}\left(\mathcal{L}(S^p) = h_0\right)
\end{align*}
Finally, note that $\mathbb{P}'_0 = \mathbb{P}'_1$, so that either $\mathbb{P}_{S^p \sim \mathbb{P}'_0}\left(\mathcal{L}(S^p) = h_1\right) \geq 1/2$ or $\mathbb{P}_{S^p \sim \mathbb{P}'_1}\left(\mathcal{L}(S^p) = h_0\right) \geq 1/2$. Therefore, for at least one of $i = 0,1$, both $$L(\mathcal{L}(S^p), \mathbb{P}_i) - L(h_i, \mathbb{P}_i) \geq \eta = \frac{\alpha}{1-\alpha}$$ and $$\mathcal{D}^{par}(\mathcal{L}(S^p), \mathbb{P}_i) - \mathcal{D}^{par}(h_i, \mathbb{P}_i) \geq \frac{\eta}{2P_0(1-P_0)} = \frac{\alpha}{2P_0(1-P_0)(1-\alpha)}$$ both hold with probability at least $1/2$ when the choice of distribution and adversary is $\mathbb{P}_i$ and $\mathcal{A}_i$ respectively. This concludes the proof in the first case.\\
\\
\paragraph{Case 2} Now suppose that $\eta = \frac{\alpha}{1-\alpha} > 2P_0(1-P_0)$. Let $\alpha_1\in (0,0.5)$ be such that $\frac{\alpha_1}{1-\alpha_1} = 2P_0(1-P_0)$. Note that since $f(x) = \frac{x}{1-x}$ is monotonically increasing in (0,1), $\alpha_1$ is unique and $\alpha_1 < \alpha$.\\
\\
Now repeat the same construction as in Case 1, but with $\eta_1 = \frac{\alpha_1}{1-\alpha_1} = 2P_0(1-P_0)$. For every marked point, the adversary does the same as in Case 1 with probability $\alpha_1/\alpha$ and does not change the point otherwise. Then the same argument as in Case 1 shows that for one $i\in\{0,1\}$, both $$L(\mathcal{L}(S^p), \mathbb{P}_i) - L(h_i, \mathbb{P}_i) \geq \eta_1 = \frac{\alpha_1}{1-\alpha_1} = 2P_0(1-P_0)$$ and $$\mathcal{D}^{par}(\mathcal{L}(S^p), \mathbb{P}_i) - \mathcal{D}^{par}(h_i, \mathbb{P}_i) \geq \frac{\eta_1}{2P_0(1-P_0)} = 1$$ both hold with probability at least $1/2$. This concludes the proof of Theorem \ref{thm:lower-bound-demog-par-with-p-app}.
\end{proof}

\begin{theorem}
\label{thm:lower-bound-equal-opp-with-p-app}
Let $0 \leq \alpha < 0.5, P_{10} \leq P_{11} < 1$ be such that $P_{10} + P_{11} < 1$. For any input set $\mathcal{X}$ with at least five distinct points, there exists a finite hypothesis space $\mathcal{H}$, such that for any learning algorithm $\mathcal{L}:\cup_{n\in\mathbb{N}}(\prodspace)^n \to \mathcal{H}$, there exists a distribution $\mathbb{P}$ \textit{for which $\mathbb{P}(A = a, Y = 1) = P_{1a}$ for $a\in\{0,1\}$}, a malicious adversary $\mathcal{A}$ of power $\alpha$ and a hypothesis $h^* \in \mathcal{H}$, such that with probability at least $0.5$
$$L(\mathcal{L}(S^p), \mathbb{P}) - L(h^*, \mathbb{P}) > \min\left\{\frac{\alpha}{1-\alpha}, 2P_{10}, 2(1-P_{10} - P_{11})\right\}$$ and $$\mathcal{D}^{opp}(\mathcal{L}(S^p), \mathbb{P}) - \mathcal{D}^{opp}(h^*, \mathbb{P}) \geq \min\left\{\frac{\alpha}{2(1-\alpha)P_{10}}, 1, \frac{1 - P_{10} - P_{11}}{P_{10}}\right\}.$$
\end{theorem}
\begin{proof}
Let $\eta = \frac{\alpha}{1-\alpha}$, so that $\alpha = \frac{\eta}{1+\eta}$.

\paragraph{Case 1} Assume that $\eta = \frac{\alpha}{1-\alpha} \leq 2\min \{P_{10}, 1 - P_{10} - P_{11}\}$. Take five distinct points $\{x_1, x_2, x_3, x_4, x_5\}\in\mathcal{X}$. We consider two distributions $\mathbb{P}_0$ and $\mathbb{P}_1$, where each $\mathbb{P}_i$ is defined as
\begin{equation*}
  \mathbb{P}_{i}(x,a,y) =
    \begin{cases}
      P_{11} & \text{if $x = x_1, a = 1, y = 1$}\\
      P_{10} - \eta/2 & \text{if $x = x_2, a = 0, y = 1$}\\
      \eta/2  & \text{if $x = x_3, a = i, y = \lnot i$}\\
      \eta/2  & \text{if $x = x_4, a = \lnot i, y = i$}\\
      1 - P_{10} - P_{11} - \eta/2  & \text{if $x = x_5, a = 0, y = 0$}\\
      0 & \text{otherwise}
    \end{cases}       
\end{equation*}
Note that these are valid distributions, since $\eta \leq 2P_{10}, \eta \leq 2(1 - P_{10} - P_{11})$ by assumption, and that $P_{1a} = \mathbb{P}_i(A = a, Y = 1)$ for both $a\in\{0,1\}, i \in \{0,1\}$. Consider the hypothesis space $\mathcal{H} = \{h_0, h_1\}$, with $$h_0(x_1) = 1 \quad h_0(x_2) = 1 \quad h_0(x_3) = 1 \quad h_0(x_4) = 0 \quad h_0(x_5) = 0$$ and $$h_1(x_1) = 1 \quad h_1(x_2) = 1 \quad h_1(x_3) = 0 \quad h_1(x_4) = 1 \quad h_1(x_5) = 0$$
\noindent Note that $L(h_i, \mathbb{P}_i) = 0$ and $\mathcal{D}^{opp}(h_i, \mathbb{P}_i) = 0$  for both $i = 0,1$. Note also that $L(h_1, \mathbb{P}_0) = L(h_0, \mathbb{P}_1) = \eta$. Moreover, 
\begin{align*}
\mathcal{D}^{opp}(h_1, \mathbb{P}_0) & = \left|\mathbb{P}_{(X,A,Y)\sim \mathbb{P}_0}(h_1(X) = 1| A = 0, Y = 1) \right. \\
& - \left. \mathbb{P}_{(X,A,Y)\sim \mathbb{P}_0}(h_1(X) = 1| A = 1, Y = 1)\right| \\
& = \left|\frac{P_{10} - \eta/2}{P_{10} -\eta/2 + \eta/2} - 1\right| \\
& = \frac{\eta}{2P_{10}}
\end{align*}
and similarly $\mathcal{D}^{opp}(h_0, \mathbb{P}_1) = \frac{\eta}{2P_{10}}$.\\
\\
Consider a (randomized) malicious adversary $\mathcal{A}_{i}$ of power $\alpha$, that given a clean distribution $\mathbb{P}_{i}$, changes every marked point to $(x_3,\lnot i, i)$ with probability $0.5$ and to $(x_4,i, \lnot i)$ otherwise. Under a distribution $\mathbb{P}_{i}$ and an adversary $\mathcal{A}_{i}$, the probability of seeing a point $(x_3, i, \lnot i)$ is $\frac{\eta}{2} (1-\alpha) = \frac{\eta}{2} \frac{1}{1 + \eta} = \alpha/2$, which is equal to the probability of seeing a point $(x_3, \lnot i, i)$. Therefore, denoting the probability distribution of the corrupted dataset, under a clean distribution $\mathbb{P}_{i}$ and an adversary $\mathcal{A}_{i}$, by $\mathbb{P}'_{i}$, we have
\begin{equation*}
  \mathbb{P}'_{i}(x,a,y) =
    \begin{cases}
      (1-\alpha)P_{11} & \text{if $x = x_1, a = 1, y = 1$}\\
      (1-\alpha)(P_{10} - \eta/2) & \text{if $x = x_2, a = 0, y = 1$}\\
      \alpha/2  & \text{if $x = x_3, a = i, y = \lnot i$}\\
      \alpha/2  & \text{if $x = x_3, a = \lnot i, y = i$}\\
      \alpha/2  & \text{if $x = x_4, a = \lnot i, y = i$}\\
      \alpha/2  & \text{if $x = x_4, a = i, y = \lnot i$}\\
      (1-\alpha)(1 - P_{10} - P_{11} - \eta/2)  & \text{if $x = x_5, a = 0, y = 0$}\\
      0 & \text{otherwise}
    \end{cases}  
\end{equation*}
In particular, $\mathbb{P}'_{0} = \mathbb{P}'_{1}$, so the two initial distributions $\mathbb{P}_0$ and $\mathbb{P}_1$ become indistinguishable under the adversarial manipulation.\\
\\
Fix an arbitrary learner $\mathcal{L}: \cup_{n\in \mathbb{N}} (\prodspace)^n \to \{h_0, h_1\}$. Note that, if the clean distribution is $\mathbb{P}_0$, the events (in the probability space defined by the sampling of the poisoned train data)
\begin{align*}
\{L(\mathcal{L}(S^p), \mathbb{P}_0) - L(h_0, \mathbb{P}_0) \geq \eta \}  & = \{\mathcal{L}(S^p) = h_1\} \\ & = \left\{\mathcal{D}^{opp}(\mathcal{L}(S^p), \mathbb{P}_0) - \mathcal{D}^{opp}(h_0, \mathbb{P}_0) \geq \frac{\eta}{2P_{10}}\right\}
\end{align*}
are all the same. Similarly, if the clean distribution is $\mathbb{P}_1$
\begin{align*}
\{L(\mathcal{L}(S^p), \mathbb{P}_1) - L(h_1, \mathbb{P}_1) \geq \eta\} & = \{\mathcal{L}(S^p) = h_0\} \\ & = \left\{\mathcal{D}^{opp}(\mathcal{L}(S^p), \mathbb{P}_1) - \mathcal{D}^{opp}(h_1, \mathbb{P}_1) \geq \frac{\eta}{2P_{10}}\right\} .
\end{align*}
Therefore, depending on whether we choose $\mathbb{P}_0$ or $\mathbb{P}_1$ as a clean distribution, we have
\begin{align*}
\mathbb{P}_{S^p \sim \mathbb{P}'_0} & \left(L(\mathcal{L}(S^p), \mathbb{P}_0) - L(h_0, \mathbb{P}_0) \geq \eta \land \mathcal{D}^{opp}(\mathcal{L}(S^p), \mathbb{P}_0) - \mathcal{D}^{opp}(h_0, \mathbb{P}_0) \geq \frac{\eta}{2P_{10}} \right) \\ & = \mathbb{P}_{S^p \sim \mathbb{P}'_0}\left(\mathcal{L}(S^p) = h_1\right)
\end{align*}
and
\begin{align*}
\mathbb{P}_{S^p \sim \mathbb{P}'_1} & \left(L(\mathcal{L}(S^p), \mathbb{P}_1) - L(h_1, \mathbb{P}_1) \geq \eta \land \mathcal{D}^{opp}(\mathcal{L}(S^p), \mathbb{P}_1) - \mathcal{D}^{opp}(h_1, \mathbb{P}_1) \geq\frac{\eta}{2P_{10}} \right) \\ & = \mathbb{P}_{S^p \sim \mathbb{P}'_1}\left(\mathcal{L}(S^p) = h_0\right)
\end{align*}
Finally, note that $\mathbb{P}'_0 = \mathbb{P}'_1$, so that either $\mathbb{P}_{S^p \sim \mathbb{P}'_0}\left(\mathcal{L}(S^p) = h_1\right) \geq 1/2$ or $\mathbb{P}_{S^p \sim \mathbb{P}'_1}\left(\mathcal{L}(S^p) = h_0\right) \geq 1/2$. Therefore, for at least one of $i = 0,1$, both $$L(\mathcal{L}(S^p), \mathbb{P}_i) - L(h_i, \mathbb{P}_i) \geq \eta = \frac{\alpha}{1-\alpha}$$ and $$\mathcal{D}^{opp}(\mathcal{L}(S^p), \mathbb{P}_i) - \mathcal{D}^{opp}(h_i, \mathbb{P}_i) \geq \frac{\eta}{2P_{10}} = \frac{\alpha}{2P_{10}(1-\alpha)}$$ both hold with probability at least $1/2$. This concludes the proof of the first case.

\paragraph{Case 2} Now assume that $\frac{\alpha}{1-\alpha} > 2\min\left\{P_{10}, 1 - P_{10} - P_{11}\right\}$. We distinguish two cases:

\paragraph{Case 2.1} Suppose that $P_{10} \leq 1 - P_{10} - P_{11}$. We have that $\frac{\alpha}{1-\alpha} > 2P_{10}$. Then, denote by $\alpha_1$ the unique number between $(0,0.5)$, such that $\frac{\alpha_1}{1-\alpha_1} = 2P_{10} = 2\min\left\{P_{10}, 1 - P_{10} - P_{11}\right\}$, and note that $\alpha_1 < \alpha$. Then repeat the same construction as in Case 1, but with $\eta_1 = \frac{\alpha_1}{1-\alpha_1}$ and an adversary that with probability $\alpha_1/\alpha$ does the same as in Case 1 and leaves a marked point untouched otherwise.

Then the same argument as in Case 1 gives that for some $i\in\{0,1\}$, with probability at least $0.5$, both of the following hold
$$L(\mathcal{L}(S^p), \mathbb{P}_i) - L(h_i, \mathbb{P}_i) \geq \frac{\alpha_1}{1-\alpha_1} = 2P_{10}$$ and $$\mathcal{D}^{opp}(\mathcal{L}(S^p), \mathbb{P}_i) - \mathcal{D}^{opp}(h_i, \mathbb{P}_i) \geq \frac{\eta_1}{2P_{10}} = 1.$$

\paragraph{Case 2.2} In the case when $1 - P_{10} - P_{11} < P_{10}$ we have that $\frac{\alpha}{1-\alpha} > 2(1 - P_{10} - P_{11})$. Then, denote by $\alpha_2$ the unique number between $(0,0.5)$, such that $\frac{\alpha_2}{1-\alpha_2} = 2(1 - P_{10} - P_{11}) = 2\min\left\{P_{10}, 1 - P_{10} - P_{11}\right\}$, and note that $\alpha_2 < \alpha$. Then repeat the same construction as in Case 1, but with $\eta_2 = \frac{\alpha_2}{1-\alpha_2}$ and an adversary that with probability $\alpha_2/\alpha$ does the same as in Case 1 and leaves a marked point untouched otherwise.

Then the same argument as in Case 1 gives that for some $i\in\{0,1\}$, with probability at least $0.5$, both of the following hold
$$L(\mathcal{L}(S^p), \mathbb{P}_i) - L(h_i, \mathbb{P}_i) \geq \frac{\alpha_2}{1-\alpha_2} = 2(1 - P_{10} - P_{11})$$ and $$\mathcal{D}^{opp}(\mathcal{L}(S^p), \mathbb{P}_i) - \mathcal{D}^{opp}(h_i, \mathbb{P}_i) \geq \frac{\eta_2}{2P_{10}} = \frac{1 - P_{10} - P_{11}}{P_{10}}.$$
This concludes the proof of Theorem \ref{thm:lower-bound-equal-opp-with-p-app}.
\end{proof}

\subsection{Hurting fairness without affecting accuracy - proofs}
\label{sec:good-accuracy-lower-bounds-proofs}

\begin{theorem}
\label{thm:lower-bound-demog-par-with-p-good-loss-app}
Let $0 \leq \alpha < 0.5, 0 < P_0 \leq 0.5$. For any input set $\mathcal{X}$ with at least four distinct points, there exists a finite hypothesis space $\mathcal{H}$, such that for any learning algorithm $\mathcal{L}:\cup_{n\in\mathbb{N}}(\prodspace)^n \to \mathcal{H}$, there exists a distribution $\mathbb{P}$ \textit{for which $\mathbb{P}(A = 0) = P_0$}, a malicious adversary $\mathcal{A}$ of power $\alpha$ and a hypothesis $h^* \in \mathcal{H}$, such that with probability at least $0.5$
\begin{equation*}
L(\mathcal{L}(S^p), \mathbb{P}) = L(h^*, \mathbb{P}) = \min_{h\in\mathcal{H}}L(h, \mathbb{P})
\end{equation*}
 and 
\begin{equation*}
\mathcal{D}^{par}(\mathcal{L}(S^p), \mathbb{P}) - \mathcal{D}^{par}(h^*, \mathbb{P}) \geq \min\left\{\frac{\alpha}{2P_0(1-P_0)(1-\alpha)}, 1\right\} \geq \min\left\{\frac{\alpha}{2P_0}, 1\right\}.
\end{equation*}
\end{theorem}
\begin{proof}
Let $\eta = \frac{\alpha}{1-\alpha}$, so that $\alpha = \frac{\eta}{1+\eta}$.

\paragraph{Case 1} First assume that $\eta = \frac{\alpha}{1-\alpha} \leq 2P_0(1-P_0)$.
Take four distinct points $\{x_1, x_2, x_3, x_4\}\in\mathcal{X}$. We consider two distributions $\mathbb{P}_0$ and $\mathbb{P}_1$, where each $\mathbb{P}_i$ is defined as
\begin{equation*}
  \mathbb{P}_{i}(x,a,y) =
    \begin{cases}
      1 - P_0 - \eta/2 & \text{if $x = x_1, a = 1, y = 1$}\\
      P_0 - \eta/2 & \text{if $x = x_2, a = 0, y = 0$}\\
      \eta/2  & \text{if $x = x_3, a = i, y = 1$}\\
      \eta/2  & \text{if $x = x_4, a = \lnot i, y = 1$}\\
      0 & \text{otherwise}
    \end{cases}       
\end{equation*}
Note that these are valid distributions, since $\eta \leq 2P_0(1-P_0) \leq 2P_0 \leq 2(1-P_0)$ by assumption and also that $P_0 = \mathbb{P}_i(A = 0)$ for both $i\in\{0,1\}$.  Consider the hypothesis space $\mathcal{H} = \{h_0, h_1\}$, with $$h_0(x_1) = 1 \quad h_0(x_2) = 0 \quad h_0(x_3) = 1 \quad h_0(x_4) = 0$$ and
$$h_1(x_1) = 1 \quad h_1(x_2) = 0 \quad h_1(x_3) = 0 \quad h_1(x_4) = 1.$$
\noindent Note that $L(h_i, \mathbb{P}_i) = L(h_{\lnot i}, \mathbb{P}_i) = \eta/2$ for both $i = 0,1$. Moreover, 
\begin{align*}
\mathcal{D}^{par}(h_0, \mathbb{P}_0) & = \left|\mathbb{P}_{(X,A,Y)\sim \mathbb{P}_0}(h_0(X) = 1| A = 0) - \mathbb{P}_{(X,A,Y)\sim \mathbb{P}_0}(h_0(X) = 1| A = 1)\right| \\
& = \left|\frac{\eta/2}{P_0 -\eta/2 + \eta/2} - \frac{1 - P_0 - \eta/2}{1 - P_0 - \eta/2 + \eta/2}\right| \\
& = \left|\frac{\eta}{2P_0} - \frac{2 - 2P_0 - \eta}{2(1-P_0)}\right| \\
& = \left|\frac{\eta}{2P_0(1-P_0)} - 1\right|\\
& = 1 - \frac{\eta}{2P_0(1-P_0)},
\end{align*}
since $\eta \leq 2P_0(1-P_0)$ by assumption. Furthermore, $\mathcal{D}^{par}(h_1, \mathbb{P}_0) = 1$, so that $\mathcal{D}^{par}(h_1, \mathbb{P}_0) - \mathcal{D}^{par}(h_0, \mathbb{P}_0) =  \frac{\eta}{2P_0(1-P_0)}$. Similarly,
\begin{align*}
\mathcal{D}^{par}(h_1, \mathbb{P}_1) & = \left|\mathbb{P}_{(X,A,Y)\sim \mathbb{P}_1}(h_1(X) = 1| A = 0) - \mathbb{P}_{(X,A,Y)\sim \mathbb{P}_1}(h_1(X) = 1| A = 1)\right| \\
& = \left|\frac{\eta/2}{P_0 -\eta/2 + \eta/2} - \frac{1 - P_0 - \eta/2}{1 - P_0 - \eta/2 + \eta/2}\right| \\
& = 1 - \frac{\eta}{2P_0(1-P_0)}
\end{align*}
and $\mathcal{D}^{par}(h_0, \mathbb{P}_1) = 1$.\\
\\
Consider a (randomized) malicious adversary $\mathcal{A}_{i}$ of power $\alpha$, that given a clean distribution $\mathbb{P}_{i}$, changes every marked point to $(x_3,\lnot i, 1)$ with probability $0.5$ and to $(x_4, i, 1)$ otherwise. Under a distribution $\mathbb{P}_{i}$ and an adversary $\mathcal{A}_{i}$, the probability of seeing a point $(x_3, i, 1)$ is $\frac{\eta}{2} (1-\alpha) = \frac{\eta}{2} \frac{1}{1 + \eta} = \alpha/2$, which is equal to the probability of seeing a point $(x_3, \lnot i, 1)$. Therefore, denoting the probability distribution of the corrupted dataset, under a clean distribution $\mathbb{P}_{i}$ and an adversary $\mathcal{A}_{i}$, by $\mathbb{P}'_{i}$, we have
\begin{equation*}
  \mathbb{P}'_{i}(x,a,y) =
    \begin{cases}
      (1-\alpha)(1 - P_0 - \eta/2) & \text{if $x = x_1, a = 1, y = 1$}\\
      (1-\alpha)(P_0 - \eta/2) & \text{if $x = x_2, a = 0, y = 0$}\\
      \alpha/2  & \text{if $x = x_3, a = i, y = 1$}\\
      \alpha/2  & \text{if $x = x_3, a = \lnot i, y = 1$}\\
      \alpha/2  & \text{if $x = x_4, a = \lnot i, y = 1$}\\
      \alpha/2  & \text{if $x = x_4, a = i, y = 1$}\\
      0 & \text{otherwise}
    \end{cases}      
\end{equation*}
In particular, $\mathbb{P}'_{0} = \mathbb{P}'_{1}$, so the two initial distributions $\mathbb{P}_0$ and $\mathbb{P}_1$ become indistinguishable under the adversarial manipulation.

Fix an arbitrary learner $\mathcal{L}: \cup_{n\in \mathbb{N}} (\prodspace)^n \to \{h_0, h_1\}$. Note that, if the clean distribution is $\mathbb{P}_0$, the events (in the probability space defined by the sampling of the poisoned train data)
\begin{align*}
\{\mathcal{L}(S^p) = h_1\} = \left\{\mathcal{D}^{par}(\mathcal{L}(S^p), \mathbb{P}_0) - \mathcal{D}^{par}(h_0, \mathbb{P}_0) \geq \frac{\eta}{2P_0(1-P_0)}\right\}
\end{align*}
are all the same. Similarly, if the clean distribution is $\mathbb{P}_1$
\begin{align*}
\{\mathcal{L}(S^p) = h_0\} = \left\{\mathcal{D}^{par}(\mathcal{L}(S^p), \mathbb{P}_1) - \mathcal{D}^{par}(h_1, \mathbb{P}_1) \geq \frac{\eta}{2P_0(1-P_0)}\right\} .
\end{align*}
Therefore, depending on whether we choose $\mathbb{P}_0$ or $\mathbb{P}_1$ as a clean distribution, we have
\begin{align*}
\mathbb{P}_{S^p \sim \mathbb{P}'_0}\left(\mathcal{D}^{par}(\mathcal{L}(S^p), \mathbb{P}_0) - \mathcal{D}^{par}(h_0, \mathbb{P}_0) \geq \frac{\eta}{2P_0(1-P_0)} \right) = \mathbb{P}_{S^p \sim \mathbb{P}'_0}\left(\mathcal{L}(S^p) = h_1\right)
\end{align*}
and
\begin{align*}
\mathbb{P}_{S^p \sim \mathbb{P}'_1}\left(\mathcal{D}^{par}(\mathcal{L}(S^p), \mathbb{P}_1) - \mathcal{D}^{par}(h_1, \mathbb{P}_1) \geq \frac{\eta}{2P_0(1-P_0)} \right) = \mathbb{P}_{S^p \sim \mathbb{P}'_1}\left(\mathcal{L}(S^p) = h_0\right)
\end{align*}
Finally, note that $\mathbb{P}'_0 = \mathbb{P}'_1$, so that either $\mathbb{P}_{S^p \sim \mathbb{P}'_0}\left(\mathcal{L}(S^p) = h_1\right) \geq 1/2$ or $\mathbb{P}_{S^p \sim \mathbb{P}'_1}\left(\mathcal{L}(S^p) = h_0\right) \geq 1/2$. Furthermore, $L(\mathcal{L}(S^p), \mathbb{P}_i) = \eta/2$ holds for both $i\in\{0,1\}$, for any realization of the randomness. Therefore, for at least one of $i = 0,1$, both $$L(\mathcal{L}(S^p), \mathbb{P}_i) = L(h_i, \mathbb{P}_i) = \frac{\eta}{2}$$ and $$\mathcal{D}^{par}(\mathcal{L}(S^p), \mathbb{P}_i) - \mathcal{D}^{par}(h_i, \mathbb{P}_i) \geq \frac{\eta}{2P_0(1-P_0)} = \frac{\alpha}{2P_0(1-P_0)(1-\alpha)}$$ both hold with probability at least $1/2$. This concludes the proof in the first case.

\paragraph{Case 2} Now suppose that $\eta = \frac{\alpha}{1-\alpha} > 2P_0(1-P_0)$. Let $\alpha_1\in (0,0.5)$ be such that $\frac{\alpha_1}{1-\alpha_1} = 2P_0(1-P_0)$. Note that since $f(x) = \frac{x}{1-x}$ is monotonically increasing in (0,1), $\alpha_1$ is unique and $\alpha_1 < \alpha$.

Now repeat the same construction as in Case 1, but with $\eta_1 = \frac{\alpha_1}{1-\alpha_1} = 2P_0(1-P_0)$. For every marked point, the adversary does the same as in Case 1 with probability $\alpha_1/\alpha$ and does not change the point otherwise. Then the same argument as in Case 1 shows that for one $i\in\{0,1\}$, both $$L(\mathcal{L}(S^p), \mathbb{P}_i) = L(h_i, \mathbb{P}_i) = \frac{\eta_1}{2} = P_0(1-P_0)$$ and $$\mathcal{D}^{par}(\mathcal{L}(S^p), \mathbb{P}_i) - \mathcal{D}^{par}(h_i, \mathbb{P}_i) \geq \frac{\eta_1}{2P_0(1-P_0)} = 1$$ both hold with probability at least $1/2$. This concludes the proof of Theorem \ref{thm:lower-bound-demog-par-with-p-good-loss-app}.
\end{proof}

\begin{theorem}
\label{thm:lower-bound-equal-opp-with-p-good-loss-app}
Let $0 \leq \alpha < 0.5, P_{10} \leq P_{11} < 1$ be such that $P_{10} + P_{11} < 1$. For any input set $\mathcal{X}$ with at least five distinct points, there exists a finite hypothesis space $\mathcal{H}$, such that for any learning algorithm $\mathcal{L}:\cup_{n\in\mathbb{N}}(\prodspace)^n \to \mathcal{H}$, there exists a distribution $\mathbb{P}$ \textit{for which $\mathbb{P}(A = a, Y = 1) = P_{1a}$ for $a\in\{0,1\}$}, a malicious adversary $\mathcal{A}$ of power $\alpha$ and a hypothesis $h^* \in \mathcal{H}$, such that with probability at least $0.5$ $$L(\mathcal{L}(S^p), \mathbb{P}) = L(h^*, \mathbb{P}) = \min_{h\in\mathcal{H}} L(h, \mathbb{P})$$ and $$\mathcal{D}^{opp}(\mathcal{L}(S^p), \mathbb{P}) - \mathcal{D}^{opp}(h^*, \mathbb{P}) \geq \min\left\{\frac{\alpha}{2(1-\alpha)P_{10}}\left(1-\frac{P_{10}}{P_{11}}\right), 1-\frac{P_{10}}{P_{11}}\right\}.$$
\end{theorem}
\begin{proof}
Let $\eta = \frac{\alpha}{1-\alpha}$, so that $\alpha = \frac{\eta}{1+\eta}$.

\paragraph{Case 1} First assume that $\eta\leq 2P_{10}$. Take five distinct points $\{x_1, x_2, x_3, x_4, x_5\}\in\mathcal{X}$. We consider two distributions $\mathbb{P}_0$ and $\mathbb{P}_1$, where each $\mathbb{P}_i$ is defined as
\begin{equation*}
  \mathbb{P}_{i}(x,a,y) =
    \begin{cases}
      P_{11} - \eta/2 & \text{if $x = x_1, a = 1, y = 1$}\\
      P_{10} - \eta/2 & \text{if $x = x_2, a = 0, y = 1$}\\
      \eta/2  & \text{if $x = x_3, a = i, y = 1$}\\
      \eta/2  & \text{if $x = x_4, a = \lnot i, y = 1$}\\
      1 - P_{10} - P_{11} & \text{if $x = x_5, a = 0, y = 0$}\\
      0 & \text{otherwise}
    \end{cases}       
\end{equation*}
Note that these are valid distributions, since $\eta \leq 2P_{10} \leq 2P_{11}$ by assumption, and that $P_{1a} = \mathbb{P}_i(A = a, Y = 1)$ for both $a\in\{0,1\}, i \in \{0,1\}$. Consider the hypothesis space $\mathcal{H} = \{h_0, h_1\}$, with $$h_0(x_1) = 1 \quad h_0(x_2) = 1 \quad h_0(x_3) = 1 \quad h_0(x_4) = 0 \quad h_0(x_5) = 0$$ and $$h_1(x_1) = 1 \quad h_1(x_2) = 1 \quad h_1(x_3) = 0 \quad h_1(x_4) = 1 \quad h_1(x_5) = 0$$
\noindent Note that $L(h_i, \mathbb{P}_i) = L(h_{\lnot i}, \mathbb{P}_i) = \eta/2$. Moreover, 
\begin{align*}
\mathcal{D}^{opp}(h_0, \mathbb{P}_0) & = \left|\mathbb{P}_{(X,A,Y)\sim \mathbb{P}_0}(h_1(X) = 1| A = 0, Y = 1) \right. \\ & - \left. \mathbb{P}_{(X,A,Y)\sim \mathbb{P}_0}(h_1(X) = 1| A = 1, Y = 1)\right| \\
& = \left|1 - \frac{P_{11} - \eta/2}{P_{11} -\eta/2 + \eta/2}\right| \\
& = \frac{\eta}{2P_{11}}
\end{align*}
and similarly $\mathcal{D}^{opp}(h_1, \mathbb{P}_0) = \frac{\eta}{2P_{10}}$. Since $P_{10}\leq P_{11}$, $\mathcal{D}^{opp}(h_0, \mathbb{P}_0) \leq \mathcal{D}^{opp}(h_1, \mathbb{P}_0)$ and $$\mathcal{D}^{opp}(h_1, \mathbb{P}_0) - \mathcal{D}^{opp}(h_0, \mathbb{P}_0) = \frac{\eta}{2P_{10}}\left(1 - \frac{P_{10}}{P_{11}}\right).$$

Similarly $\mathcal{D}^{opp}(h_0, \mathbb{P}_1) = \frac{\eta}{2P_{10}}$ and $\mathcal{D}^{opp}(h_1, \mathbb{P}_1) = \frac{\eta}{2P_{11}}$, so that $\mathcal{D}^{opp}(h_1, \mathbb{P}_1) \leq \mathcal{D}^{opp}(h_0, \mathbb{P}_1)$ and $$\mathcal{D}^{opp}(h_0, \mathbb{P}_1) - \mathcal{D}^{opp}(h_1, \mathbb{P}_1) = \frac{\eta}{2P_{10}}\left(1 - \frac{P_{10}}{P_{11}}\right).$$
Consider a (randomized) malicious adversary $\mathcal{A}_{i}$ of power $\alpha$, that given a clean distribution $\mathbb{P}_{i}$, changes every marked point to $(x_3,\lnot i, 1)$ with probability $0.5$ and to $(x_4,i, 1)$ otherwise. Under a distribution $\mathbb{P}_{i}$ and an adversary $\mathcal{A}_{i}$, the probability of seeing a point $(x_3, i, 1)$ is $\frac{\eta}{2} (1-\alpha) = \frac{\eta}{2} \frac{1}{1 + \eta} = \alpha/2$, which is equal to the probability of seeing a point $(x_3, \lnot i, 1)$. Therefore, denoting the probability distribution of the corrupted dataset, under a clean distribution $\mathbb{P}_{i}$ and an adversary $\mathcal{A}_{i}$, by $\mathbb{P}'_{i}$, we have
\begin{equation*}
  \mathbb{P}'_{i}(x,a,y) =
    \begin{cases}
      (1-\alpha)(P_{11} - \eta/2) & \text{if $x = x_1, a = 1, y = 1$}\\
      (1-\alpha)(P_{10} - \eta/2) & \text{if $x = x_2, a = 0, y = 1$}\\
      \alpha/2  & \text{if $x = x_3, a = i, y = 1$}\\
      \alpha/2  & \text{if $x = x_3, a = \lnot i, y = 1$}\\
      \alpha/2  & \text{if $x = x_4, a = \lnot i, y = 1$}\\
      \alpha/2  & \text{if $x = x_4, a = i, y = 1$}\\
      (1-\alpha)(1 - P_{10} - P_{11})  & \text{if $x = x_5, a = 0, y = 0$}\\
      0 & \text{otherwise}
    \end{cases}
\end{equation*}
In particular, $\mathbb{P}'_{0} = \mathbb{P}'_{1}$, so the two initial distributions $\mathbb{P}_0$ and $\mathbb{P}_1$ become indistinguishable under the adversarial manipulation.

Fix an arbitrary learner $\mathcal{L}: \cup_{n\in \mathbb{N}} (\prodspace)^n \to \{h_0, h_1\}$. Note that, if the clean distribution is $\mathbb{P}_0$, the events (in the probability space defined by the sampling of the poisoned train data)
\begin{align*}
\{\mathcal{L}(S^p) = h_1\} = \left\{\mathcal{D}^{opp}(\mathcal{L}(S^p), \mathbb{P}_0) - \mathcal{D}^{opp}(h_0, \mathbb{P}_0) \geq \frac{\eta}{2P_{10}}\left(1 - \frac{P_{10}}{P_{11}}\right)\right\}
\end{align*}
are all the same. Similarly, if the clean distribution is $\mathbb{P}_1$
\begin{align*}
\{\mathcal{L}(S^p) = h_0\} = \left\{\mathcal{D}^{opp}(\mathcal{L}(S^p), \mathbb{P}_1) - \mathcal{D}^{opp}(h_1, \mathbb{P}_1) \geq \frac{\eta}{2P_{10}}\left(1 - \frac{P_{10}}{P_{11}}\right)\right\} .
\end{align*}
Therefore, depending on whether we choose $\mathbb{P}_0$ or $\mathbb{P}_1$ as a clean distribution, we have
\begin{align*}
\mathbb{P}_{S^p \sim \mathbb{P}'_0}\left(\mathcal{D}^{opp}(\mathcal{L}(S^p), \mathbb{P}_0) - \mathcal{D}^{opp}(h_0, \mathbb{P}_0) \geq \frac{\eta}{2P_{10}}\left(1 - \frac{P_{10}}{P_{11}}\right) \right) = \mathbb{P}_{S^p \sim \mathbb{P}'_0}\left(\mathcal{L}(S^p) = h_1\right)
\end{align*}
and
\begin{align*}
\mathbb{P}_{S^p \sim \mathbb{P}'_1}\left(\mathcal{D}^{opp}(\mathcal{L}(S^p), \mathbb{P}_1) - \mathcal{D}^{opp}(h_1, \mathbb{P}_1) \geq \frac{\eta}{2P_{10}}\left(1 - \frac{P_{10}}{P_{11}}\right) \right) = \mathbb{P}_{S^p \sim \mathbb{P}'_1}\left(\mathcal{L}(S^p) = h_0\right)
\end{align*}
Finally, note that $\mathbb{P}'_0 = \mathbb{P}'_1$, so that either $\mathbb{P}_{S^p \sim \mathbb{P}'_0}\left(\mathcal{L}(S^p) = h_1\right) \geq 1/2$ or $\mathbb{P}_{S^p \sim \mathbb{P}'_1}\left(\mathcal{L}(S^p) = h_0\right) \geq 1/2$. Moreover, $L(\mathcal{L}(S^p), \mathbb{P}_i) = L(h_i, \mathbb{P}_i) = \eta/2$ holds for both $i\in\{0,1\}$, for any realization of the randomness. Therefore, for at least one of $i = 0,1$, both $$L(\mathcal{L}(S^p), \mathbb{P}_i) = L(h_i, \mathbb{P}_i) = \frac{\eta}{2}$$ and $$\mathcal{D}^{opp}(\mathcal{L}(S^p), \mathbb{P}_i) - \mathcal{D}^{opp}(h_i, \mathbb{P}_i) \geq \frac{\eta}{2P_{10}}\left(1 - \frac{P_{10}}{P_{11}}\right) = \frac{\alpha}{2P_{10}(1-\alpha)}\left(1 - \frac{P_{10}}{P_{11}}\right)$$ both hold with probability at least $1/2$. This concludes the proof in the first case.

\paragraph{Case 2}  Now assume that $\frac{\alpha}{1-\alpha} > 2P_{10}$. Then denote by $\alpha_1$ the unique number between $(0,0.5)$, such that $\frac{\alpha_1}{1-\alpha_1} = 2P_{10}$, and note that $\alpha_1 < \alpha$. Then repeat the same construction as in Case 1, but with $\eta_1 = \frac{\alpha_1}{1-\alpha_1}$ and an adversary that with probability $\alpha_1/\alpha$ does the same as in Case 1 and leaves a marked point untouched otherwise.

Then the same argument as in Case 1 gives that for some $i\in\{0,1\}$, with probability at least $0.5$, both of the following hold
$$L(\mathcal{L}(S^p), \mathbb{P}_i) = L(h_i, \mathbb{P}_i) = \frac{\eta_1}{2} = P_{10}$$ and $$\mathcal{D}^{opp}(\mathcal{L}(S^p), \mathbb{P}_i) - \mathcal{D}^{opp}(h_i, \mathbb{P}_i) \geq \frac{\eta_1}{2P_{10}} = \frac{\eta_1}{2P_{10}}\left(1 - \frac{P_{10}}{P_{11}}\right) = 1 - \frac{P_{10}}{P_{11}}.$$
This concludes the proof of Theorem \ref{thm:lower-bound-equal-opp-with-p-good-loss-app}.
\end{proof}

%% file: appendix_upper_bounds_proofs.tex
\addtocounter{theorem}{-4}

We now present the complete proofs of our upper bounds. The main challenge lies in understanding the concentration properties of the empirical estimates of the fairness measures, as introduced in the main body of the paper. To this end, we first bound the effect that the data corruption may have on these estimates. We then leverage classic concentration techniques to relate the ``ideal'' clean data estimates to the corresponding population fairness measures.

\subsection{Concentration tools and notation}
\label{sec:appendix-notation-and-tools}
We will use the following versions of the classic Chernoff bounds for large deviations of Binomial random variables, as they can be found, for example, in \cite{kearns1993learning}. Let $X \sim Bin(n, p)$. Then
\begin{align*}
\mathbb{P}(X \leq (1-\alpha)pn) \leq e^{-\alpha^2 np/2}
\end{align*}
and 
\begin{align*}
\mathbb{P}(X \geq (1+\alpha)pn) \leq e^{-\alpha^2 np/3},
\end{align*}
for any $\alpha \in (0,1)$. We will also use the Hoeffding's inequality \citep{hoeffding1963probability}. Let $X_1, X_2, \ldots, X_n$ be independent random variables, such that each $X_i$ is bounded in $[a_i, b_i]$ and let $\bar{X} = \frac{1}{n}\sum_{i=1}^n X_i$. Then
\begin{align*}
\mathbb{P}\left(\left|\bar{X} - \mathbb{E}(\bar{X})\right| > t\right) \leq 2\exp\left(-\frac{2n^2t^2}{\sum_{i=1}^n (b_i - a_i)^2}\right).
\end{align*}
Throughout the section we denote the clean data distribution by $\mathbb{P}\in\mathcal{P}(\prodspace)$. As in the main body of the paper, we denote $P_{a} = \mathbb{P}(A = a)$ and $P_{1a} = \mathbb{P}(Y = 1, A = a)$ for both $a\in\{0,1\}$. We assume without loss of generality that $0 < P_0 \leq \frac{1}{2} \leq P_1$ (when studying demographic parity) and $0 < P_{10} \leq P_{11}$ (when studying equal opportunity).

We will be interested in the concentration properties of certain empirical estimates based on the corrupted data $S^p$. Therefore, we denote the distribution that corresponds to all the randomness of the sampling of $S^p$, that is the randomness of the clean data, the marked points and the adversary, by $\mathbb{P}^{\mathcal{A}}$. Here we consider both $\mathbb{P}$ and $\mathcal{A}$ arbitrary, but fixed.

\subsection{Concentration results}
\label{sec:concentration-lemmas-proofs}

We study the concentration of the demographic parity and the equal opportunity fairness estimates in Sections \ref{sec:concentration-lemmas-proofs-demog-par} and \ref{sec:concentration-lemmas-proofs-equal-opp} respectively.

\subsubsection{Concentration for demographic parity}
\label{sec:concentration-lemmas-proofs-demog-par}
We use the notation $C_a = \sum_{i=1}^n \mathbbm{1}\{a^p_i = a, i\not\in\poisoned\}$ for the number of points in $S^p$ that \textit{were not} marked (that is, are \textit{clean}) and  contain a point from protected group $a$ and $B_a = \sum_{i=1}^n \mathbbm{1}\{a^p_i = a, i\in\poisoned\}$ for the number of points in $S^p$ that \textit{were} marked (that is, are potentially \textit{bad}\footnote{We use $B_a$ with $B$ for \textit{bad} here, instead of $P$ for \textit{poisoned}, to avoid confusion with the protected group frequencies $P_i$.}) and  contain a point from protected group $a$. Note that $B_0 + B_1 = |\poisoned|$ is the total number of poisoned points, which is $\Bin(n, \alpha)$, and $B_0 + B_1 = n - C_0 - C_1$. Similarly, denote by $C^{1}_a(h) = \sum_{i=1}^n \mathbbm{1}\{h(x^p_i) = 1, a^p_i = a, i\not\in\poisoned\}$ and $B^{1}_a(h) = \sum_{i=1}^n \mathbbm{1}\{h(x^p_i) = 1, a^p_i = a, i\in\poisoned\}$.
\\
\noindent Denote
\begin{align*}
\gamma^p_a(h) = \frac{\sum_{i=1}^n \mathbbm{1}\{h(x^p_i) = 1, a^p_i = a\}}{\sum_{i=1}^n \mathbbm{1}\{a^p_i = a\}}
\end{align*}
and 
\begin{align*}
\gamma_a(h) = \mathbb{P}(h(X) = 1| A = a),
\end{align*}
so that $\widehat{\mathcal{D}}^{par}(h) = |\gamma^p_0(h) - \gamma^p_1(h)|$ and $\mathcal{D}^{par}(h) = |\gamma_0(h) - \gamma_1(h)|$.  Note that $\gamma^p_a(h)$ is an estimate of a conditional probability \textit{based on the corrupted data}. We now introduce the corresponding estimate that only uses the clean (but unknown) subset of the training set $S^p$
\begin{align*}
\gamma^c_a(h) = \frac{C^1_a(h)}{C_a(h)} = \frac{\sum_{i=1}^n \mathbbm{1}\{h(x^p_i) = 1, a^p_i = a, i\not\in\poisoned\}}{\sum_{i=1}^n \mathbbm{1}\{a^p_i = a, i\not\in\poisoned\}}.
\end{align*}

First we bound how far the corrupted estimates $\gamma^p_a(h)$ of $\gamma_a(h)$ are from the clean estimates $\gamma^c_a(h)$, uniformly over the hypothesis space $\mathcal{H}$:
\begin{lemma}
\label{lemma:upper-bound-delta-par-app}
If $n \geq \max \left\{\frac{8\log(4/\delta)}{(1-\alpha)P_0}, \frac{12\log(3/\delta)}{\alpha}\right\}$, we have
\begin{align}
\label{eqn:bound_on_delta_appendix}
\mathbb{P}^{\mathcal{A}}\left(\sup_{h\in\mathcal{H}}\left(\left|\gamma^p_0(h) - \gamma^c_0(h)\right| + \left|\gamma^p_1(h) - \gamma^c_1(h)\right|\right) \geq \frac{2\alpha}{P_0/3 + \alpha}\right) < \delta.
\end{align}
\end{lemma}
\begin{proof}
First we show that certain bounds on the random variables $B_a$ and $C_a$ hold with high probability. Then we show that the supremum in equation (\ref{eqn:bound_on_delta_appendix}) is bounded when these bounds hold.

\paragraph{Step 1} Specifically, since $B_0 + B_1 \sim \Bin(n, \alpha)$, by the Chernoff bounds and the assumption on $n$
\begin{align*}
\mathbb{P}^{\mathcal{A}}\left(B_0 + B_1 \geq \frac{3\alpha}{2}n\right) \leq e^{-\alpha n/12} \leq \frac{\delta}{3}.
\end{align*}
Similarly, $C_0 \sim \Bin(n, (1-\alpha)P_0)$ and $C_1 \sim \Bin(n, (1-\alpha)P_1)$ and since $P_0\leq P_1$ we get
\begin{align*}
\mathbb{P}^{\mathcal{A}}\left(C_0 \leq \frac{1-\alpha}{2}P_0n\right) \leq e^{-(1-\alpha) P_0 n/8} \leq \frac{\delta}{4}
\end{align*}
and
\begin{align*}
\mathbb{P}^{\mathcal{A}}\left(C_1 \leq \frac{1-\alpha}{2}P_1n\right) \leq e^{-(1-\alpha) P_1n/8} \leq \frac{\delta}{4}
\end{align*}
Therefore, by a union bound
\begin{align*}
\mathbb{P}^{\mathcal{A}}\left(\left(B_0 + B_1 \geq \frac{3\alpha}{2}n\right) \lor \left(C_0 \leq \frac{1-\alpha}{2}P_0n\right) \lor \left(C_1 \leq \frac{1-\alpha}{2}P_1n\right)\right) \leq \frac{\delta}{3} + \frac{\delta}{4} + \frac{\delta}{4} < \delta.
\end{align*}

\paragraph{Step 2} Now assume that all of $B_0 + B_1 < \frac{3\alpha}{2}n$, $C_0 > \frac{1-\alpha}{2}P_0n$, $C_1 > \frac{1-\alpha}{2}P_1n$ hold. This happens with probability at least $1-\delta$ according to Step 1. Let $h$ be an arbitrary classifier. Since we consider $h$ fixed, we will drop the dependence on $h$ from the notation for the rest of this proof and write $\gamma^p_a = \gamma^p_a(h), C^1_a = C^1_a(h)$, \etc .

We now prove that for both $a\in\{0,1\}$
\begin{equation}
\label{eqn:lemma_upper_bound_delta_par_what_to_prove}
\Delta_a \vcentcolon = \left|\gamma^p_a - \gamma^c_a\right| \leq \frac{B_a}{C_a + B_a}.
\end{equation}
For each $a\in\{0,1\}$, this can be shown as follows. First, if $\sum_{i=1}^n \mathbbm{1}\{a^p_i = a\} = B_a + C_a = 0$, then both $\gamma^p_a(h)$ and $\gamma^c_a(h)$ are equal to $0$, because of the convention that $\frac{0}{0} = 0$. In addition, $B_a = C_a = 0$. Therefore, inequality (\ref{eqn:lemma_upper_bound_delta_par_what_to_prove}) trivially holds. 

Similarly, if $B_a = 0$, but $C_a > 0$, then $\gamma^p_a(h) = \gamma^c_a(h)$ and so $\Delta_a = 0$ and (\ref{eqn:lemma_upper_bound_delta_par_what_to_prove}) holds.

Assume now that $B_a > 0$. Note that if $C_a = \sum_{i=1}^n \mathbbm{1}\{a^p_i = a, i\not\in\poisoned\} = 0$, then $$\Delta_a = \left|\gamma^p_a(h) - \gamma^c_a(h)\right| =  \left|\frac{B^1_a}{B_a} - 0\right| = \frac{B^1_a}{B_a} = \frac{B^1_a}{B_a + C_a} \leq \frac{B_a}{C_a + B_a}.$$
Finally, assume that both $C_a > 0$ and $B_a > 0$. Note that under any realization of the randomness of the data sampling and the adversary, for any $a\in\{0,1\}$
\begin{align*}
\gamma^p_a(h) & = \frac{\sum_{i=1}^n \mathbbm{1}\{h(x^p_i) = 1, a^p_i = a\}}{\sum_{i=1}^n \mathbbm{1}\{a^p_i = a\}} \\ & = \frac{\sum_{i=1}^n \mathbbm{1}\{h(x^p_i) = 1, a^p_i = a, i\not\in\poisoned\} + \sum_{i=1}^n \mathbbm{1}\{h(x^p_i) = 1, a^p_i = a, i\in\poisoned\}}{\sum_{i=1}^n \mathbbm{1}\{a^p_i = a, i\not\in\poisoned\} + \sum_{i=1}^n \mathbbm{1}\{a^p_i = a, i\in\poisoned\}} \\
& = \frac{C^1_a + B^1_a}{C_a + B_a}.
\end{align*} 
Therefore,
\begin{align*}
\Delta_a = \left|\gamma^p_a - \gamma^c_a\right| = \left|\frac{C^1_a + B^1_a}{C_a + B_a} - \frac{C^1_a}{C_a}\right| = \frac{B_a}{C_a + B_a}\left|\frac{C^1_a}{C_a} - \frac{B^1_a}{B_a}\right| \leq \frac{B_a}{C_a + B_a}
\end{align*}
and so (\ref{eqn:lemma_upper_bound_delta_par_what_to_prove}) holds in all cases. Therefore, we can bound the sum $\Delta_0 + \Delta_1$ as follows:

\begin{align*}
\Delta_0 + \Delta_1 & \leq \frac{B_0}{C_0 + B_0} + \frac{B_1}{C_1 + B_1} \\ & < \frac{B_0}{\frac{1-\alpha}{2}P_0 n + B_0} + \frac{B_1}{\frac{1-\alpha}{2}P_1 n + B_1} \\
& \leq \frac{B_0}{\frac{1-\alpha}{2}P_0 n + B_0} + \frac{B_1}{\frac{1-\alpha}{2}P_0 n + B_1} \\ 
& = \frac{B_0}{\frac{1-\alpha}{2}P_0 n + B_0}  + 1 - \frac{\frac{1-\alpha}{2}P_0 n}{\frac{1-\alpha}{2}P_0 n - B_0 + (B_0 + B_1)} \\
& < \frac{B_0}{\frac{1-\alpha}{2}P_0 n + B_0}  + 1 - \frac{\frac{1-\alpha}{2}P_0 n}{\frac{1-\alpha}{2}P_0 n - B_0 + \frac{3\alpha}{2}n} \\
& = 2 - (1-\alpha)P_0n\left(\frac{1}{(1-\alpha)P_0 n + 2B_0} + \frac{1}{(1-\alpha)P_0 n + 3\alpha n - 2B_0}\right) 
\end{align*}
Studying the function $f(x) = \frac{1}{(1-\alpha)P_0 n + 2x} + \frac{1}{(1-\alpha)P_0 n + 3\alpha n - 2x}$, we see that
\begin{align*}
f'(x) = 2\left(\frac{1}{((1-\alpha)P_0 n + 3\alpha n - 2x)^2} - \frac{1}{((1-\alpha)P_0 n + 2x)^2}\right).
\end{align*}
Note that $B_0 \leq B_0 + B_1 < \frac{3\alpha}{2}$, so we may assume $0 \leq x < \frac{3\alpha}{2}$. Therefore, both $(1-\alpha)P_0 n + 3\alpha n - 2x > 0$ and $(1-\alpha)P_0 n + 2x > 0$. Therefore, $f'(x) = 0$ if and only if $(1-\alpha)P_0 n + 3\alpha n - 2x = (1-\alpha)P_0 n + 2x$, that is, $x = \frac{3\alpha}{4}n$. Moreover $f'(x) < 0$ if $x\in [0,\frac{3\alpha}{4}n)$ and $f'(x) > 0$ if $x\in (\frac{3\alpha}{4}n, \frac{3\alpha}{2}n)$. Therefore, $f(x)$ is minimized at $x = \frac{3\alpha}{4}n$ and so
\begin{align*}
\Delta_0 + \Delta_1 & \leq 2 - (1-\alpha)P_0n\left(\frac{1}{(1-\alpha)P_0 n + 2B_0} + \frac{1}{(1-\alpha)P_0 n + 3\alpha n - 2B_0}\right)\\
& \leq 2 - (1-\alpha)P_0n\left(\frac{1}{(1-\alpha)P_0 n +  \frac{3\alpha}{2}n} + \frac{1}{(1-\alpha)P_0 n + 3\alpha n - \frac{3\alpha}{2}n}\right)\\
& = \frac{6\alpha}{2(1-\alpha)P_0 + 3\alpha}\\
& \leq \frac{6\alpha}{P_0 + 3\alpha} = \frac{2\alpha}{P_0/3 + \alpha}
\end{align*}
and hence (\ref{eqn:lemma_upper_bound_delta_par_what_to_prove}) holds in this case as well. Since the derivations hold for any classifier $h\in\mathcal{H}$, the result follows.
\end{proof}

For the rest of the section, we keep the notation $\Delta_a(h) = \left|\gamma^p_a(h) - \gamma^c_a(h)\right|$ for $a\in\{0,1\}$ and $\Delta^{par} = \frac{2\alpha}{P_0/3 + \alpha}$. 

Next we use the previous result and the technique of \citep{woodworth2017learning} for proving concentration results about conditional probability estimates to bound the probability of a large deviation of $\widehat{\mathcal{D}}^{par}(h)$ from $\mathcal{D}^{par}(h)$, for a fixed hypothesis $h\in\mathcal{H}$.

\begin{lemma}
\label{lemma:non-uniform-bound-demog-parity}
Let $h\in\mathcal{H}$ be a fixed hypothesis and $\mathbb{P}\in \mathcal{P}(\prodspace)$ be a fixed distribution. Denote $P_{a} = \mathbb{P}(A = a)$ for $a\in\{0,1\}$. Let $\mathcal{A}$ be any malicious adversary and denote by $\mathbb{P}^{\mathcal{A}}$ the probability distribution of the poisoned data $S^p$, under the random sampling of the clean data, the marked points and the randomness of the adversary. Then for any $n \geq \max \left\{\frac{8\log(8/\delta)}{(1-\alpha) P_0}, \frac{12\log(6/\delta)}{\alpha}\right\}$ and $\delta\in(0,1)$
\begin{equation}
\label{eqn:non-uniform-bound-demog-parity}
\mathbb{P}^{\mathcal{A}}\left(\left|\widehat{\mathcal{D}}^{par}(h) - \mathcal{D}^{par}(h)\right| \leq \Delta^{par} + 2\sqrt{\frac{\log(16/\delta)}{n(1-\alpha)P_0}}\right) \geq 1-\delta.
\end{equation}
\end{lemma}
\begin{proof}
Again we write $\gamma^p_a = \gamma^p_a(h), C^1_a = C^1_a(h)$, \etc since $h$ is fixed. First we study the concentration of the clean estimate $\frac{C^1_a}{C_a}$ around $\gamma_a$. To this end, denote by $S^c_{a} = \{i: a^p_i = a, i \not\in \poisoned\}$ the set of indexes of the poisoned data for which the protected group is $a$ and the corresponding point was not marked for the adversary. Notice that $S^c_{a}$ is a random variable and that $|S^c_{a}| = C_a$. Since $n \geq \frac{8\log(8/\delta)}{(1-\alpha)P_a}$ for both $a\in\{0,1\}$, we have
\begin{align*}
\mathbb{P}^{\mathcal{A}}\left(\left|\gamma^c_a - \gamma_{a}\right| > t\right) & = \sum_{S^c_{a}} \mathbb{P}^{\mathcal{A}}\left(\left|\gamma^c_a - \gamma_{a}\right| > t\middle| S^c_{a}\right)\mathbb{P}(S^c_{a}) \\
& \leq \mathbb{P}^{\mathcal{A}}\left(C_a \leq \frac{(1-\alpha)}{2}P_an\right) \\
& + \sum_{S^c_{a}: C_a > \frac{(1-\alpha)}{2}P_an}\mathbb{P}^{\mathcal{A}}\left(\left|\frac{C^1_a}{C_a} - \gamma_{a}\right| > t\middle| S^c_{a}\right)\mathbb{P}^{\mathcal{A}}(S^c_{a}) \\ 
& \leq \exp \left(-\frac{(1-\alpha)P_a n}{8}\right) + \sum_{S^p_{a}: C_a > \frac{(1-\alpha)}{2}P_an} 2 \exp\left(-2t^2C_a\right) \mathbb{P}^{\mathcal{A}}(S^c_{a}) \\ 
& \leq \frac{\delta}{8} + 2\exp\left(-t^2 (1-\alpha)P_{a}n\right),
\end{align*}
where the second inequality follows from Hoeffding's inequality. Note that this step crucially uses that the marked indexes are independent of the data. The triangle law gives
\begin{align*}
||\gamma^p_{0} - \gamma^p_{1}| - |\gamma_{0} - \gamma_{1}|| \leq |\gamma^p_{0} - \gamma^p_{1} - \gamma_{0} + \gamma_{1}| & \leq |\gamma^p_{0} - \gamma_{0}| + |\gamma^p_{1} - \gamma_{1}| \\
& \leq \left|\gamma^p_{0} - \gamma^c_0\right| + \left|\gamma^c_0 - \gamma_0\right| + \left|\gamma^p_{1} - \gamma^c_1\right| + \left|\gamma^c_1 - \gamma_1\right| \\
& = \left|\gamma^c_0 - \gamma_0\right| + \left|\gamma^c_1 - \gamma_1\right| + \Delta_0 + \Delta_1.
\end{align*}
Combining the previous two results (recall that we assume $P_0 \leq P_1$)
\begin{align*}
\mathbb{P}^{\mathcal{A}} &(||\gamma^p_{0} - \gamma^p_{1}| - |\gamma_{0} - \gamma_{1}|| > 2t + \Delta_0 + \Delta_1) \\ 
& \leq \mathbb{P}^{\mathcal{A}}\left(\left|\gamma^c_0 - \gamma_0\right| + \left|\gamma^c_1 - \gamma_1\right| + \Delta_0 + \Delta_1 > 2t + \Delta_0 + \Delta_1\right) \\ & \leq \mathbb{P}^{\mathcal{A}}\left(\left(\left|\gamma^c_0 - \gamma_0\right| > t\right) \lor \left(\left|\gamma^c_1 - \gamma_1\right| > t\right)\right) \\ & \leq \mathbb{P}^{\mathcal{A}}\left(\left|\gamma^c_0 - \gamma_0\right| > t\right) + \mathbb{P}^{\mathcal{A}}\left(\left|\gamma^c_1 - \gamma_1\right| > t\right) \\ & \leq \frac{\delta}{4} + 4\exp (-t^2n (1-\alpha) P_0).
\end{align*}
Setting $t = t_0 = \sqrt{\frac{\log(16/\delta)}{n(1-\alpha)P_0}}$ gives
\begin{align}
\label{eqn:high_prob_bound_with_E}
\mathbb{P}^{\mathcal{A}}\left(||\gamma^p_{0} - \gamma^p_{1}| - |\gamma_{0} - \gamma_{1}|| > \Delta_0 + \Delta_1 + 2\sqrt{\frac{\log(16/\delta)}{n(1-\alpha)P_0}}\right)  \leq \frac{\delta}{4} + 4\frac{\delta}{16} = \frac{\delta}{2}.
\end{align}
In addition Lemma \ref{lemma:upper-bound-delta-par-app} gives
\begin{align}
\label{eqn:bound_on_E_inside_proof}
\mathbb{P}^{\mathcal{A}}\left(\Delta_0 + \Delta_1 > \Delta^{par}\right) \leq \frac{\delta}{2}.
\end{align}
Using (\ref{eqn:high_prob_bound_with_E}) and (\ref{eqn:bound_on_E_inside_proof}) we obtain that:
\begin{align*}
& \mathbb{P}^{\mathcal{A}}\left(||\gamma^p_{0} - \gamma^p_{1}| - |\gamma_{0} - \gamma_{1}|| \leq \Delta^{par} + 2\sqrt{\frac{\log(16/\delta)}{N(1-\alpha)P_0}}\right) \\
& \geq \mathbb{P}^{\mathcal{A}}\left(\left(||\gamma^p_{0} - \gamma^p_{1}| - |\gamma_{0} - \gamma_{1}|| \leq \Delta_0 + \Delta_1 + 2\sqrt{\frac{\log(16/\delta)}{N(1-\alpha)P_0}}\right) \land \left(\Delta_0 + \Delta_1 \leq  \Delta^{par}\right)\right) \\
& \geq 1 - \frac{\delta}{2} - \frac{\delta}{2} = 1 - \delta.
\end{align*}
\end{proof}

Finally, we show how to extend the previous result to hold uniformly over the whole hypothesis space, provided that $\mathcal{H}$ has a finite VC-dimension $d \vcentcolon = VC(\mathcal{H})$

\begin{lemma}
\label{lemma:uniform-bound-demog-par}
Under the setup of Lemma \ref{lemma:non-uniform-bound-demog-parity}, assume additionally that $\mathcal{H}$ has a finite VC-dimension $d$. Then for any $n \geq \max \left\{\frac{8\log(8/\delta)}{(1-\alpha) P_0}, \frac{12\log(6/\delta)}{\alpha}, \frac{d}{2}\right\}$ and $\delta\in(0,1)$
\begin{align}
\label{eqn:uniform-bound-demog-parity}
\mathbb{P}^{\mathcal{A}}_{S^p} & \left(\sup_{h\in\mathcal{H}}|\widehat{\mathcal{D}}^{par}(h) - \mathcal{D}^{par}(h)| \leq \Delta^{par} +  16 \sqrt{\frac{2d\log(\frac{2en}{d}) + 2\log(48/\delta)}{(1-\alpha)P_0n}}\right) \geq 1 - \delta.
\end{align}
\end{lemma}

\begin{proof}
From Lemma \ref{lemma:upper-bound-delta-par-app}, we have that whenever $n \geq \max \left\{\frac{8\log(8/\delta)}{(1-\alpha) P_0}, \frac{12\log(6/\delta)}{\alpha}\right\}$ and $\delta\in(0,1)$
\begin{align}
\label{eqn:bound_on_delta_appendix_uniform}
\mathbb{P}^{\mathcal{A}}\left(\sup_{h\in\mathcal{H}}\left(\Delta_0(h) + \Delta_1(h) \right) \geq \Delta^{par}\right) < \frac{\delta}{2}.
\end{align}
Additionally, in the proof of Lemma \ref{lemma:non-uniform-bound-demog-parity} we showed that for a fixed classifier $h\in\mathcal{H}$ for any $\delta\in(0,1), t\in (0,1)$ and both $a\in\{0,1\}$, we have
\begin{align}
\label{eqn:clean-concentration-non-uniform-demog-par}
\mathbb{P}^{\mathcal{A}}\left(\left|\gamma^c_a(h) - \gamma_{a}(h)\right| > t\right) & \leq \exp \left(-\frac{(1-\alpha)P_a n}{8}\right) + 2\exp\left(-t^2 (1-\alpha)P_{a}n\right) \notag \\ 
& \leq 3\exp\left(-\frac{t^2 (1-\alpha)P_{a}n}{8}\right).
\end{align}

The proof consists of two steps. In Steps 1 and 2 we show how to extend inequality (\ref{eqn:clean-concentration-non-uniform-demog-par}) to hold uniformly over $\mathcal{H}$. Then, we combine the two uniform bounds with a similar argument as in the proof of Lemma \ref{lemma:non-uniform-bound-demog-parity}.

The first step uses the classic symmetrization technique \citep{vapnik2013nature} for proving bounds uniformly over hypothesis spaces of finite VC dimension. However, since the objective is different from the 0-1 loss, care is needed to ensure that the proof goes through, so we present it here in full detail.

\paragraph{Step 1} To make the dependence of the left-hand side of (\ref{eqn:clean-concentration-non-uniform-demog-par}) on both $h$ and the data $S^p$ explicit, we set $\gamma^c_a(h, S^p) \vcentcolon = \frac{C^1_a(h)}{C_a}$.

Introduce a ghost sample $S^1 = \{(x^1_i, a^1_i, y^1_i)\}_{i=1}^n$ also sampled in an \iid manner from $\mathbb{P}^{\mathcal{A}}$, that is, $S^1$ is another, independent poisoned dataset \footnote{Formally, we associate $S^1$ also with a set $\poisoned_1$ of marked indexes.}. Let $\gamma^c_a(h, S^1)$ be the empirical estimate of $\gamma_a(h)$ based on $S^1$.

First we show a symmetrization inequality for the $\gamma_a$ measures
\begin{align}
\label{eqn:symmetrization-demog-par}
\mathbb{P}^{\mathcal{A}}_{S^p}\left(\sup_{h\in\mathcal{H}}\left|\gamma_a(h) - \gamma^c_a(h, S^p)\right| \geq t\right) & \leq 
 2 \mathbb{P}^{\mathcal{A}}_{S^p, S^1}\left(\sup_{h\in\mathcal{\mathcal{H}}}\left|\gamma^c_a(h, S^1) - \gamma^c_a(h, S^p)\right| \geq t/2\right),
\end{align}
for any constant $1 > t \geq 2\sqrt{\frac{8\log(6)}{(1-\alpha)P_0n}}$.
 
Indeed, let $h^*$ be the hypothesis achieving the supremum on the left-hand side \footnote{If the supremum is not attained, this argument can be repeated for each element of a sequence of classifiers approaching the supremum}. Note that
\begin{align*}
\mathbbm{1}(\left|\gamma_a(h^{*}) - \gamma^c_a(h^*,S^p)\right| \geq t) & \mathbbm{1}(\left|\gamma_a(h^*) - \gamma^c_a(h^*, S^1)\right| \leq t/2) \\ & \leq \mathbbm{1}(\left|\gamma^c_a(h^*,S^1) - \gamma^c_a(h^*,S^p)\right| \geq t/2).
\end{align*}
Taking expectation with respect to $S^1$
\begin{align*}
\mathbbm{1}(\left|\gamma_a(h^{*}) - \gamma^c_a(h^*,S^p)\right| \geq t) & \mathbb{P}^{\mathcal{A}}_{S^1}(\left|\gamma_a(h^*) - \gamma^c_a(h^*,S^1)\right| \leq t/2) \\ & \leq \mathbb{P}^{\mathcal{A}}_{S^1}(\left|\gamma^c_a(h^*,S^1) - \gamma^c_a(h^*,S^p)\right| \geq t/2).
\end{align*}
Now using Lemma \ref{lemma:non-uniform-bound-demog-parity}
\begin{align*}
\mathbb{P}^{\mathcal{A}}_{S^1}\left(\left|\gamma_a(h^*) - \gamma^c_a(h^*,S^1)\right| \leq t/2\right) & \geq \mathbb{P}^{\mathcal{A}}_{S^1}\left(\left|\gamma_a(h^*) - \gamma^c_a(h^*, S^1)\right| \leq \sqrt{\frac{8\log(6)}{(1-\alpha)P_an}}\right) \\
& \geq 1 - \frac{1}{2} \\
& = \frac{1}{2}.
\end{align*}
so
\begin{align*}
\frac{1}{2}\mathbbm{1}(\left|\gamma_a(h^{*}) - \gamma^c_a(h^*,S^p)\right| \geq t) \leq \mathbb{P}^{\mathcal{A}}_{S^1}(\left|\gamma^c_a(h^*,S^1) - \gamma^c_a(h^*,S^p)\right| \geq t/2).
\end{align*}
Taking expectation with respect to $S^p$
\begin{align*}
\mathbb{P}^{\mathcal{A}}_{S^p}(\left|\gamma_a(h^{*}) - \gamma^c_a(h^*,S^p)\right| \geq t) & \leq 2 \mathbb{P}^{\mathcal{A}}_{S^p, S^1}(\left|\gamma^c_a(h^*,S^1) - \gamma^c_a(h^*,S^p)\right| \geq t/2) \\
& \leq 2 \mathbb{P}^{\mathcal{A}}_{S^p, S^1}(\sup_{h\in\mathcal{H}}\left|\gamma^c_a(h,S^1) - \gamma^c_a(h,S^p)\right| \geq t/2).
\end{align*}

\paragraph{Step 2} Next we use the growth function of $\mathcal{H}$ and the symmetrization inequality (\ref{eqn:symmetrization-demog-par}) to bound the large deviations of $\gamma^c_a(h)$ uniformly over $\mathcal{H}$.

Specifically, gives $n$ points $x_1, \ldots, x_n \in \mathcal{X}$, denote $$ \mathcal{H}_{x_1, \ldots, x_n} \{(h(x_1), \ldots, h(x_n)): h \in \mathcal{H}\}.$$ 

Then define the growth function of $\mathcal{H}$ as
\begin{equation*}
S_{\mathcal{H}}(n) = \sup_{x_1, \ldots, x_n} |\mathcal{H}_{x_1, \ldots, x_n}|.
\end{equation*}
We will use that well-known Sauer's lemma (see, for example, \citep{bousquet2003introduction}), which states that whenever $n\geq d$, $S_{\mathcal{H}}(n) \leq \left(\frac{en}{d}\right)^d$

Notice that given the two datasets $S^p, S^1$ and the corresponding sets of marked indexes, the values of $\gamma^c_a(h, S^p)$ and $\gamma^c_a(h, S^1)$ depend only on the values of $h$ on $S^p$ and $S^1$ respectively. Therefore for any $1 > t \geq  2\sqrt{\frac{8\log(6)}{(1-\alpha)P_0n}}$,
\begin{align*}
\mathbb{P}^{\mathcal{A}}_{S^p} & \left(\sup_{h\in\mathcal{H}}\left|\gamma_a(h) - \gamma^c_a(h, S^p)\right| \geq t\right) \\ 
& \leq 2 \mathbb{P}^{\mathcal{A}}_{S^p, S^1}\left(\sup_{h\in\mathcal{\mathcal{H}}}\left|\gamma^c_a(h, S^1) - \gamma^c_a(h, S^p)\right| \geq t/2\right) \\
 & \leq 2 S_{\mathcal{H}}(2n) \mathbb{P}^{\mathcal{A}}_{S^p, S^1}\left(\left|\gamma^c_a(h, S^1) - \gamma^c_a(h, S^p)\right| \geq t/2\right) \\
 & \leq 2 S_{\mathcal{H}}(2n) \mathbb{P}^{\mathcal{A}}_{S^p, S^1}\left(\left|\gamma^c_a(h, S^1) - \gamma^c_a(h)\right| \geq t/4 \lor  \left|\gamma^c_a(h, S^p) - \gamma^c_a(h)\right| \geq t/4\right) \\
 & \leq 4 S_{\mathcal{H}}(2n) \mathbb{P}^{\mathcal{A}}_{S^p}\left(\left|\gamma^c_a(h, S^p) - \gamma^c_a(h)\right| \geq t/4\right)\\
 & \leq 12 S_{\mathcal{H}}(2n) \exp\left(-\frac{t^2(1-\alpha)P_an}{128}\right).
\end{align*}
Using $P_0 \leq P_1$ and Sauer's lemma, whenever $2n \geq d$ we have
\begin{equation*}
\mathbb{P}^{\mathcal{A}}_{S^p}\left(\sup_{h\in\mathcal{H}}\left|\gamma_a(h) - \gamma^c_a(h, S^p)\right| \geq t\right)  \leq 12 \left(\frac{2en}{d}\right)^d  \exp\left(-\frac{t^2(1-\alpha)P_0 n}{128}\right).
\end{equation*}
Using inversion, we get that
\begin{equation}
\label{eqn:clean-concentration-uniform-demog-par}
\mathbb{P}^{\mathcal{A}}_{S^p}\left(\sup_{h\in\mathcal{H}}\left|\gamma_a(h) - \gamma^c_a(h, S^p)\right| \geq  8 \sqrt{\frac{2d\log(\frac{2en}{d}) + 2\log(48/\delta)}{(1-\alpha)P_0n}}\right)  \leq \frac{\delta}{4},
\end{equation}
whenever
\begin{equation*}
1 > 8 \sqrt{2\frac{d\log(\frac{2en}{d}) + 2\log(12/\delta)}{(1-\alpha)P_0n}} \geq 2\sqrt{\frac{8\log(6)}{(1-\alpha)P_0n}}.
\end{equation*}
It's easy to see that the right inequality holds whenever $\delta < 1$ and $2n \geq d$. In addition, inequality (\ref{eqn:clean-concentration-uniform-demog-par}) trivially holds if the left inequality is not fulfilled. Therefore, (\ref{eqn:clean-concentration-uniform-demog-par}) holds whenever $2n \geq d$.

\paragraph{Step 3} Finally, we use (\ref{eqn:bound_on_delta_appendix_uniform}) and (\ref{eqn:clean-concentration-uniform-demog-par}) to proof the lemma. Recall from the proof of Lemma \ref{lemma:non-uniform-bound-demog-parity} that
\begin{align*}
|\widehat{\mathcal{D}}^{par}(h) - \mathcal{D}^{par}(h)| & = ||\gamma^c_0(h, S^p) - \gamma^c_1(h, S^p)| - |\gamma_0(h) - \gamma_1(h)|| \\ & \leq \left|\gamma^c_0(h, S^p) - \gamma_0(h)\right| + \left| \gamma^c_1(h, S^p) - \gamma_1(h)\right| + \Delta_0(h) + \Delta_1(h).
\end{align*}
Therefore, 
\begin{align*}
\sup_{h\in\mathcal{H}}|\widehat{\mathcal{D}}^{par}(h) - \mathcal{D}^{par}(h)| \leq \sup_{h\in\mathcal{H}}\left|\gamma^c_0(h, S^p) - \gamma_0(h)\right| & + \sup_{h\in\mathcal{H}}\left| \gamma^c_1(h, S^p) - \gamma_1(h)\right| \\ 
& + \sup_{h\in\mathcal{H}}(\Delta_0(h) + \Delta_1(h)).
\end{align*}
Now, using the union bound and inequalities (\ref{eqn:bound_on_delta_appendix_uniform}) and (\ref{eqn:clean-concentration-uniform-demog-par}), whenever $$n \geq \max \left\{\frac{8\log(8/\delta)}{(1-\alpha) P_0}, \frac{12\log(6/\delta)}{\alpha}, \frac{d}{2}\right\}$$ we get
\begin{align*}
\mathbb{P}^{\mathcal{A}}_{S^p} & \left(\sup_{h\in\mathcal{H}}|\widehat{\mathcal{D}}^{par}(h) - \mathcal{D}^{par}(h)| \geq \Delta^{par} +  16 \sqrt{\frac{2d\log(\frac{2en}{d}) + 2\log(48/\delta)}{(1-\alpha)P_0n}}\right) \\
& \leq \mathbb{P}^{\mathcal{A}}_{S^p}\left(\sup_{h\in\mathcal{H}}\left|\gamma_0(h) - \gamma^c_0(h, S^p)\right| \geq  8 \sqrt{\frac{2d\log(\frac{2en}{d}) + 2\log(48/\delta)}{(1-\alpha)P_0n}}\right) \\
& + \mathbb{P}^{\mathcal{A}}_{S^p}\left(\sup_{h\in\mathcal{H}}\left|\gamma_1(h) - \gamma^c_1(h, S^p)\right| \geq  8 \sqrt{\frac{2d\log(\frac{2en}{d}) + 2\log(48/\delta)}{(1-\alpha)P_0n}}\right) \\
& + \mathbb{P}^{\mathcal{A}}\left(\sup_{h\in\mathcal{H}}\left(\Delta_0(h) + \Delta_1(h) \right) \geq \Delta^{par}\right) \\
& \leq \frac{\delta}{4} + \frac{\delta}{4} + \frac{\delta}{2} = \delta
\end{align*}
\end{proof}

\subsubsection{Concentration for equal opportunity}
\label{sec:concentration-lemmas-proofs-equal-opp}
We introduce similar notation as in Section \ref{sec:concentration-lemmas-proofs-demog-par}, but tailored to the equal opportunity conditional probabilities.

We use the notation $C_{1a} = \sum_{i=1}^n \mathbbm{1}\{i: a^p_i = a, y^p_i = 1, i\not\in\poisoned\}|$ for the number of points in $S^p$ that \textit{were not} marked (are \textit{clean}) and  contain a point from protected group $a$ and label $y = 1$ and $B_{1a} = \sum_{i=1}^n \mathbbm{1}\{i: a^p_i = a, y^p_i = 1, i\in\poisoned\}|$ for the number of points in $S^p$ that \textit{were} marked (are potentially \textit{bad}) and  contain a point from protected group $a$ and label $y = 1$. Note that $B_{10} + B_{11}$ is the total number of poisoned points for which $y = 1$ and so is at most $\Bin(n, \alpha)$. Similarly, denote by $C^{1}_{1a}(h) = \sum_{i=1}^n \mathbbm{1}\{i: h(x^p_i) = 1, a^p_i = a, y^p_1 = 1, i\not\in\poisoned\}|$ and $B^{1}_{1a}(h) = \sum_{i=1}^n \mathbbm{1}\{i: h(x^p_i) = 1, a^p_i = a, y^p_i = 1, i\in\poisoned\}|$.

\noindent Denote
\begin{align*}
\gamma^p_{1a}(h) = \frac{\sum_{i=1}^n \mathbbm{1}\{h(x^p_i) = 1, a^p_i = a, y^p_1 = 1\}}{\sum_{i=1}^n \mathbbm{1}\{a^p_i = a, y^p_i = 1\}}
\end{align*}
and $$\gamma_{1a}(h) = \mathbb{P}(h(X) = 1| A = a, Y = 1),$$ so that $\widehat{\mathcal{D}}^{opp}(h) = |\gamma^p_{10}(h) - \gamma^p_{11}(h)|$ and $\mathcal{D}^{opp}(h) = |\gamma_{10}(h) - \gamma_{11}(h)|$. Note that $\gamma^p_{1a}(h)$ is an estimate of a conditional probability \textit{based on the corrupted data}. We now introduce the corresponding estimate that only uses the clean (but unknown) subset of the training set $S^p$:
\begin{align*}
\gamma^c_a(h) = \frac{C^1_a(h)}{C_a(h)} = \frac{\sum_{i=1}^n \mathbbm{1}\{h(x^p_i) = 1, a^p_i = a, y^p_i = 1, i\not\in\poisoned\}}{\sum_{i=1}^n \mathbbm{1}\{a^p_i = a, y^p_i = 1, i\not\in\poisoned\}}.
\end{align*}
Similarly to before, we first bound how far the corrupted estimates $\gamma^p_{1a}(h)$ of $\gamma_{1a}(h)$ are from the clean estimates $\gamma^c_{1a}(h)$, uniformly over the hypothesis space $\mathcal{H}$:

\begin{lemma}
\label{lemma:upper-bound-delta-opp}
If $n \geq \max \left\{\frac{8\log(4/\delta)}{(1-\alpha)P_0}, \frac{12\log(3/\delta)}{\alpha}\right\}$, we have
\begin{align}
\label{eqn:bound_on_delta_appendix_opp}
\mathbb{P}^{\mathcal{A}}\left(\sup_{h \in \mathcal{H}}\left(\left|\gamma^p_{10}(h) - \gamma^c_{10}(h)\right| + \left|\gamma^p_{11}(h) - \gamma^c_{11}(h)\right|\right) \geq \frac{2\alpha}{P_{10}/3 + \alpha}\right) < \delta.
\end{align}
\end{lemma}
\begin{proof}
Similarly to the proof of Lemma \ref{lemma:upper-bound-delta-par-app}, we first show that certain bounds on $B_{1a}$ and $C_{1a}$ hold with high probability. Then we show that the supremum in (\ref{eqn:bound_on_delta_appendix_opp}) is bounded whenever these bounds hold. 

\paragraph{Step 1} Note that $B_{10} + B_{11} \leq B_0 + B_1 \sim \Bin(n, \alpha)$, and so
\begin{align*}
\mathbb{P}^{\mathcal{A}}\left(B_{10} + B_{11} \geq \frac{3\alpha}{2}n\right) \leq \mathbb{P}^{\mathcal{A}}\left(B_{0} + B_{1} \geq \frac{3\alpha}{2}n\right) \leq  e^{-\alpha n/12} \leq \frac{\delta}{3}.
\end{align*}
Similarly, $C_{10} \sim \Bin(n, (1-\alpha)P_{1a})$ and $C_{11} \sim \Bin(n, (1-\alpha)P_{11})$ and so
\begin{align*}
\mathbb{P}^{\mathcal{A}}\left(C_{10} \leq \frac{1-\alpha}{2}P_{10}n\right) \leq e^{-(1-\alpha) P_{10} n/8} \leq \frac{\delta}{4}
\end{align*}
and
\begin{align*}
\mathbb{P}^{\mathcal{A}}\left(C_{11} \leq \frac{1-\alpha}{2}P_{11}n\right) \leq e^{-(1-\alpha) P_{11}n/8} \leq \frac{\delta}{4}.
\end{align*}
Now since $n \geq \max \left\{\frac{8\log(4/\delta)}{(1-\alpha)P_{10}}, \frac{12\log(3/\delta)}{\alpha}\right\}$ and $P_{10} \leq P_{11}$
\begin{align}
\label{eqn:bound-on-rare-events}
\mathbb{P}^{\mathcal{A}}\left(\left(B_{10} + B_{11} \geq \frac{3\alpha}{2}n\right) \lor \left(C_{10} \leq \frac{1-\alpha}{2}P_{10}n\right) \lor \left(C_{11} \leq \frac{1-\alpha}{2}P_{11}n\right)\right) & \leq \frac{\delta}{3} +  \frac{\delta}{4} +  \frac{\delta}{4} \notag \\ & < \delta,
\end{align}

\paragraph{Step 2} Now assume that all of $B_{10} + B_{11} < \frac{3\alpha}{2}n, C_{10} > \frac{1-\alpha}{2}P_{10}n, C_{11} > \frac{1-\alpha}{2}P_{11}n$ hold.

Consider an arbitrary, fixed $h\in\mathcal{H}$. Since $h$ is fixed, we drop the dependence on $h$ from the notation for the rest of the proof and write $\gamma^p_{1a} = \gamma^p_{1a}(h), C^1_{1a} = C^1_{1a}(h)$ \etc .

We now prove that for both $a\in \{0,1\}$
\begin{equation}
\label{eqn:lemma_upper_bound_equal_opp_what_to_prove}
\Delta_{1a} \coloneqq |\gamma^p_a - \gamma^c_a| \leq \frac{B_{1a}}{C_{1a} + B_{1a}}.
\end{equation}
For each $a\in\{0,1\}$, this can be shown as follows. First, if $\sum_{i=1}^n \mathbbm{1}\{a^p_i = a, y^p_i = 1\} = B_{1a} + C_{1a} = 0$, then both $\gamma^p_{1a}(h)$ and $\gamma^c_{1a}(h)$ are equal to $0$, because of the convention that $\frac{0}{0} = 0$. In addition, $B_{1a} = C_{1a} = 0$. Therefore, inequality (\ref{eqn:lemma_upper_bound_equal_opp_what_to_prove}) trivially holds. 

Similarly, if $B_{1a} = 0$ and $C_{1a} > 0$, then $\gamma^p_{1a}(h) = \gamma^c_{1a}(h)$ and so $\Delta_{1a} = 0$ and (\ref{eqn:lemma_upper_bound_equal_opp_what_to_prove}) holds.

Assume now that $B_{1a} > 0$. Note that if $C_{1a} = \sum_{i=1}^n \mathbbm{1}\{a^p_i = a, y^p_i =1, i\not\in\poisoned\} = 0$, then $$\Delta_{1a} = \left|\gamma^p_{1a}(h) - \gamma^c_{1a}(h)\right| =  \left|\frac{B^1_{1a}}{B_{1a}} - 0\right| = \frac{B^1_{1a}}{B_{1a}} = \frac{B^1_{1a}}{B_{1a} + C_{1a}} \leq \frac{B_{1a}}{C_{1a} + B_{1a}}.$$
Finally, assume that both $C_{1a} > 0$ and $B_{1a} > 0$. Note that under any realization of the randomness of the data sampling and the adversary, for any $a\in\{0,1\}$
\begin{align*}
\gamma^p_{1a} & = \frac{\sum_{i=1}^n \mathbbm{1}\{h(x^p_i) = 1, a^p_i = a, y^p_i =1, i\not\in\poisoned\} + \sum_{i=1}^n \mathbbm{1}\{h(x^p_i) = 1, a^p_i = a, y^p_i =1, i\in\poisoned\}}{\sum_{i=1}^n \mathbbm{1}\{a^p_i = a, y^p_i =1, i\not\in\poisoned\} + \sum_{i=1}^n \mathbbm{1}\{a^p_i = a, y^p_i =1, i\in\poisoned\}} \\ & = \frac{C^1_{1a} + B^1_{1a}}{C_{1a} + B_{1a}}.
\end{align*} 
Next we bound how far this quantity is from the clean estimator $\frac{C^1_{1a}}{C_{1a}}$
\begin{align*}
\Delta_{1a} \vcentcolon = \left|\gamma^p_{1a} - \gamma^c_{1a}\right| = \left|\frac{C^1_{1a} + B^1_{1a}}{C_{1a} + B_{1a}} - \gamma^c_{1a}\right| = \frac{B_{1a}}{C_{1a} + B_{1a}}\left|\gamma^c_{1a} - \frac{B^1_{1a}}{B_{1a}}\right| \leq \frac{B_{1a}}{C_{1a} + B_{1a}}.
\end{align*}

Now since $B_{10} + B_{11} < \frac{3\alpha}{2}n, C_{10} > \frac{1-\alpha}{2}P_{10}n, C_{11} > \frac{1-\alpha}{2}P_{11}n$ hold, we get
\begin{align*}
\Delta_{10} + \Delta_{11} & \leq \frac{B_{10}}{C_{10} + B_{10}} + \frac{B_{11}}{C_{11} + B_{11}} \\ & < \frac{B_{10}}{\frac{1-\alpha}{2}P_{10} n + B_{10}} + \frac{B_{11}}{\frac{1-\alpha}{2}P_{11} n + B_{11}} \\
& \leq \frac{B_{10}}{\frac{1-\alpha}{2}P_{10} n + B_{10}} + \frac{B_{11}}{\frac{1-\alpha}{2}P_{10} n + B_{11}} \\ 
& = \frac{B_{10}}{\frac{1-\alpha}{2}P_{10} n + B_{10}}  + 1 - \frac{\frac{1-\alpha}{2}P_{10} n}{\frac{1-\alpha}{2}P_{10} n - B_{10} + (B_{10} + B_{11})} \\
& < \frac{B_{10}}{\frac{1-\alpha}{2}P_{10} n + B_{10}}  + 1 - \frac{\frac{1-\alpha}{2}P_{10} n}{\frac{1-\alpha}{2}P_{10} n - B_{10} + \frac{3\alpha}{2}n} \\
& = 2 - (1-\alpha)P_{10}n\left(\frac{1}{(1-\alpha)P_{10} n + 2B_{10}} + \frac{1}{(1-\alpha)P_{10} n + 3\alpha n - 2B_{10}}\right) 
\end{align*}
The same argument as in Lemma \ref{lemma:upper-bound-delta-par-app} shows that this is maximized at $B_{10} = \frac{3\alpha}{4}n$ and so 
\begin{align}
\label{eqn:bound-on-sum-delta-equal-opp}
\Delta_{10} + \Delta_{11} & \leq \frac{B_{10}}{C_{10} + B_{10}} + \frac{B_{11}}{C_{11} + B_{11}} \\ & < 2 - (1-\alpha)P_{10}n\left(\frac{1}{(1-\alpha)P_{10} n + \frac{3\alpha}{2}n} + \frac{1}{(1-\alpha)P_{10} n + 3\alpha n - \frac{3\alpha}{2}n}\right) \notag \\
& \leq \frac{2\alpha}{P_{10}/3 + \alpha}. \notag
\end{align}
Since this holds for any arbitrary hypothesis $h\in\mathcal{H}$, the result follows.
\end{proof}

Denote the irreducible error term for equal opportunity by $\Delta^{opp} = \frac{2\alpha}{P_{10}/3 + \alpha}$. We then have the following bound for a fixed $h\in\mathcal{H}$:
\begin{lemma}
\label{lemma:non-uniform-bound-equal-opp}
Let $h\in\mathcal{H}$ be a fixed hypothesis and $D\in \mathcal{P}(\prodspace)$ be a fixed distribution. Denote $P_{1a} = \mathbb{P}(A = a, Y = 1)$ for $a\in\{0,1\}$. Let $\mathcal{A}$ be any malicious adversary and denote by $\mathbb{P}^{\mathcal{A}}$ the probability distribution of the poisoned data $S^p$, under the random sampling of the clean data, the marked points and the randomness of the adversary. Then for any $n \geq \max \left\{\frac{8\log(8/\delta)}{(1-\alpha)P_{10}}, \frac{12\log(6/\delta)}{\alpha}\right\}$ and $\delta\in(0,1)$
\begin{equation}
\label{eqn:non-uniform-bound-equal-opp}
\mathbb{P}^{\mathcal{A}}\left(\left|\widehat{\mathcal{D}}^{opp}(h) - \mathcal{D}^{opp}(h)\right| \leq \Delta^{opp} + 2\sqrt{\frac{\log(16/\delta)}{n(1-\alpha)P_{10}}}\right) \geq 1-\delta
\end{equation}
\end{lemma}

\begin{proof}
The proof is exactly the same as the one of Lemma \ref{lemma:non-uniform-bound-demog-parity} , but with conditioning on $S^c_{1a} = \{i: a^p_i = a, y^p_i = 1, i\not\in\poisoned\}$ (the set of indexes of the poisoned data for which the protected group is $a$, the label is $1$ and the corresponding point was not marked for the adversary) instead.
\end{proof}

The same argument as in Lemma \ref{lemma:uniform-bound-demog-par} gives a uniform bound over the whole hypothesis space, provided that $\mathcal{H}$ has a finite VC-dimension $d \vcentcolon = VC(\mathcal{H})$:

\begin{lemma}
\label{lemma:uniform-bound-equal-opp}
Under the setup of Lemma \ref{lemma:non-uniform-bound-equal-opp}, assume additionally that $\mathcal{H}$ has a finite VC-dimension $d$. Then for any $n \geq \max \left\{\frac{8\log(8/\delta)}{(1-\alpha) P_{10}}, \frac{12\log(6/\delta)}{\alpha}, \frac{d}{2}\right\}$ and $\delta\in(0,1)$
\begin{align}
\label{eqn:uniform-bound-equal-opp-parity}
\mathbb{P}^{\mathcal{A}}_{S^p} & \left(\sup_{h\in\mathcal{H}}|\widehat{\mathcal{D}}^{opp}(h) - \mathcal{D}^{opp}(h)| \leq \Delta^{opp} +  16 \sqrt{\frac{2d\log(\frac{2en}{d}) + 2\log(48/\delta)}{(1-\alpha)P_{10}n}}\right) \geq 1 - \delta.
\end{align}
\end{lemma}

\noindent Finally, we prove multiplicative bounds and claims in the case when $\mathbb{P}(h(X) = 1| A = 0, Y = 1) = \mathbb{P}(h(X) = 1| A = 1, Y = 1) = 1$ (which holds for example when $h(X) = Y$ almost surely). These will come in useful for proving the component-wise upper bound with fast rates.

We will be interested in the estimate
\begin{align*}
\bar{\gamma}^p_{1a}(h) = \frac{\sum_{i=1}^n \mathbbm{1}\{h(x^p_i) = 0, a^p_i = a, y^p_1 = 1\}}{\sum_{i=1}^n \mathbbm{1}\{a^p_i = a, y^p_i = 1\}}
\end{align*}
of $\bar{\gamma}_{1a}(h) = \mathbb{P}(h(X) = 0 | A = a, Y = 1)$. Again, we also introduce the corresponding clean data estimate $C^0_{1a}(h) \coloneqq \sum_{i=1}^n \mathbbm{1}\{i: h(x^p_i) = 0, a^p_i = a, y^p_1 = 1, i\not\in\poisoned\}$ and
\begin{align*}
\bar{\gamma}^c_{1a}(h) = \frac{C^0_{1a}(h)}{C_{1a}} = \frac{\sum_{i=1}^n \mathbbm{1}\{i: h(x^p_i) = 0, a^p_i = a, y^p_1 = 1, i\not\in\poisoned\}}{\sum_{i=1}^n \mathbbm{1}\{i: a^p_i = a, y^p_1 = 1, i\not\in\poisoned\}}.
\end{align*}
Denote also
\begin{align*}
\bar{\Delta}_{1a}(h) \vcentcolon =\left|\bar{\gamma}^p_{1a}(h) - \gamma^c_{1a}(h)\right|,
\end{align*}
We only show non-uniform bounds for a fixed $h\in\mathcal{H}$ here, so we omit the dependence of these quantities on $h$. We have:
\begin{lemma}
\label{lemma:realizable-non-uniform}
Let $\mathbb{P}\in \mathcal{P}(\prodspace)$ be a fixed distribution and let $h\in\mathcal{H}$ be a fixed classifier. Denote $P_{1a} = \mathbb{P}(A = a, Y = 1)$ for $a\in\{0,1\}$. Let $\mathcal{A}$ be any malicious adversary and denote by $\mathbb{P}^{\mathcal{A}}$ the probability distribution of the poisoned data $S^p$, under the random sampling of the clean data, the marked points and the randomness of the adversary. Then:\\
(a) For any $n > 0$ and any $\eta, \delta\in(0,1)$
\begin{align}
\mathbb{P}^{\mathcal{A}}\left(\bar{\gamma}^p_{1a} \geq (1+\eta)\bar{\gamma}_{1a} + \bar{\Delta}_{1a}\right) \leq \exp \left(-\frac{(1-\alpha)P_{1a} n}{8}\right) + \exp\left(-\frac{1}{6}\eta^2 (1-\alpha)P_{1a}\bar{\gamma}_{1a}n\right).
\end{align}
and 
\begin{align}
\mathbb{P}^{\mathcal{A}}\left(\bar{\gamma}^p_{1a} \leq (1-\eta)\bar{\gamma}_{1a} - \bar{\Delta}_{1a}\right) \leq \exp \left(-\frac{(1-\alpha)P_{1a} n}{8}\right) + \exp\left(-\frac{1}{4}\eta^2 (1-\alpha)P_{1a}\bar{\gamma}_{1a}n\right).
\end{align}
(b) Assume further that $\mathbb{P}(h(X) = 0 | A = 0, Y = 1) = \mathbb{P}(h(X) = 0 | A = 1, Y = 1) = 0$. Then for any $\delta \in (0,1)$ and $n \geq \max \left\{\frac{8\log(4/\delta)}{(1-\alpha)P_{10}},  \frac{12\log(3/\delta)}{\alpha}\right\}$
\begin{align}
\mathbb{P}^{\mathcal{A}}\left(\bar{\gamma}^p_{10} + \bar{\gamma}^p_{11} \geq \Delta^{opp}\right) \leq \delta
\end{align}
\end{lemma}

\begin{proof}
Let $S^c_{1a} = \{i: a^p_i = a, y^p_i = 1, i\not\in\poisoned\}$. For any $a\in\{0,1\}$ we have
\begin{align*}
\mathbb{P}^{\mathcal{A}} & \left(\bar{\gamma}^p_{1a} \geq (1+\eta)\bar{\gamma}_{1a} + \bar{\Delta}_{1a}\right) \\ 
& = \sum_{S^c_{1a}} \mathbb{P}^{\mathcal{A}}\left(\bar{\gamma}^p_{1a} \geq (1+\eta)\bar{\gamma}_{1a} + \bar{\Delta}_{1a}\middle| S^c_{a}\right)\mathbb{P}(S^c_{a}) \\
& \leq \mathbb{P}^{\mathcal{A}}\left(C_{1a} \leq \frac{(1-\alpha)}{2}P_{1a}n\right) \\ 
& + \sum_{S^c_{1a}: C_{1a} \geq \frac{(1-\alpha)}{2}P_{1a}n}\mathbb{P}^{\mathcal{A}}\left(\bar{\gamma}^p_{1a} \geq (1+\eta)\bar{\gamma}_{1a} + \bar{\Delta}_{1a}\middle| S^c_{1a}\right)\mathbb{P}^{\mathcal{A}}(S^c_{1a}) \\ 
& \leq \mathbb{P}^{\mathcal{A}}\left(C_{1a} \leq \frac{(1-\alpha)}{2}P_{1a}n\right) \\
& + \sum_{S^c_{1a}: C_{1a} \geq \frac{(1-\alpha)}{2}P_{1a}n}\mathbb{P}^{\mathcal{A}}\left(\bar{\gamma}^p_{1a} - \frac{C^1_{1a}}{C_{1a}} + \frac{C^1_{1a}}{C_{1a}} \geq (1+\eta)\bar{\gamma}_{1a} + \bar{\Delta}_{1a}\middle| S^c_{1a}\right)\mathbb{P}^{\mathcal{A}}(S^c_{1a}) \\
& \leq \mathbb{P}^{\mathcal{A}}\left(C_{1a} \leq \frac{(1-\alpha)}{2}P_{1a}n\right) \\ & + \sum_{S^c_{1a}: C_{1a} \geq \frac{(1-\alpha)}{2}P_{1a}n}\mathbb{P}^{\mathcal{A}}\left(\frac{C^1_{1a}}{C_{1a}} \geq (1+\eta)\bar{\gamma}_{1a}\middle| S^c_{1a}\right)\mathbb{P}^{\mathcal{A}}(S^c_{1a}) \\
& \leq \exp \left(-\frac{(1-\alpha)P_{1a} n}{8}\right) + \sum_{S^p_{a}: C_{1a} \geq \frac{(1-\alpha)}{2}P_{1a}n} \exp\left(-\frac{\eta^2C_{1a}\bar{\gamma}_{1a}}{3}\right) \mathbb{P}^{\mathcal{A}}(S^c_{1a}) \\ 
& \leq \exp \left(-\frac{(1-\alpha)P_{1a} n}{8}\right) + \exp\left(-\frac{1}{6}\eta^2 (1-\alpha)P_{1a}\bar{\gamma}_{1a}n\right).
\end{align*}
A similar argument, with the other direction of the Chernoff bounds, gives the other bound.

(b) Similarly to the argument in the proof of Lemma \ref{lemma:upper-bound-delta-opp}
\begin{align}
\label{eqn:yet-another-bound}
\bar{\Delta}_{1a} = \left|\bar{\gamma}^p_{1a} - \frac{C^0_{1a}}{C_{1a}}\right| \leq \frac{B_{1a}}{C_{1a} + B_{1a}}.
\end{align}
Using the inequalities (\ref{eqn:bound-on-rare-events}) and (\ref{eqn:bound-on-sum-delta-equal-opp}),
\begin{align}
\label{eqn:yet-another-bound2}
\mathbb{P}^{\mathcal{A}}\left(\bar{\Delta}_{10} + \bar{\Delta}_{11} \geq \frac{2\alpha}{P_{10}/3 + \alpha}\right) \leq \mathbb{P}^{\mathcal{A}}\left(\frac{B_{10}}{C_{10} + B_{10}} + \frac{B_{11}}{C_{11} + B_{11}} \geq \frac{2\alpha}{P_{10}/3 + \alpha}\right) < \delta.
\end{align}
Since also $$\mathbb{P}^{\mathcal{A}}\left(\frac{C^0_{1a}}{C_{1a}} > 0\right) = \sum_{S^c_{1a}} \mathbb{P}^{\mathcal{A}}\left(\frac{C^0_{1a}}{C_{1a}} > 0 \middle| S^c_{1a}\right)   \mathbb{P}^{\mathcal{A}}\left(S^c_{1a}\right) = \sum_{S^c_{1a}} \mathbb{P}\left(\Bin(|S^c_{1a}|, 0) > 0\right) \mathbb{P}^{\mathcal{A}}(S^c_{1a}) = 0,$$
we have that $0\leq \bar{\gamma}^p_{1a} = \bar{\Delta}_{1a}$ almost surely, for both $a\in\{0,1\}$. Therefore, $0< \bar{\gamma}^p_{10} + \bar{\gamma}^p_{11} = \bar{\Delta}_{10} + \bar{\Delta}_{11}$ and the result follows.
\end{proof}

\addtocounter{theorem}{-3}

\subsection{Upper bound theorems - proofs}
\label{sec:upper-bounds-themselves-proofs}

We are now ready to present the proofs of the upper bound results from the main body of the paper.

\subsubsection{Upper bounds on the $\lambda$-weighted objective}

First we prove the bounds for the $\lambda$-weighted objective.

\paragraph{Bound for demographic parity}
Let $\lambda \geq 0$ be fixed. Recall our notation for the $\lambda$-weighted objective:
\begin{align*}
L_{\lambda}^{par}(h) = \mathcal{R}(h) + \lambda \mathcal{D}^{par}(h).
\end{align*}
Suppose that a learner $\mathcal{L}^{par}_{\lambda}:\cup_{n = 1}^{\infty} (\prodspace)^n \to \mathcal{H}$ is such that $$\mathcal{L}^{par}(S^p) \in \argmin_{h\in\mathcal{H}}(\widehat{R}^p(h) + \lambda \widehat{\mathcal{D}}^{par}(h)) \quad \text{for all } S^p.$$ That is, $\mathcal{L}^{par}_{\lambda}$ always returns a minimizer of the $\lambda$-weighted empirical objective. Then we have the following:
\begin{theorem}
\label{thm:agnostic-upper-bound-demog-par-app}
Let $\mathcal{H}$ be any hypothesis space with $d = VC(\mathcal{H}) < \infty$. Let $\mathbb{P}\in \mathcal{P}(\prodspace)$ be a fixed distribution and $\mathcal{A}$ be any malicious adversary of power $\alpha < 0.5$. Denote by $\mathbb{P}^{\mathcal{A}}$ the probability distribution of the poisoned data $S^p$, under the random sampling of the clean data, the marked points and the randomness of the adversary. Then for any $\delta\in(0,1)$ and $n \geq \max \left\{\frac{8\log(16/\delta)}{(1-\alpha)P_0}, \frac{12\log(12/\delta)}{\alpha}, \frac{d}{2}\right\}$, we have
\begin{align*}
\mathbb{P}^{\mathcal{A}}\left(L_{\lambda}^{par}(\mathcal{L}^{par}_{\lambda}(S^p))  \leq \min_{h\in\mathcal{H}} L_{\lambda}^{par}(h) + \Delta^{par}_{\lambda}\right) > 1-  \delta,
\end{align*}
where\footnote{the $\widetilde{\mathcal{O}}$-notation hides constant and logarithmic factors}
\begin{align*}
\Delta^{par}_{\lambda} = 3\alpha + 2\lambda\Delta^{par} + \widetilde{\mathcal{O}}\left(\sqrt{\frac{d}{n}} + \lambda \sqrt{\frac{d}{P_0n}}\right)
\end{align*}
and 
\begin{align*}
\Delta^{par} = \frac{2\alpha}{P_0/3 + \alpha} = \mathcal{O}\left(\frac{\alpha}{P_0}\right).
\end{align*}
\end{theorem}
\begin{proof}
By the standard concentrations results for the $0/1$ loss (see, for example, Chapter 28.1 in \citep{shalev2014understanding})
\begin{equation*}
\mathbb{P}\left(\sup_{h\in\mathcal{H}}|\widehat{\mathcal{R}}^c(h) - \mathcal{R}(h)| > 2\sqrt{\frac{8d\log(\frac{en}{d}) + 2\log(16/\delta)}{n}}\right) \leq \frac{\delta}{4},
\end{equation*}
where $\widehat{\mathcal{R}}^c(h) = \frac{1}{n}\sum_{i=1}^n \mathbbm{1}(h(x^c_i) \neq y^c_i)$ is the loss of $h$ on the clean data. Since the total number of poisoned points $|\poisoned| \sim \Bin(n, \alpha)$ and since $n > \frac{12\log(4/\delta)}{\alpha}$
\begin{equation*}
\mathbb{P}^{\mathcal{A}}\left(\sup_{h\in\mathcal{H}}|\widehat{\mathcal{R}}^c(h) - \widehat{\mathcal{R}}^p(h)| > \frac{3\alpha}{2}\right) \leq \mathbb{P}^{\mathcal{A}}\left(|\poisoned| \geq \frac{3\alpha}{2}n\right) \leq e^{-\alpha n/12} \leq \frac{\delta}{4}.
\end{equation*}
Since $\sup_{h\in\mathcal{H}}|\widehat{\mathcal{R}}^p(h) - \mathcal{R}(h)| \leq \sup_{h\in\mathcal{H}}|\widehat{\mathcal{R}}^p(h) - \widehat{\mathcal{R}}^c(h)| + \sup_{h\in\mathcal{H}}|\widehat{\mathcal{R}}^c(h) - \mathcal{R}(h)|$, we obtain
\begin{equation}
\label{eqn:concentration-loss-finite-h-poisoned-data}
\mathbb{P}^{\mathcal{A}}\left(\sup_{h\in\mathcal{H}}|\widehat{\mathcal{R}}^p(h) - \mathcal{R}(h)| > \frac{3\alpha}{2} + 2\sqrt{\frac{8d\log(\frac{en}{d}) + 2\log(16/\delta)}{n}}\right) \leq \frac{\delta}{2}.
\end{equation}
In addition, Lemma \ref{lemma:non-uniform-bound-demog-parity} implies that
\begin{equation}
\label{eqn:uniform-bound-demog-parity-finite-h}
\mathbb{P}^{\mathcal{A}}\left(\sup_{h\in\mathcal{H}}\left|\widehat{\mathcal{D}}^{par}(h) - \mathcal{D}^{par}(h)\right| > \Delta^{par} + 16 \sqrt{\frac{2d\log(\frac{2en}{d}) + 2\log(96/\delta)}{(1-\alpha)P_0n}}\right) \leq \frac{\delta}{2}.
\end{equation}
Now let $h_{\lambda} = \argmin_{h\in\mathcal{H}}(\widehat{R}^p(h) + \lambda \widehat{\mathcal{D}}^{par}(h))$ and let
\begin{align*}
\Delta^{par}_{\lambda} & = 3\alpha + 4\sqrt{\frac{8d\log(\frac{en}{d}) + 2\log(16/\delta)}{n}} + 2\lambda \Delta^{par} + 32\lambda \sqrt{\frac{2d\log(\frac{2en}{d}) + 2\log(96/\delta)}{(1-\alpha)P_0n}} \\ & = \widetilde{\mathcal{O}} \left(\sqrt{\frac{d}{n}} + \lambda \sqrt{\frac{d}{P_0n}}\right).
\end{align*} 
Then, using (\ref{eqn:concentration-loss-finite-h-poisoned-data}) and (\ref{eqn:uniform-bound-demog-parity-finite-h}), we have that with probability at least $1 - \delta$

\begin{align*}
L_{\lambda}^{par}(\mathcal{L}^{par}_{\lambda}(S^p)) & = \mathcal{R}(\mathcal{L}^{par}_{\lambda}(S^p)) + \lambda \mathcal{D}^{par}(\mathcal{L}^{par}_{\lambda}(S^p)) \\ & \leq  \widehat{\mathcal{R}}^{p}(\mathcal{L}^{par}_{\lambda}(S^p)) + \lambda \widehat{\mathcal{D}}^{par}(\mathcal{L}^{par}_{\lambda}(S^p)) + \frac{1}{2}\Delta^{par}_{\lambda} \\ & = \min_{h\in \mathcal{H}} \left(\widehat{\mathcal{R}}^{p}(h) + \lambda \widehat{\mathcal{D}}^{par}(h)\right) + \frac{1}{2}\Delta^{par}_{\lambda} \\
& \leq \min_{h\in\mathcal{H}}L_{\lambda}^{par}(h) + \Delta^{par}_{\lambda}.
\end{align*}
\end{proof}

\paragraph{Bound for equal opportunity} We now show a similar result for the weighted-objective with the equal opportunity deviation measure
\begin{align*}
L_{\lambda}^{opp}(h) = \mathcal{R}(h) + \lambda \mathcal{D}^{opp}(h).
\end{align*}
Let $\mathcal{L}^{opp}_{\lambda}:\cup_{n = 1}^{\infty} (\prodspace)^n \to \mathcal{H}$ be such that $$\mathcal{L}^{opp}_{\lambda}(S^p) \in \argmin_{h\in\mathcal{H}}(\widehat{R}^p(h) + \lambda \widehat{\mathcal{D}}^{opp}(h)), \quad \text{for all } S^p.$$ That is, $\mathcal{L}^{opp}_{\lambda}$ always returns a minimizer of the $\lambda$-weighted empirical objective. Then:

\begin{theorem}
\label{thm:agnostic-upper-bound-equal-opp-app}
Let $\mathcal{H}$ be any hypothesis space with $d = VC(\mathcal{H}) < \infty$. Let $\mathbb{P}\in \mathcal{P}(\prodspace)$ be a fixed distribution and $\mathcal{A}$ be any malicious adversary of power $\alpha < 0.5$. Denote by $\mathbb{P}^{\mathcal{A}}$ the probability distribution of the poisoned data $S^p$, under the random sampling of the clean data, the marked points and the randomness of the adversary. Then for any $\delta\in(0,1)$ and $n \geq \max \left\{\frac{8\log(16/\delta)}{(1-\alpha) P_{10}}, \frac{12\log(12/\delta)}{\alpha}, \frac{d}{2}\right\}$, we have
\begin{align*}
\mathbb{P}^{\mathcal{A}}\left(L_{\lambda}^{opp}(\mathcal{L}^{opp}_{\lambda}(S^p))  \leq \min_{h\in\mathcal{H}} L_{\lambda}^{opp}(h) + \Delta^{opp}_{\lambda}\right) \leq \delta,
\end{align*}
where
\begin{align*}
\Delta^{opp}_{\lambda} = 3\alpha + 2\lambda \Delta^{opp} + \widetilde{\mathcal{O}}\left(\sqrt{\frac{d}{n}} + \lambda \sqrt{\frac{d}{P_{10}n}}\right)
\end{align*}
and 
\begin{align*}
\Delta^{opp} = \frac{2\alpha}{P_{10}/3 + \alpha} = \mathcal{O}\left(\frac{\alpha}{P_{10}}\right).
\end{align*}
\end{theorem}

\begin{proof}
Similarly to the proof of Theorem \ref{thm:agnostic-upper-bound-demog-par}, we combine
\begin{equation}
\mathbb{P}^{\mathcal{A}}\left(\sup_{h\in\mathcal{H}}|\widehat{\mathcal{R}}^p(h) - \mathcal{R}(h)| > \frac{3\alpha}{2} + 2\sqrt{\frac{8d\log(\frac{en}{d}) + 2\log(16/\delta)}{n}}\right) \leq \frac{\delta}{2}.
\end{equation}
and Lemma \ref{lemma:non-uniform-bound-equal-opp}
\begin{equation}
\label{eqn:uniform-bound-demog-parity-finite-h}
\mathbb{P}^{\mathcal{A}}_{S^p} \left(\sup_{h\in\mathcal{H}}|\widehat{\mathcal{D}}^{opp}(h) - \mathcal{D}^{opp}(h)| > \Delta^{opp} +  16 \sqrt{\frac{2d\log(\frac{2en}{d}) + 2\log(96/\delta)}{(1-\alpha)P_{10}n}}\right) \leq \frac{\delta}{2}
\end{equation}
Now let $h_{\lambda} = \argmin_{h\in\mathcal{H}}(\widehat{R}^p(h) + \lambda \widehat{\mathcal{D}}^{opp}(h))$ and let
\begin{align*}
\Delta^{opp}_{\lambda} = 3\alpha + 4\sqrt{\frac{8d\log(\frac{en}{d}) + 2\log(16/\delta)}{n}} + 2\lambda \Delta^{opp} + 32\lambda \sqrt{\frac{2d\log(\frac{2en}{d}) + 2\log(96/\delta)}{(1-\alpha)P_{10}n}}.
\end{align*} 
Then we have that with probability at least $1 - \delta$
\begin{align*}
L_{\lambda}^{opp}(\mathcal{L}^{opp}_{\lambda}(S^p)) & = \mathcal{R}(\mathcal{L}^{opp}_{\lambda}(S^p)) + \lambda \mathcal{D}^{opp}(\mathcal{L}^{opp}_{\lambda}(S^p)) \\ & \leq  \widehat{\mathcal{R}}^{p}(\mathcal{L}^{opp}_{\lambda}(S^p)) + \lambda \widehat{\mathcal{D}}^{opp}(\mathcal{L}^{par}_{\lambda}(S^p)) + \frac{1}{2}\Delta^{opp}_{\lambda} \\ & = \min_{h\in\mathcal{H}}\left(\widehat{\mathcal{R}}^{p}(h) + \lambda \widehat{\mathcal{D}}^{opp}(h)\right) + \frac{1}{2}\Delta^{opp}_{\lambda} \\
& \leq \min_{h\in\mathcal{H}}L_{\lambda}^{opp}(h) + \Delta^{opp}_{\lambda}.
\end{align*}
\end{proof}

\subsubsection{Component-wise upper bounds}

We now prove the component-wise upper bound results.
\paragraph{Bound for demographic parity}
Recall our notation $\widehat{h}^{r} \in \argmin_{h\in\mathcal{H}} \widehat{\mathcal{R}}^p(h)$ and $\widehat{h}^{par} \in \argmin_{h\in\mathcal{H}} \widehat{\mathcal{D}}^{par}(h)$. Further, we define the sets
\begin{align*}
\mathcal{H}_1 & = \left\{h\in\mathcal{H}: \widehat{\mathcal{R}}^p(h) - \widehat{\mathcal{R}}^p(\widehat{h}^{r}) \leq 3\alpha + 4\sqrt{\frac{8d\log(\frac{en}{d}) + 2\log(16/\delta)}{n}} \right\} \\
\mathcal{H}_2 & = \left\{h\in\mathcal{H}: \widehat{\mathcal{D}}^{par}(h) - \widehat{\mathcal{D}}^{par}(\widehat{h}^{par}) \leq 2\Delta^{par} + 32 \sqrt{\frac{2d\log(\frac{2en}{d}) + 2\log(96/\delta)}{(1-\alpha)P_0n}} \right\}.
\end{align*}
That is, $\mathcal{H}_1$ and $\mathcal{H}_2$ are the sets of classifiers that are not far from optimal on the train data, in terms of their risk and their fairness respectively. Define the \textit{component-wise learner}:
\begin{align*}
\mathcal{L}^{par}_{cw}(S^p) = 
\begin{cases}
    \text{any }h \in \mathcal{H}_1 \cap \mathcal{H}_2, & \text{if } \mathcal{H}_1 \cap \mathcal{H}_2 \neq \emptyset\\
    \text{any }h \in \mathcal{H}, & \text{otherwise,}
\end{cases}
\end{align*}
that returns a classifier that is good in both metrics, if such exists, or an arbitrary classifier otherwise. Then we have the following:

\begin{theorem}
\label{thm:comp-wise-upper-bound-demog-par-app}
Let $\mathcal{H}$ be any hypothesis space with $d = VC(\mathcal{H}) < \infty$. Let $\mathbb{P}\in \mathcal{P}(\prodspace)$ be a fixed distribution and let $\mathcal{A}$ be any malicious adversary of power $\alpha < 0.5$. Suppose that there exists a hypothesis $h^*\in\mathcal{H}$, such that $\ObjVec(h^*) \preceq \ObjVec(h)$ for all $h\in\mathcal{H}$. Then for any $\delta\in(0,1)$ and $n \geq \max \left\{\frac{8\log(16/\delta)}{(1-\alpha)P_0}, \frac{12\log(12/\delta)}{\alpha}, \frac{d}{2}\right\}$, with probability at least $1-\delta$:
\begin{align*}
\mathbf{L}^{par}(\mathcal{L}^{par}_{cw}(S^p))  \preceq \left(6\alpha + \widetilde{\mathcal{O}}\left(\sqrt{\frac{d}{n}}\right), 4\Delta^{par} +  \widetilde{\mathcal{O}}\left(\sqrt{\frac{d}{P_0 n}}\right) \right).
\end{align*}
\end{theorem}
\begin{proof} From the proof of Theorem \ref{thm:agnostic-upper-bound-demog-par-app}, we have that with probability at least $1-\delta$, both of the following hold:
\begin{equation*}
\sup_{h\in\mathcal{H}}|\widehat{\mathcal{R}}^p(h) - \mathcal{R}(h)| \leq \frac{3\alpha}{2} + 2\sqrt{\frac{8d\log(\frac{en}{d}) + 2\log(16/\delta)}{n}},
\end{equation*}
\begin{equation*}
\sup_{h\in\mathcal{H}}\left|\widehat{\mathcal{D}}^{par}(h) - \mathcal{D}^{par}(h)\right| \leq \Delta^{par} + 16 \sqrt{\frac{2d\log(\frac{2en}{d}) + 2\log(96/\delta)}{(1-\alpha)P_0n}}.
\end{equation*}
We show that under this event, $\mathcal{H}_1 \cap \mathcal{H}_2 \neq \emptyset$ and for any $h\in \mathcal{H}_1 \cap \mathcal{H}_2$, 
\begin{equation*}
\mathbf{L}^{par}(h) \preceq \left(6\alpha + 8\sqrt{\frac{8d\log(\frac{en}{d}) + 2\log(16/\delta)}{n}} , 4\Delta^{par} + 64\sqrt{\frac{2d\log(\frac{2en}{d}) + 2\log(96/\delta)}{(1-\alpha)P_0n}}\right),
\end{equation*}
from which the result follows. Note that
\begin{align*}
\widehat{\mathcal{R}}^p(h^*) & \leq \mathcal{R}(h^*) + \frac{3\alpha}{2} + 2\sqrt{\frac{8d\log(\frac{en}{d}) + 2\log(16/\delta)}{n}} \\
& \leq \mathcal{R}(\widehat{h}^r) + \frac{3\alpha}{2} + 2\sqrt{\frac{8d\log(\frac{en}{d}) + 2\log(16/\delta)}{n}} \\
& \leq \widehat{\mathcal{R}}^p(\widehat{h}^r) + 3\alpha + 4\sqrt{\frac{8d\log(\frac{en}{d}) + 2\log(16/\delta)}{n}}
\end{align*}
and similarly
\begin{align*}
\widehat{\mathcal{D}}^{par}(h^*) \leq \widehat{\mathcal{D}}^{par}(\widehat{h}^r) + 2\Delta^{par} + 32 \sqrt{\frac{2d\log(\frac{2en}{d}) + 2\log(96/\delta)}{(1-\alpha)P_0n}}
\end{align*}
Therefore, $h^* \in \mathcal{H}_1 \cap \mathcal{H}_2$ and so $\mathcal{H}_1 \cap \mathcal{H}_2\neq \emptyset$.

Now take any $h \in \mathcal{H}_1 \cap \mathcal{H}_2$. We have that
\begin{align*}
\mathcal{R}(h) & \leq \widehat{\mathcal{R}}^p(h) +  \frac{3\alpha}{2} + 2\sqrt{\frac{8d\log(\frac{en}{d}) + 2\log(16/\delta)}{n}}\\
& \leq \widehat{\mathcal{R}}^p(\widehat{h}^r) +  3\frac{3\alpha}{2} + 6\sqrt{\frac{8d\log(\frac{en}{d}) + 2\log(16/\delta)}{n}} \\
& \leq \widehat{\mathcal{R}}^p(h^*) +  3\frac{3\alpha}{2} + 6\sqrt{\frac{8d\log(\frac{en}{d}) + 2\log(16/\delta)}{n}} \\
& \leq \mathcal{R}(h^*) +  6\alpha + 8\sqrt{\frac{8d\log(\frac{en}{d}) + 2\log(16/\delta)}{n}}.
\end{align*}
Similarly,
\begin{align*}
\mathcal{D}^{par}(h) \leq \mathcal{D}^{par}(h^*) + 4\Delta^{par} + 64\sqrt{\frac{2d\log(\frac{2en}{d}) + 2\log(96/\delta)}{(1-\alpha)P_0n}}
\end{align*}
and the result follows.
\end{proof}

\paragraph{Bound for equal opportunity}
Similarly, let $\widehat{h}^{opp} \in \argmin_{h\in\mathcal{H}} \widehat{\mathcal{D}}^{opp}(h)$. Further, we define the set
\begin{align*}
\mathcal{H}_3 & = \left\{h\in\mathcal{H}: \widehat{\mathcal{D}}^{opp}(h) - \widehat{\mathcal{D}}^{opp}(\widehat{h}^{opp}) \leq 2\Delta^{opp} + 32 \sqrt{\frac{2d\log(\frac{2en}{d}) + 2\log(96/\delta)}{(1-\alpha)P_{10}n}} \right\}.
\end{align*}
That is, $\mathcal{H}_3$ is the set of classifiers that are not far from optimal on the train data, in terms of equal opportunity fairness. Now define the \textit{component-wise learner} for equal opportunity:
\begin{align*}
\mathcal{L}^{opp}_{cw}(S^p) = 
\begin{cases}
    \text{any }h \in \mathcal{H}_1 \cap \mathcal{H}_3, & \text{if } \mathcal{H}_1 \cap \mathcal{H}_3 \neq \emptyset\\
    \text{any }h \in \mathcal{H}, & \text{otherwise,}
\end{cases}
\end{align*}
that returns a classifier that is good in both metrics, if such exists, or an arbitrary classifier otherwise. Then we have the following:

\begin{theorem}
\label{thm:comp-wise-upper-bound-equal-opp-app}
Let $\mathcal{H}$ be any hypothesis space with $d = VC(\mathcal{H}) < \infty$. Let $\mathbb{P}\in \mathcal{P}(\prodspace)$ be a fixed distribution and let $\mathcal{A}$ be any malicious adversary of power $\alpha < 0.5$. Suppose that there exists a hypothesis $h^*\in\mathcal{H}$, such that $\ObjVec(h^*) \preceq \ObjVec(h)$ for all $h\in\mathcal{H}$. Then for any $\delta\in(0,1)$ and $n \geq \max \left\{\frac{8\log(16/\delta)}{(1-\alpha) P_{10}}, \frac{12\log(12/\delta)}{\alpha}, \frac{d}{2}\right\}$, with probability at least $1-\delta$
\begin{align*}
\mathbf{L}^{opp}(\mathcal{L}^{opp}_{cw}(S^p))  \preceq \left(6\alpha + \widetilde{\mathcal{O}}\left(\sqrt{\frac{d}{n}}\right), 4\Delta^{opp} +  \widetilde{\mathcal{O}}\left(\sqrt{\frac{d}{P_{10} n}}\right) \right).
\end{align*}
\end{theorem}
\begin{proof}
From the proof of Theorem \ref{thm:agnostic-upper-bound-equal-opp-app} we have that with probability at least $1-\delta$:
\begin{equation*}
\sup_{h\in\mathcal{H}}|\widehat{\mathcal{R}}^p(h) - \mathcal{R}(h)| \leq \frac{3\alpha}{2} + 2\sqrt{\frac{8d\log(\frac{en}{d}) + 2\log(16/\delta)}{n}}
\end{equation*}
and Lemma \ref{lemma:non-uniform-bound-equal-opp}
\begin{equation*}
\sup_{h\in\mathcal{H}}|\widehat{\mathcal{D}}^{opp}(h) - \mathcal{D}^{opp}(h)| \leq \Delta^{opp} +  16 \sqrt{\frac{2d\log(\frac{2en}{d}) + 2\log(96/\delta)}{(1-\alpha)P_{10}n}}.
\end{equation*}
The proof proceeds in an identical manner to that of Theorem \ref{thm:comp-wise-upper-bound-demog-par-app}.
\end{proof}

\paragraph{Upper bound with fast rates}
Recall our notation:
\begin{align}
\label{eqn:estimates-zero-conditional-probs}
\bar{\gamma}^p_{1a}(h) = \frac{\sum_{i=1}^n \mathbbm{1}\{h(x^p_i) = 0, a^p_i = a, y^p_1 = 1\}}{\sum_{i=1}^n \mathbbm{1}\{a^p_i = a, y^p_i = 1\}}
\end{align}
as the empirical estimate of $\bar{\gamma}_{1a}(h) \coloneqq \mathbb{P}(h(X) = 0| A = a, Y = 1) = 0$ for $a\in\{0,1\}$. Given a (corrupted) training set $S^p$, denote by 
\begin{equation}
\label{eqn:defn-of-h-star-set}
\mathcal{H}^*  \vcentcolon = \left\{h\in\mathcal{H}\middle| \max_a \bar{\gamma}^p_{1a}(h) \leq \Delta^{opp} \land \widehat{\mathcal{R}}^p(h) \leq \frac{3\alpha}{2}\right\}
\end{equation}
the set of all classifiers that have a small loss and small values of $\bar{\gamma}^p_{1a}$ for both $a\in\{0,1\}$ on $S^p$. Consider the learner $\mathcal{L}^{fast}$ defined by
\begin{align}
\label{eqn:fancy-learner}
\mathcal{L}^{fast}(S^p) = 
\begin{cases}
    \text{any }h \in \mathcal{H}^*, & \text{if } \mathcal{H}^*\neq \emptyset\\
    \text{any }h \in \mathcal{H}, & \text{otherwise.}
\end{cases}
\end{align}
We then have the following:

\begin{theorem}
\label{thm:upper-bound-realizable-app}
Let $\mathcal{H}$ be finite and $\mathbb{P}\in \mathcal{P}(\prodspace)$ be such that for some $h^*\in \mathcal{H}$, $\mathbb{P}(h^*(X) = Y) = 1$. Denote by $P_{1a} = \mathbb{P}(Y = 1, A = a)$ for $a\in\{0,1\}$. Let $\mathcal{A}$ be any malicious adversary of power $\alpha < 0.5$.
Then for any $\delta, \eta\in(0,1)$ and any 
\begin{align*}
n & \geq \max \left\{\frac{8\log(16|\mathcal{H}|/\delta)}{(1-\alpha) P_{10}},  \frac{12\log(12/\delta)}{\alpha}, \frac{2\log(8|\mathcal{H}|/\delta)}{3\eta^2\alpha}, \frac{2\log(\frac{16|\mathcal{H}|}{\delta})}{3\eta^2 (1-\alpha) P_{10}\alpha}\right\} \\ & = \Omega\left(\frac{\log(|\mathcal{H}|/\delta)}{\eta^2 P_{10}\alpha}\right)
\end{align*}
with probability at least $1-\delta$
\begin{align*}
\mathbf{L}^{opp}(\mathcal{L}^{fast}) \preceq \left(\frac{3\alpha}{1-\eta}, \frac{2\Delta^{opp}}{1-\eta}\right).
\end{align*}
\end{theorem}

\begin{proof}
Throughout the proof we will drop the dependence of $\mathcal{H}^*$ (and other subsets of $\mathcal{H}$) of the data $S^p$. We will be interested in the probability of certain events involving $\mathcal{H}^*$ under all randomness in the generation of $S^p$: the random sampling of the clean data, the marked point and the adversary (denoted by $\mathbb{P}^{\mathcal{A}}$ as elsewhere). 

\paragraph{Step 1} First note that by Lemma \ref{lemma:realizable-non-uniform}(b), whenever $n \geq \max \left\{\frac{8\log(16/\delta)}{(1-\alpha)P_{10}}, \frac{12\log(12/\delta)}{\alpha}\right\}$
\begin{align*}
\mathbb{P}^{\mathcal{A}}\left(\left(\bar{\gamma}^p_{10}(h^*) > \Delta^{opp}\right) \lor \left(\bar{\gamma}^p_{11}(h^*) > \Delta^{opp}\right)\right) \leq \mathbb{P}^{\mathcal{A}}\left(\bar{\gamma}^p_{10}(h^*) + \bar{\gamma}^p_{11}(h^*) > \Delta^{opp}\right) \leq \frac{\delta}{4}
\end{align*}
In addition, since $|\poisoned| \sim \Bin(n,\alpha)$
\begin{align*}
\mathbb{P}^{\mathcal{A}}\left(\widehat{\mathcal{R}}^p(h^*) > \frac{3\alpha}{2}\right) \leq \mathbb{P}^{\mathcal{A}}\left(|\poisoned| \geq \frac{3\alpha}{2}n\right) \leq \exp(-\frac{\alpha n}{12}) \leq  \frac{\delta}{12}.
\end{align*}
It follows that $\mathbb{P}^{\mathcal{A}}\left(h^*\not\in\mathcal{H}^*\right) \leq \frac{\delta}{4} + \frac{\delta}{12} = \frac{\delta}{3}$.

\paragraph{Step 2} Next let $\mathcal{H}_1\subset \mathcal{H}$ be the set $\left\{h\in\mathcal{H}\middle| \mathcal{R}(h, \mathbb{P}) > \frac{3\alpha}{1-\eta}\right\}$. For any $h\in\mathcal{H}_1$
\begin{align*}
\mathbb{P}^{\mathcal{A}}\left(\widehat{\mathcal{R}}^c(h)\leq 3\alpha\right) \leq \mathbb{P}^{\mathcal{A}}\left(\Bin\left(n,\frac{3\alpha}{1-\eta}\right)\leq (1-\eta)\frac{3\alpha}{(1-\eta)}n\right) & \leq \exp\left(-\eta^2\frac{3\alpha}{2(1-\eta)}n \right) \\ & \leq \frac{\delta}{8|\mathcal{H}|},
\end{align*}
as long as $n \geq \frac{2\log(\frac{8|\mathcal{H}|}{\delta})}{3\eta^2\alpha} > \frac{2\log(\frac{8|\mathcal{H}|}{\delta})(1-\eta)}{3\eta^2\alpha}$. Taking a union bound over all $h\in\mathcal{H}_1$,
\begin{align*}
\mathbb{P}^{\mathcal{A}}\left(\min_{h\in\mathcal{H}_1}\widehat{\mathcal{R}}^c(h)\leq 3\alpha\right) \leq \frac{\delta}{8}
\end{align*}
Since also $\mathbb{P}^{\mathcal{A}}(|\poisoned|\geq \frac{3\alpha}{2}) \leq \frac{\delta}{12}$ and $\widehat{\mathcal{R}}^p(h) \geq \widehat{\mathcal{R}}^c(h) - |\poisoned|$, we obtain
\begin{align}
\label{eqn:H_1_result}
\mathbb{P}^{\mathcal{A}}\left(\min_{h\in\mathcal{H}_1}\widehat{\mathcal{R}}^p(h)\leq \frac{3\alpha}{2}\right) \leq \mathbb{P}^{\mathcal{A}}\left(\left(\min_{h\in\mathcal{H}_1}\widehat{\mathcal{R}}^c(h)\leq 3\alpha\right)\lor \left(|\poisoned| \geq \frac{3\alpha}{2}\right)\right) \leq \frac{\delta}{8} + \frac{\delta}{12} = \frac{5\delta}{24}.
\end{align}
Similarly, let $\mathcal{H}_2 = \left\{h\in\mathcal{H}\middle|\mathcal{D}^{opp}(h) > \frac{2}{1-\eta}\Delta^{opp} \right\}$. Fix any $h\in\mathcal{H}_2$. Assume without loss of generality that $\bar{\gamma}_{10} \geq \bar{\gamma}_{11} \geq 0$ (for this particular $h$ only). Then $\bar{\gamma}_{10} \geq \bar{\gamma}_{10} - \bar{\gamma}_{11} = |\bar{\gamma}_{10} - \bar{\gamma}_{11}| = |\gamma_{10} - \gamma_{11}| > \frac{2}{1-\eta}\Delta^{opp}$ (note that the $\gamma_{1a}$ are non-negative). At the same time, by Lemma \ref{lemma:realizable-non-uniform}(a), 
\begin{align*}
\mathbb{P}^{\mathcal{A}}\left(\bar{\gamma}^p_{10} \leq (1-\eta)\bar{\gamma}_{10} - \bar{\Delta}_{10}\right) \leq \frac{\delta}{8|\mathcal{H}|},
\end{align*}
whenever $$n > \max\left\{\frac{8\log(\frac{16|\mathcal{H}|}{\delta})}{(1-\alpha) P_{10}}, \frac{4\log(\frac{16|\mathcal{H}|}{\delta})}{\eta^2 (1-\alpha)P_{10}\bar{\gamma}_{10}}\right\}.$$ This is indeed the case since $n > \frac{8\log(\frac{16|\mathcal{H}|}{\delta})}{(1-\alpha) P_{10}}$ by assumption and also $$n > \frac{2\log(\frac{16|\mathcal{H}|}{\delta})}{3\eta^2 (1-\alpha) P_{10}\alpha} \geq \frac{4\log(\frac{16|\mathcal{H}|}{\delta})}{\eta^2 (1-\alpha) P_{10}\bar{\gamma}_{10}}.$$ The last inequality is obtained by observing that $\bar{\gamma}_{10} \geq \frac{2}{1-\eta}\Delta^{opp} \geq 6\alpha$, which follows by using $P_{10} \leq 0.5, \alpha \leq 0.5, \eta > 0$.

Therefore, with probability at least $1-\frac{\delta}{8|\mathcal{H}|}$, $\max_a \bar{\gamma}^p_{1a} = \bar{\gamma}^p_{10} > (1-\eta)\bar{\gamma}_{10} - \Delta_{10} \geq 2\Delta^{opp} - E_{10} \geq 2\Delta^{opp} - E_{10} - E_{11}$, with $E_{1a} = \frac{B_{1a}}{C_{1a} + B_{1a}}$, where we used inequality (\ref{eqn:yet-another-bound}).

Crucially, $2\Delta^{opp} - E_{10} - E_{11}$ does not depend on $h$. Therefore, taking a union bound over all $h \in \mathcal{H}_2$,
\begin{align*}
\mathbb{P}^{\mathcal{A}}\left(\min_{h\in\mathcal{H}_2} \max_{a} \bar{\gamma}^p_{1a}(h) \leq 2\Delta^{opp} - E_{10} - E_{11}\right) \leq \frac{\delta}{8}.
\end{align*}
Note also that since $n \geq \max \left\{\frac{8\log(16/\delta)}{(1-\alpha)P_{10}}, \frac{12\log(12/\delta)}{\alpha}\right\}$, using inequality (\ref{eqn:yet-another-bound2}),
\begin{align*}
\mathbb{P}^{\mathcal{A}}\left(E_{10} + E_{11} > \Delta^{opp}\right) \leq \frac{\delta}{4}.
\end{align*}
Therefore
\begin{align}
\label{eqn:H_2_result}
\mathbb{P}^{\mathcal{A}}\left(\min_{h\in\mathcal{H}_2}\max_{a} \bar{\gamma}^p_{1a}(h) \leq \Delta^{opp}\right) & \leq \mathbb{P}^{\mathcal{A}}\left(\min_{h\in\mathcal{H}_2} \max_{a} \bar{\gamma}^p_{1a} \leq 2\Delta^{opp} - E_{10} - E_{11}\right) \notag \\ & + \mathbb{P}^{\mathcal{A}}\left(E_{10} + E_{11} > \Delta^{opp}\right) \notag \\ & \leq \frac{3\delta}{8}.
\end{align}
Finally, using (\ref{eqn:H_1_result}) and (\ref{eqn:H_2_result}),
\begin{align*}
\mathbb{P}^{\mathcal{A}}\left(\mathcal{H}^* \cap \left(\mathcal{H}_1\cup \mathcal{H}_2\right) \neq \emptyset\right) & = \mathbb{P}^{\mathcal{A}}\left(\left(\min_{h\in\mathcal{H}_1}\widehat{\mathcal{R}}^p(h)\leq \frac{3\alpha}{2}\right) \lor \left(\min_{h\in\mathcal{H}_2}\max_{a} \bar{\gamma}^p_{1a}(h) \leq \Delta^{opp}\right)\right) \\ & \leq \frac{5\delta}{24} + \frac{3\delta}{8} \\ & < \frac{2\delta}{3}.
\end{align*}
\paragraph{Step 3} Combining steps 1 and 2, we have that with probability at least $1-\delta$, $h^* \in \mathcal{H}^*$ (and so $\mathcal{H}^*$ is non-empty) and for any $h\in\mathcal{H}$, $\mathcal{R}(h, \mathbb{P}) \leq \frac{3\alpha}{1-\eta}$ and $\mathcal{D}^{opp}(h)\leq \frac{2}{1-\eta}\Delta^{opp}$ which completes the proof.
\end{proof}